\pdfoutput=1
\documentclass[10pt]{article}

\usepackage{amsmath,amsfonts,bm}

\def\eqref#1{equation~\ref{#1}}

\def\1{\bm{1}}

\DeclareMathAlphabet{\mathsfit}{\encodingdefault}{\sfdefault}{m}{sl}
\SetMathAlphabet{\mathsfit}{bold}{\encodingdefault}{\sfdefault}{bx}{n}

\newcommand{\R}{\mathbb{R}}

\usepackage[top=1in, bottom=1in, left=1.25in, right=1.25in]{geometry}
\usepackage{authblk}                    \usepackage{xurl}      

\usepackage{subcaption}
\usepackage{natbib}                     \usepackage{amsmath, amssymb, amsthm}   \usepackage{graphicx}                   \usepackage{hyperref}                   \usepackage{enumitem}                   \usepackage{bm}                         \usepackage[nice]{nicefrac}
\usepackage{microtype}

\usepackage{algorithm}
\usepackage{algorithmic}
\usepackage{booktabs}
\usepackage{xcolor}                     
\definecolor{darkblue}{rgb}{0.0,0.0,0.65}
\definecolor{darkred}{rgb}{0.68,0.05,0.0}
\definecolor{darkgreen}{rgb}{0.0,0.29,0.29}
\definecolor{darkpurple}{rgb}{0.47,0.09,0.29}

\usepackage{hyperref}
\hypersetup{
   colorlinks = true,
   citecolor  = darkblue,
   linkcolor  = darkred,
   filecolor  = darkblue,
   urlcolor   = darkblue,
 }

\usepackage[nameinlink,noabbrev]{cleveref}

\usepackage{caption}

\usepackage{wrapfig}

\newtheorem{lemma}{Lemma}
\newtheorem{proposition}{Proposition}

\theoremstyle{definition}

\newcommand{\norm}[1]{\lVert #1 \rVert}

\newcommand{\ip}[2]{\langle#1,#2\rangle}
\newcommand{\Sym}{\mathbb{S}}

\renewcommand{\eqref}[1]{(\ref{#1})}

\title{Implicit Bias in Matrix Factorization and its \\Explicit Realization in a New Architecture}

\author[1]{Yikun Hou}
\author[2]{Suvrit Sra}
\author[3]{Alp Yurtsever}
\affil[1,3]{Umeå University, Sweden}
\affil[2]{Technical University of Munich, Germany}
\affil[1]{\texttt{yikun.hou@umu.se}}
\date{}

\begin{document}

\maketitle

\begin{abstract}
Gradient descent for matrix factorization exhibits an implicit bias toward approximately low-rank solutions. While existing theories often assume the boundedness of iterates, empirically the bias persists even with unbounded sequences. 
This reflects a dynamic where factors develop low-rank structure while their magnitudes increase, tending to align with certain directions. To capture this behavior in a stable way, we introduce a new factorization model: $X\approx UDV^\top$, where $U$ and $V$ are constrained within norm balls, while $D$ is a diagonal factor allowing the model to span the entire search space. Experiments show that this model consistently exhibits a strong implicit bias, yielding truly (rather than approximately) low-rank solutions. Extending the idea to neural networks, we introduce a new model featuring constrained layers and diagonal components that achieves competitive performance on various regression and classification tasks while producing lightweight, low-rank representations.
\end{abstract}

\section{Introduction}
The Burer--Monteiro (BM) factorization~\citep{burer2003nonlinear} is a classical technique for obtaining low-rank solutions in optimization. 
One can view it as a simple neural network that uses a single layer of hidden neurons under linear activation. 
Given the factorization $X=UV^\top$ where $U \in \mathbb{R}^{d \times r}$ and $V \in \mathbb{R}^{c \times r}$, $U$ and $V$ can be interpreted as the weights of the first and second layers, and $r$ as the number of hidden neurons. 
As a result, matrix factorization has become a convenient model for investigating the principles underlying the empirical success of neural networks. 

For instance, \citet{gunasekar2017implicit} demonstrate that gradient descent (with certain parameter selection) on BM factorization tends to converge toward approximately low-rank solutions even when $r=d$. Based on this observation, they conjecture that \emph{``with small enough step-sizes and initialization close enough to the origin, gradient descent on full-dimensional factorization converges to the minimum nuclear norm solution.''} 

In a follow-up work, \citet{razin2020implicit} present a counterexample demonstrating that implicit regularization in BM factorization \emph{cannot be explained by minimal nuclear norm}, or in fact any norm. Specifically, they present instances where the gradient method applied to BM factorization yields a diverging sequence, and all norms thus grow toward infinity.
Intriguingly, despite this divergence, they found that the rank of the estimate decreases toward its minimum. 

In fact, it is not uncommon for diverging sequences to follow a structured path. 
A prime example is the Power Method, the fundamental algorithm for finding the largest eigenvalue and eigenvector pair of a matrix.
Starting from a random initial point $x_0$, the Power Method iteratively updates the estimate by multiplying it with the matrix. 
This process amplifies the component of the vector that aligns with the direction of the dominant eigenvector more than the other components, progressively leading $x_k$ to align with this eigenvector.
In practice, $x_k$ is scaled after each iteration to avoid numerical issues from divergence. 

This perspective underpins our approach. 
We interpret implicit bias in BM factorization  as arising from growth dynamics that amplify certain directions while suppressing others, thereby inducing low-rank structure. 
To capture this effect in a stable way, we propose a new factorization model that disentangles structured bias from instability.

\subsection{Overview of main contributions}
\begin{list}{{\color{darkred!90}\tiny$\blacksquare$}}{\leftmargin=1.25em}
  \setlength{\itemsep}{0pt}
\item \textbf{A novel matrix factorization formulation}. 
We introduce a new factorization model, $X = UDV^\top$, where $U$ and $V$ are constrained within Frobenius norm balls. Projection onto this ball results in a scaling step. The middle term $D$ is a diagonal matrix that allows the model to explore the entire search space despite $U$ and $V$ being bounded.\footnote{The reader may notice a ``syntactic'' similarity with singular value decomposition; except using simpler Frobenius norm constraints on $U$ and $V$ instead of orthogonality.} 

\vskip1pt
Through extensive empirics we demonstrate that the gradient method applied to the proposed formulation exhibits a pronounced implicit bias toward low-rank solutions. 
We compare our formulation against the standard BM factorization with two unconstrained factors.  
We investigate how step size and initialization influence implicit bias, building on prior work that points to these as potential contributors. 
We find that our factorization approach largely obviates the need to rely on these conditions: it consistently finds \textbf{\emph{truly (rather than approximately) low-rank solutions}} across a wide range of initializations and step-sizes in our experiments. 

\item \textbf{A novel neural network architecture}. Motivated by the strong bias for low-rank solutions of the proposed factorization, we extend it to deep neural networks. We do so by adding constrained layers and diagonal components. We show that this new model achieves competitive performance with standard architectures across various regression and classification tasks. Importantly, our approach exhibits bias towards low-rank solutions, resulting in a natural pruning procedure that delivers compact, lightweight networks without compromising performance. 
\end{list}

\subsection{Related Work}
\label{gen_inst}

\textbf{Burer-Monteiro factorization.}
BM factorization was proposed for solving semidefinite programs \citep{burer2003nonlinear, 01_BM_first2} and has been recognized for its efficiency in addressing low-rank optimization problems \citep{boumal2016non,park2018finding}. 
Building on the connections between matrix factorization and training problems for two-layer neural networks, BM has served as foundational model for understanding implicit bias and developing theoretical insights. 

\textbf{Implicit regularization.} 
One promising line of research that aims to explain the successful generalization abilities of neural networks is that of `implicit regularization' induced by the optimization methods and architectures \citep{neyshabur2014search, neyshabur2017exploring, 01_IRandNN}. 
Several studies explore matrix factorization to investigate implicit bias \citep{gunasekar2017implicit, arora2018stronger, razin2020implicit, belabbas2020implicit,li2021towards}. 
Much of the existing work focuses on gradient flow dynamics in the limit of infinitesimal learning rates. 
Exceptionally, \citet{gidel2019implicit} examine discrete gradient dynamics in two-layer linear neural networks, showing that the dynamics progressively learn solutions of reduced-rank regression with a gradually increasing rank. 

\textbf{Constrained neural networks.}
Regularizers are frequently used in neural network training to prevent overfitting and improve generalization, or to achieve structural benefits such as sparse and compact network architectures \citep{scardapane2017group}. 
However, it is conventional to apply these regularizers as penalty functions in the objective rather than constraints. 
This approach is likely favored due to the ease of implementation, as pre-built functions are readily available in common neural network packages. 
Regularization in the form of constraints appears to be rare in neural network training. 
One notable exception is in the context of neural network training with the Frank-Wolfe algorithm \citep{pokutta2020deep, zimmer2022compression, macdonald2022interpretable}. 
Recently, \citet{pethick2025training} revealed parallels between Frank-Wolfe on constrained networks and algorithms that post-process update steps, such as Muon \citep{jordan2024muon}, which achieves state-of-the-art results on nanoGPT by orthogonalizing the update directions before applying them. 

\textbf{Pruning.}
Neural networks are often overparameterized, which can enhance generalization and avoid poor local minima. But such models then suffer from excessive memory and computational demands, making them less efficient for deployment in real-world applications \citep{00_overPbenefits}. 
Pruning reduces the number of parameters, resulting in more compact and efficient models that are easier to deploy. 
A comprehensive review on pruning is beyond the scope of this paper due to space limitations and the diversity of approaches. 
We refer to \citep{reed1993pruning,00_purningSurvey,cheng2024survey} and the references therein for detailed reviews. 
Pruning by singular value thresholding has recently shown promising results, particularly in natural language processing \citep{chen2021drone}, and is often used along with various enhancements such as importance weights and data whitening~for effective compression of large language models \citep{hsu2022language, yuan2023asvd, wang2024svd}. 

\section{Matrix Factorization with a Diagonal Component}

Consider the problem of reconstructing a matrix $\smash{X \in \R^{d \times c}}$ from a set of linear measurements $\smash{b = \mathcal{A}(X) \in \R^n}$. 
We focus on recovering a PSD matrix, without loss of generality since the general case can be reformulated as a PSD matrix sensing problem~\citep{park2017non}. 
We particularly focus on the data-scarce setting where $n \ll d^2$, a notable example is matrix completion. 
This problem is inherently under-determined, but successful recovery is possible if $X$ is low-rank \citep{candes2012exact}. 

One popular approach is the BM factorization, which reparametrizes the decision variable $X$ as $UU^\top$ with $U \in \R^{d \times r}$, and $r$ is a positive integer that controls the rank of the resulting product.  
Then, the problem can be formulated as: 
\begin{equation}\label{eqn:uu-ls}
\min_{U \in \R^{d\times r}} \quad \frac{1}{2} \|\mathcal{A}(UU^\top) - b \|_2^2. 
\end{equation}
While finding the global minimum is intractable, a local solution can be found by using gradient descent~\citep{lee2016gradient}. 
Initializing at some $U_0 \in \mathbb{R}^{d \times r}$, perform:
\begin{equation*}
U_{k+1} = U_k - \eta \nabla_U f(U_kU_k^\top),
\end{equation*}
where $\eta > 0$ is the step-size, and the gradient is computed as $\nabla_U f(UU^\top) = 2 \nabla f(UU^\top) U$. 

Selecting the factorization rank $r$ is a critical decision. A small $r$ may lead to spurious local minima, resulting in inaccurate outcomes  \citep{waldspurger2020rank}. 
Conversely, a large $r$ might weaken rank regularization, rendering the problem underdetermined. Conventional wisdom suggests finding a moderate compromise between these two extremes. 
However, a key observation in \citep{gunasekar2017implicit} is that the gradient method applied to \eqref{eqn:uu-ls} exhibits a tendency towards \emph{approximately low-rank} solutions even when $r = d$.

\subsection{The Proposed Factorization}

We propose reparameterizing $X = UDU^\top$, where $U \in \mathbb{R}^{d\times r}$ is constrained to have a bounded norm, and $D \in \mathbb{R}^{r \times r}$ is a non-negative diagonal matrix:
\begin{equation}\label{eqn:udu-ls}
    \min_{\substack{U \in \R^{d\times r} \\ D \in \R^{r\times r}}} ~ \frac{1}{2} \|\mathcal{A}(UDU^\top) - b \|_2^2 \quad \mathrm{s.t.} \quad \|U\|_F \leq \alpha,~~D_{ii} \geq 0, ~~D_{ij} = 0, \quad  \forall i ~\text{and}~ \forall j \neq i
\end{equation}
where $\alpha > 0$ is a model parameter. When the problem is well-scaled, for instance through basic preprocessing with data normalization, we found that $\alpha = 1$ is a reasonable choice. 

We perform simultaneous projected-gradient updates on $U$ and $D$ with step-size $\eta >0$:
\begin{equation}\label{alg:udu}
\begin{aligned}
U_{k+1} & = \Pi_{U} \left( U_k - \eta \nabla_U f(U_kD_kU_k^\top) \right) \\
D_{k+1} & = \Pi_{D} \left( D_k - \eta \nabla_D f(U_{k}D_kU_{k}^\top) \right),
\end{aligned}
\end{equation}
where $\Pi_{U}$ and $\Pi_{D}$ are projections for the constraints in \eqref{eqn:udu-ls}; while the gradients are
\begin{equation*}
\begin{aligned}
    &\nabla_U f(UDU^\top) =  2 \nabla f(UDU^\top) U D \\
    &\nabla_D f(UDU^\top) =  U^\top \nabla f(UDU^\top) U.
\end{aligned}
\end{equation*}

\subsection{Experiments on Matrix Factorization}
\label{sec:matrix-factorization-experiments}

We present numerical experiments that compare the proposed approach with the classical BM factorization on synthetic matrix completion and on a real Fourier ptychography phase retrieval task. 

\subsubsection{Matrix Completion}
We set up a synthetic matrix completion problem to recover a PSD matrix $\smash{X_\natural = U_\natural U_\natural^\top}$, where the entries of $\smash{U_\natural \in \mathbb{R}^{100 \times 3}}$ are drawn independently from $N(0,1)$. We randomly sample $n = 900$ entries of $X_\natural$ and store them in the vector $b \in \mathbb{R}^n$. 
The goal is to recover $X_\natural$ from $b$ by solving problems \eqref{eqn:uu-ls} and \eqref{eqn:udu-ls}. 

Although the ground-truth rank is three, we adopt a full-rank factorization with $r = d$ in order to study the implicit bias. 
For initialization, we generate $U_0 \in \mathbb{R}^{d \times d}$ with entries drawn independently from $N(0,1)$; and rescale it to have Frobenius norm $\xi > 0$, which sets its initial distance from the origin. 
For the UDU formulation, we set $D_0 = I$ and the constraint parameter $\alpha = 1$. 
Our experiments focus on how initialization ($\xi$) and step size ($\eta$) affect the singular value spectrum of the solutions. 

We first investigate the impact of step-size. 
To this end, we fix $\xi = 10^{-2}$ and test different values of~$\eta$. 
The results are shown in \Cref{fig:BM-fact}. 
In the left panel, we plot the objective residual as a function of iterations. 
As expected, we observe that a smaller step-size slows down convergence. 
The right panel presents the singular value spectra after $10^{6}$ iterations, revealing no clear correlation between step size and the spectral decay of the solutions. 
Notably, while the classical BM factorization finds an approximately low-rank solution, the proposed factorization converges to a solution of true rank three, matching the ground truth.

\begin{figure}[!h]
\centering
  \includegraphics[width=0.65\linewidth]{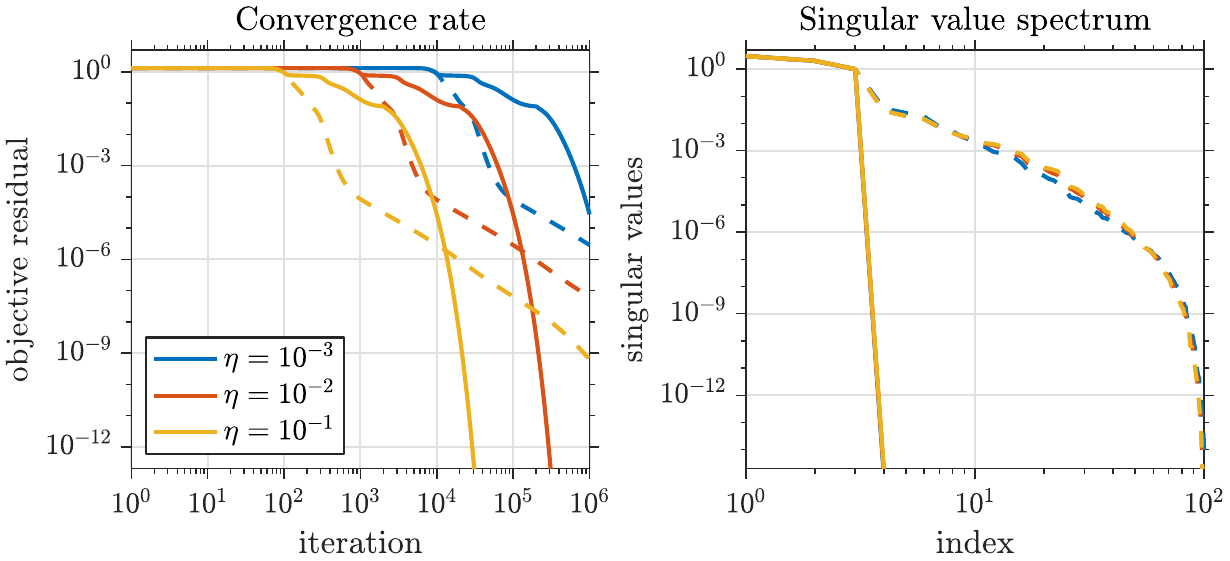} 
  
{\scriptsize Impact of \textbf{step-size} ($\eta$), with fixed initialization.} 
  
  \vspace{1em}

\centering
  \includegraphics[width=0.65\linewidth]{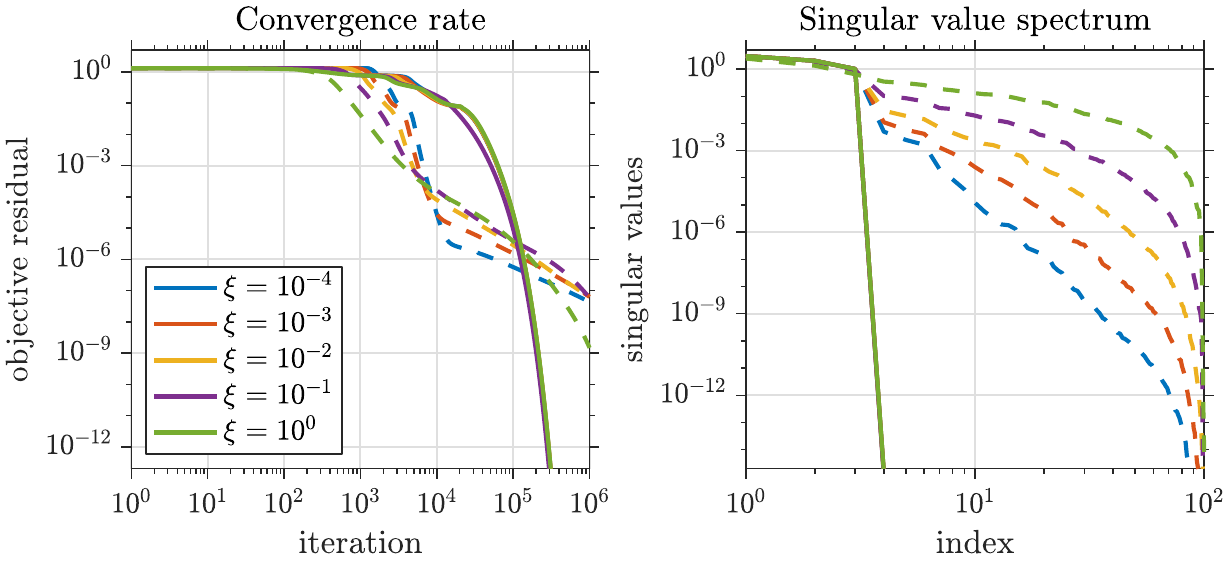}
  
  {\scriptsize Impact of \textbf{initial distance to origin} $(\xi)$, with fixed step-size.} 
  
 \captionsetup{font=small}
  \caption{Impact of step-size and initialization on implicit bias. \textbf{Solid lines represent our UDU factorization}, while \textbf{dashed lines denote the classical BM factorization}. [\textit{Left}] Objective residual vs. iterations. [\textit{Right}] Singular value spectrum after $10^{6}$ iterations. In all cases, UDU produces truly low-rank solutions, whereas the classical approach results in approximate low-rank structures.}
    \label{fig:BM-fact}
\end{figure}

Next, we examine the impact of initialization. 
We fix the step-size at $\eta = 10^{-2}$ and evaluate the effect of $\xi$. 
In the classical BM factorization, initializing closer to the origin leads to solutions with faster spectral decay, consistent with the observations of \citet{gunasekar2017implicit}. 
In contrast, the proposed UDU factorization exhibits a pronounced implicit bias toward exactly low-rank solutions, independent of the choice of $\xi$. 

We also conducted the same experiments under Gaussian noise in \Cref{app:matrix-completion-noisy}, and the results are qualitatively consistent. 
The UDU model continues to exhibit an implicit bias toward truly low-rank solutions, while the classical method gets approximately low-rank solutions. 
Exact rank-3 recovery is no longer possible~in the noisy setting, but the UDU factorization still generates a truly low-rank solution, at a slightly higher rank.

\subsubsection{Fourier Ptychography}
\label{sec:fourier-ptychography}

We consider a phase retrieval problem that arises in Fourier ptychography (FP), an imaging technique that reconstructs the phase of a complex-valued image (the sample’s transmission function) from multiple Fourier measurements under illuminations from different angles. 
We use real measurements acquired from a functional FP system applied to a human blood cell sample, obtained from \citet{horstmeyer2015solving}. 
The target is a $160 \times 160$ complex-valued image, vectorized into $d = 160^2 = 25'600$ dimensions and represented by a vector $x_\natural \in \mathbb{C}^d$. 

The measurements are modeled as $b_i \approx |\langle a_i, x_\natural \rangle|^2$ for $i=1,\dots,n$, where the approximation accounts for experimental noise from sensor limitations, illumination variability, and other imperfections. 
The sensing vectors $a_i \in \mathbb{C}^d$ are derived from windowed Fourier transforms under different illumination patterns. 
The dataset consists of $29$ illumination patterns, each recorded as an $80 \times 80$ intensity image, yielding a total of $n = 185'600$ measurements. 

FP is a nonconvex quadratic inverse problem, which admits a lifted formulation as a PSD matrix recovery problem with a rank-1 constraint. 
The lifted formulation corresponds to recovering a $25'600 \times 25'600$ complex PSD matrix from the measurements, which makes solving at full-rank factorization (with $r = d$) computationally infeasible. 
Instead, we set the factorization rank to $r=20$ and compare the proposed UDU formulation with the standard BM factorization. 
Both methods are run for $10^4$ iterations with step size $0.1$, initialized from the same random $U_0$ with i.i.d.\ Gaussian complex entries normalized in Frobenius norm. For UDU, we set $D_0 = I$ and $\alpha = 1$. 
Once solved, we recover the image estimate by extracting the dominant singular vector from the ($d \times r$) dimensional factors. 

The results are shown in \Cref{fig:ptychography}. 
The UDU factorization successfully constructs the image, whereas BM produces a heavily corrupted reconstruction. 
An inspection of the singular value spectrum reveals that UDU effectively converges to a rank-1 solution, while BM yields a spectrum with slow decay, explaining the difference in reconstruction quality. 
Importantly, UDU achieves this from a random initialization, in contrast to existing nonconvex phase retrieval methods that typically require carefully designed spectral initialization procedures \citep{candes2015phase,bian2015fourier}. 

\begin{figure}[!htbp]
\centering
    \includegraphics[width=0.8\linewidth]{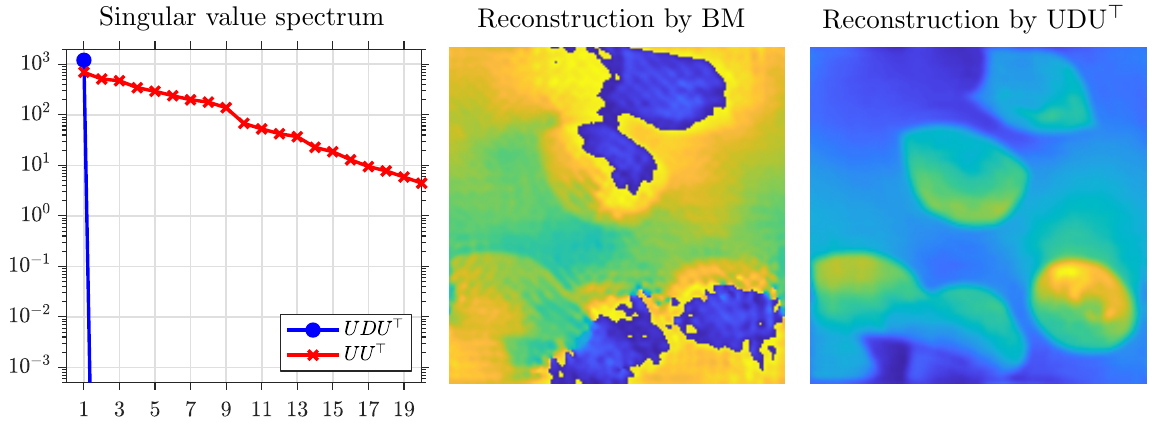}
    \captionsetup{font=small}
    \caption{[\textit{Left}] Singular value spectra of the solutions obtained by BM and UDU after $10^4$ iterations with step size $0.1$. [\textit{Middle}] The corresponding reconstructed image from the BM factorization exhibits artifacts. [\textit{Right}] The UDU factorization produces a clean recovery.}    
    \label{fig:ptychography}
\end{figure}

Additional details are provided in \Cref{app:phase-retrieval}. 
There, we show reconstructions at intermediate iterations, and experiments on smaller synthetic phase retrieval problems using full-rank factorizations ($r=d$). 
In all cases, the UDU framework promotes low-rank solutions, and this structural bias improves phase recovery.

\subsection{Theoretical Insights into the Inner Workings and Implicit Bias}
\label{sec:theoretical-insights}

A fixed-point analysis of the proposed method provides valuable insights into its inner workings.

\begin{lemma}[Fixed-point characterization]
\label{lem:fixed-point-char}
Define the update variables before projection as $\bar{U} = U - 2 \eta \nabla f(X) U D$ and $\bar{D} = D - \eta U^\top \nabla f(X) U$, with  $X = UDU^\top\!$. 
Suppose $(U, D)$ is a fixed point of the algorithm in \eqref{alg:udu}, and let $u_j$ denote the $j^{\text{th}}$ column of $U$ and $\lambda_j$ the $j^{\text{th}}$ diagonal entry of $D$. Then,

\begin{enumerate}[label=(\alph*),itemsep=0.5em, topsep=0.5em, parsep=0pt, partopsep=0pt]
    \item If $\|\bar{U}\|_F \leq \alpha$, then $\nabla f(X) u_j \lambda_j = 0$ for all $j$.
    \item If $\|\bar{U}\|_F > \alpha$, then there exists some $\beta > 0$ such that 
    $\nabla f(X) u_j \lambda_j = - \beta u_j$ for all $j$.
    \item If $\lambda_j = 0$, then $\smash{u_j^\top \nabla f(X) u_j \geq 0}$.
    \item If $\lambda_j > 0$, then $\smash{u_j^\top \nabla f(X) u_j = 0}$.
\end{enumerate}
The proof is given in \Cref{sec:appendix-fixed-point-analysis}. 
\end{lemma}

At first glance, condition $(b)$ might suggest that choosing a smaller $\alpha$ leads to fixed points where the columns of $U$ align with the negative eigenvectors of $\nabla f(X)$. However, the following proposition shows that no valid fixed point can satisfy $\|\bar U\|_F > \alpha$.

\begin{proposition}[Exclusion of spurious fixed points]
There are no fixed points $(U,D)$ of the algorithm in~\eqref{alg:udu} such that $\|\bar U\|_F > \alpha$. 
As a consequence, all valid fixed points satisfy $\|\bar U\|_F \leq \alpha$, and the fixed point structure coincides with that of BM factorization (after the change of variables $U \mapsto UD^{1/2}$). 
Thus, adding the diagonal factor $D$ and bounding $U$ does not create new fixed points for $X$.
\end{proposition}

\begin{proof}
Suppose $\|\bar U\|_F > \alpha$ and $\lambda_j > 0$. Then, by combining \Cref{lem:fixed-point-char}$\,(b)$ and $(d)$, we get
\begin{equation*}
    0 \overset{(d)}{=} u_j ^\top\nabla f(X) u_j  \overset{(b)}{=} - \tfrac{\beta}{\lambda_j} \|u_j\|^2.
\end{equation*}
which forces $u_j = 0$. 
If instead $\lambda_j = 0$, then $(b)$ again implies $u_j = 0$. 
Thus in both cases the columns of $U$ must vanish, leading to $U = 0$, and hence $\bar U = 0$, which contradicts the assumption $\|\bar U\|_F > \alpha$. 
Hence, for any fixed point of the algorithm, the constraint is inactive and $\|\bar U\|_F \leq \alpha$.

It remains to characterize these fixed points. 
From Lemma~\ref{lem:fixed-point-char}\,$(a)$, any fixed point satisfies $\nabla f(X) u_j \lambda_j = 0$ for all $j$. 
Define $\hat{u}_j := u_j \sqrt{\lambda_j}$ and let $\hat U = [\hat u_1,\ldots,\hat u_r]$. 
Then $\hat U \hat U^\top = U D U^\top = X$ and $\nabla f(X)\hat u_j=0$ for all $j$, which coincides with the fixed-point condition of the BM factorization. 
Hence, the fixed points of our method correspond exactly to those of BM. 

Conversely, let $\hat U$ be a stationary point of the BM factorization, so $X = \hat U \hat U^\top$ and 
$\nabla f(X)\hat U = 0$. For any $\alpha > 0$, set $U = (\alpha / \|\hat U\|_F) \hat U$ and 
$D = (\|\hat U\|_F^2 / \alpha^2 ) I$. Then $\|U\|_F = \alpha$, $UDU^\top = X$, and $\nabla f(X) U D = 0$. 
Thus, every stationary point of BM corresponds to a stationary point of UDU, and vice versa, at the level of $X = UDU^\top = \hat{U}\hat{U}^\top$. 
\end{proof}

Although the fixed points of UDU coincide with those of BM, the above analysis still provides insight into the low-rank bias of the algorithm. 
In particular, when $U$ grows such that $\norm{\bar{U}} \geq \alpha$, the updates tend to amplify directions aligned with the negative eigenvectors of $\nabla f(X)$.  
Concretely, we can express the update rules in terms of the columns of $u_j$ and $\lambda_j$ as $\bar{u}_j = u_j - 2 \eta \nabla f(X) u_j \lambda_i$ and $\smash{\bar{\lambda}_j = \lambda_j - \eta u_j^\top \nabla f(X) u_j}$. 
These expressions show that both $\bar{u}_j$ and $\bar{\lambda}_j$ tend to grow in the directions aligned with the negative eigenvectors of $\nabla f(X)$, even though, by \Cref{lem:fixed-point-char}\,$(d)$, no fixed point with $\lambda_j>0$ is possible unless $u_j^\top \nabla f(X) u_j = 0$. 
This tension suggests that the dynamics either suppress some columns of $U$ (via projection) or drive $X$ toward configurations where $\nabla f(X)u_j=0$, thereby reinforcing low-rank structure along the path to convergence. 

It is worth noting, that this discussion is a simplified perspective aimed at gaining intuition. 
In reality, the alignment or shrinking of the columns of $U$ and the minimization of $f(X)$ along these directions occur simultaneously and interact in a complex manner. 
Nevertheless, we can clearly observe these effects in our numerical experiments.
In \Cref{app:rank-revealing}, we present the evolution of the column norms of $U$ and the diagonal entries of $D$ over the iterations in our matrix completion experiment.
Our results show that initially, a few specific columns of $U$ grow, pushing the others numerically to zero.
Once $f(X)$ is effectively minimized with respect to these initial columns, additional columns are identified and begin to grow.
Eventually, the algorithm converges to a low-rank solution, with $UDU^\top$ becoming rank-revealing since only a few columns of $U$ remain nonzero.
We further observe that these nonzero columns are orthogonal, highlighting how the algorithm’s tendency to align $u_j$ with negative eigenvectors of $\nabla f(X)$ implicitly induces structured low-rank solutions along the path.

\section{Neural Networks with Diagonal Hidden Layers}
\label{sec: NNwithDiaonal}

\begin{figure}[!htbp]
  \centering
  \vspace{-0.5em}
  \includegraphics[width=0.7\linewidth]{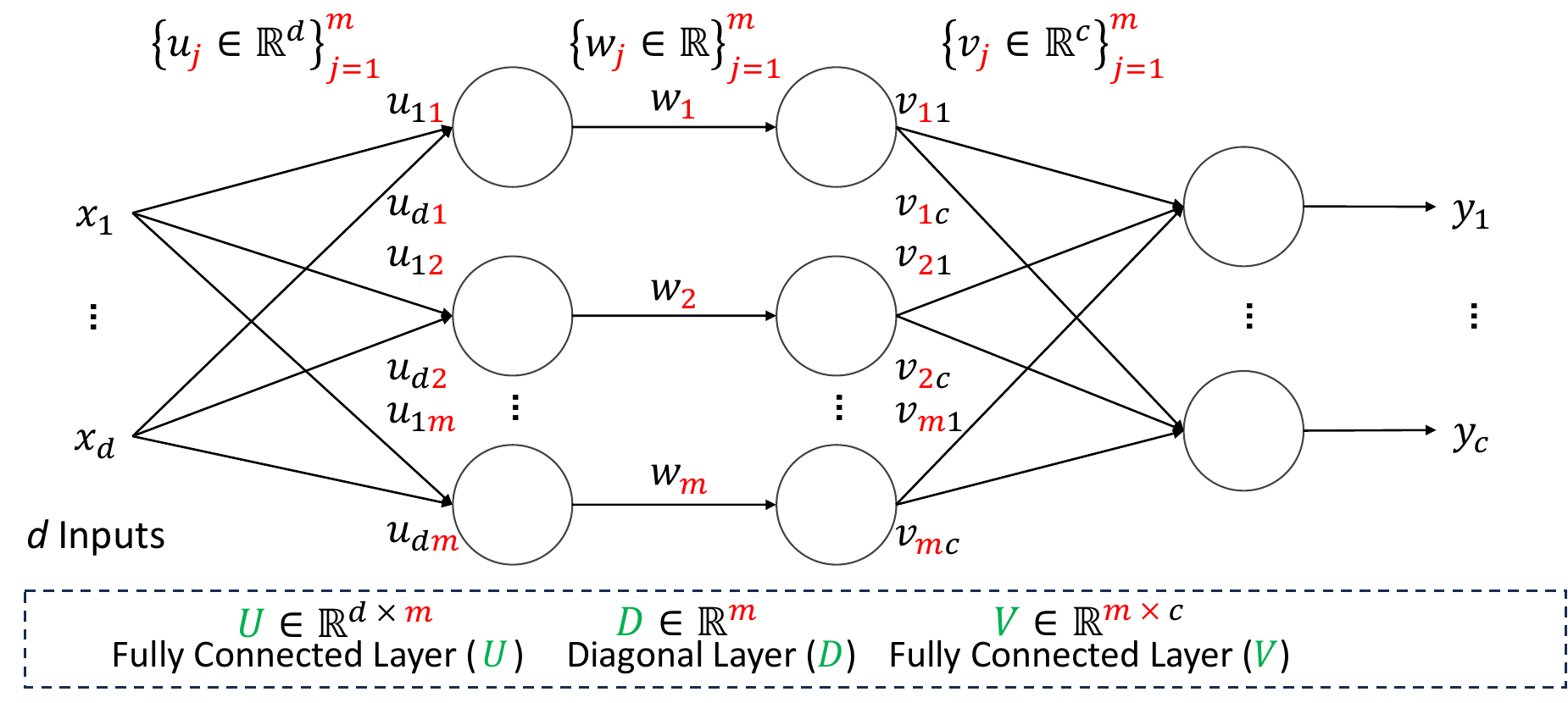}
  \captionsetup{font=small}
  \caption{UDV structure. The weights in diagonal layer $D$ are denoted as $w_j$.}
  \label{fig:UDV-Dense}
\end{figure}

This section extends our approach to neural networks. 
We introduce a three-layer neural network, with the first and third layers fully connected and the middle layer diagonal, as illustrated in \Cref{fig:UDV-Dense}. 

Let $\{(\mathbf{x}_i,\mathbf{y}_i)\}_{i=1}^n \subset \mathbb{R}^d \times \mathbb{R}^c$ denote a dataset of size $n$. 
The resulting network is given by
\begin{equation}
    \phi(\mathbf{x}) := \sum\nolimits_{j=1}^m \mathbf{v}_j w_j \mathbf{u}_j^\top \mathbf{x}
\end{equation}
with the goal of approximating $\phi(\mathbf{x}_i) \approx \mathbf{y}_i$.
Drawing parallels between our matrix factorization model in~\eqref{eqn:udu-ls} and neural network training, we impose Euclidean norm constraints on the weights of the fully connected layers.
Under these conditions, the training problem can be formulated as follows: 
\begin{equation}
\label{eqn:extra-layer-relu}
    \begin{aligned}
        & \min_{\mathbf{u}_j, w_j, \mathbf{v}_j} 
        & & \frac{1}{2n} \sum\nolimits_{i=1}^n \|  \sum\nolimits_{j = 1}^m \mathbf{v}_j \,  w_j\, \mathbf{u}_j^\top \mathbf{x}_i - \mathbf{y}_i\|_2^2 \\ 
        & ~\mathrm{s.t.} & & \sum\nolimits_{j = 1}^m \|\mathbf{u}_j \|_2^2 \leq 1, ~~ \sum\nolimits_{j = 1}^m \|\mathbf{v}_j \|_2^2 \leq 1, \\
        &
        & & \text{and } ~~ w_j \geq 0; \quad \text{for all}~~j = 1,\ldots,m. 
    \end{aligned}
\end{equation}
The norm constraints can be interpreted as a stronger form of weight decay, a standard regularization in neural networks, which lends further justification to our formulation. 
We refer to this neural network structure as UDV.

\subsection{Experiments on Neural Networks}
\label{sec:numerical-experiments-neural-networks}

In this section, we test the UDV structure on regression and classification tasks, comparing it with fully connected two-layer neural networks (denoted as UV in the subsequent text).
This comparison is fair in terms of computational cost since the overhead introduced by the diagonal layer, which can also be regarded as a parameterized linear activation function, is negligible.
To further ensure fairness in comparison, no additional fine-tuning was introduced for the proposed method.

We observe a strong empirical bias toward low-rank solutions in all our experiments with the UDV structure. 
We also present a \textit{proof-of-concept} use case of this strong bias, combined with an SVD-based pruning strategy, to produce compact networks. 

\subsubsection{Implementation Details}

\noindent \textbf{Datasets.} 
We used two datasets for the regression tasks: House Prices - Advanced Regression Techniques (HPART) \citep{04_HP} and New York City Taxi Trip Duration (NYCTTD) \citep{04_NYC}. 
We allocated 80\% of the data for the training and reserved the remaining 20\% for test.
Following \citep{04_hrelu}, we set the number of hidden neurons in the diagonal layer $\smash{m = \texttt{round}\big(\sqrt{(c+2)d}+2\sqrt{d/(c+2)}\big)}$. 
This results in a network structure ($d$-$m$-$c$) of 79-26-1 for HPART, and 12-10-1 for NYCTTD. 

For classification tasks, we used the  CIFAR-10 and CIFAR-100 dataset \citep{krizhevsky2009learning}. 
We applied transfer learning to three modern high-performing architectures: MaxViT-T \citep{04_maxvit}, EfficientNet-B0 \citep{04_effnet}, and RegNetX-32GF \citep{04_regnet}. 
The original classification head (the final linear layer) of each model was replaced with two linear layers, forming the backbone classifier (UV). 
By inserting a diagonal layer between them, the classifier was transformed into the proposed UDV structure. 
The new linear layers were initialized with a truncated normal distribution, and the diagonal layer was initialized as an identity matrix. 
The vector dimension followed that of the preceding layer, resulting in a UDV classifier structure ($d$-$m$-$c$), where $c$ denotes the number of classes (10 for CIFAR-10 and 100 for CIFAR-100). Specifically, the structures are 512-512-$c$ for MaxViT-T, 1280-1280-$c$ for EfficientNet-B0, and 2520-2520-$c$ for RegNetX-32GF.

\noindent \textbf{Loss function.} 
We used mean squared error for regression and cross-entropy loss for classification. 

\noindent \textbf{Optimization methods.} 
We tested the results using four different optimization algorithms for training: Adam \citep{04_Adam}, Mini-Batch Gradient Descent (MBGD) \citep{04_MBGD}, NAdam \citep{04_NAdam}, and Mini-Batch Gradient Descent with Momentum (MBGDM) \citep{04_MBGDM}. 
We tuned the learning rates (LRs, the term \textit{learning rates} is used instead of \textit{step sizes} for clarity in the machine learning context) for all models and optimization algorithms.
In addition, a cosine annealing learning rate scheduler was applied to stabilize training and improve convergence.

\noindent \textbf{Training Procedure and Metrics.}  
The models were trained for 200 epochs on HPART, 50 epochs on NYCTTD, and 100 epochs on CIFAR-10 and CIFAR-100. 
Test loss for regression was averaged over the last 20 and 5 epochs, respectively, whereas test accuracy for classification was averaged over the final 3 epochs.

\subsubsection{Low-rank Bias in Neural Networks}
\label{sec:numerical-experiments-neural-networks-low-rank}

\Cref{tab:results_baseline} presents the test loss (for regression) or test accuracy (for classification) of the UDV model compared to the classical UV model.
For each configuration (dataset and model architecture), the results are obtained by selecting the best algorithm and learning rate pair. 
On CIFAR-100, the UDV model underperforms UV by about 1\% in test accuracy. 
We note, however, that UDV achieves its best performance earlier in training (around 90 epochs), suggesting that early stopping may close this gap. 

\begin{table}[!ht]
\centering
\captionsetup{font=small}
\caption{Comparison of models with their best performance. M, E, and R represent the transferred models MaxVit-T, EfficientNet-B0, and RegNetX-32GF, respectively.}
\label{tab:results_baseline}
\resizebox{0.7\columnwidth}{!}{
\begin{tabular}{@{}ccccllcll@{}}
\toprule
Tasks   & \multicolumn{2}{c}{Regression (Test Loss)} & \multicolumn{6}{c}{Classification (Test Accuracy \%)}           \\ \midrule
Dataset & HPART               & NYCTTD               & \multicolumn{3}{c}{CIFAR-10} & \multicolumn{3}{c}{CIFAR-100} \\ \midrule
UDV &
  $1.304 \times 10^{-3}$ &
  $5.248 \times 10^{-6}$ &
  \multicolumn{3}{c}{\begin{tabular}[c]{@{}c@{}}M: 99.17\\ E: 97.87\\ R: 98.67\end{tabular}} &
  \multicolumn{3}{c}{\begin{tabular}[c]{@{}c@{}}M: 86.18\\ E: 91.63\\ R: 89.79\end{tabular}} \\
UV &
  $1.333 \times 10^{-3}$ &
  $5.251 \times 10^{-6}$ &
  \multicolumn{3}{c}{\begin{tabular}[c]{@{}c@{}}M: 99.11\\ E: 97.99\\ R: 98.65\end{tabular}} &
  \multicolumn{3}{c}{\begin{tabular}[c]{@{}c@{}}M: 87.75\\ E: 92.03\\ R: 90.33\end{tabular}} \\ \bottomrule
\end{tabular}
}
\end{table}

\Cref{fig:AIS_main_SV_Pruning} (top row) shows the singular value spectra for RegNetX-32GF on CIFAR-100 under different optimization algorithms. 
We report the spectra for the $U$ (for the classical model) and $UD$ (for the proposed model) layers, which generate the main data representation, and omit the V layers since it primarily acts as a feature selector and is typically tall ($c \ll m$). 
Collectively, the results show that the UDV framework matches the predictive accuracy of UV while exhibiting a clear implicit bias toward low-rank solutions, as reflected in the faster spectral decay. 

\Cref{fig:svals-hpart} presents similar results for the HPART regression task. The findings are consistent, as UDV induces a strong low-rank bias while maintaining predictive performance comparable to UV.

\begin{figure}[!htbp] 
  \centering
  \vspace{-0.5em}
  \includegraphics[clip, trim=0cm 0cm 0cm 0cm, width=0.5\linewidth]{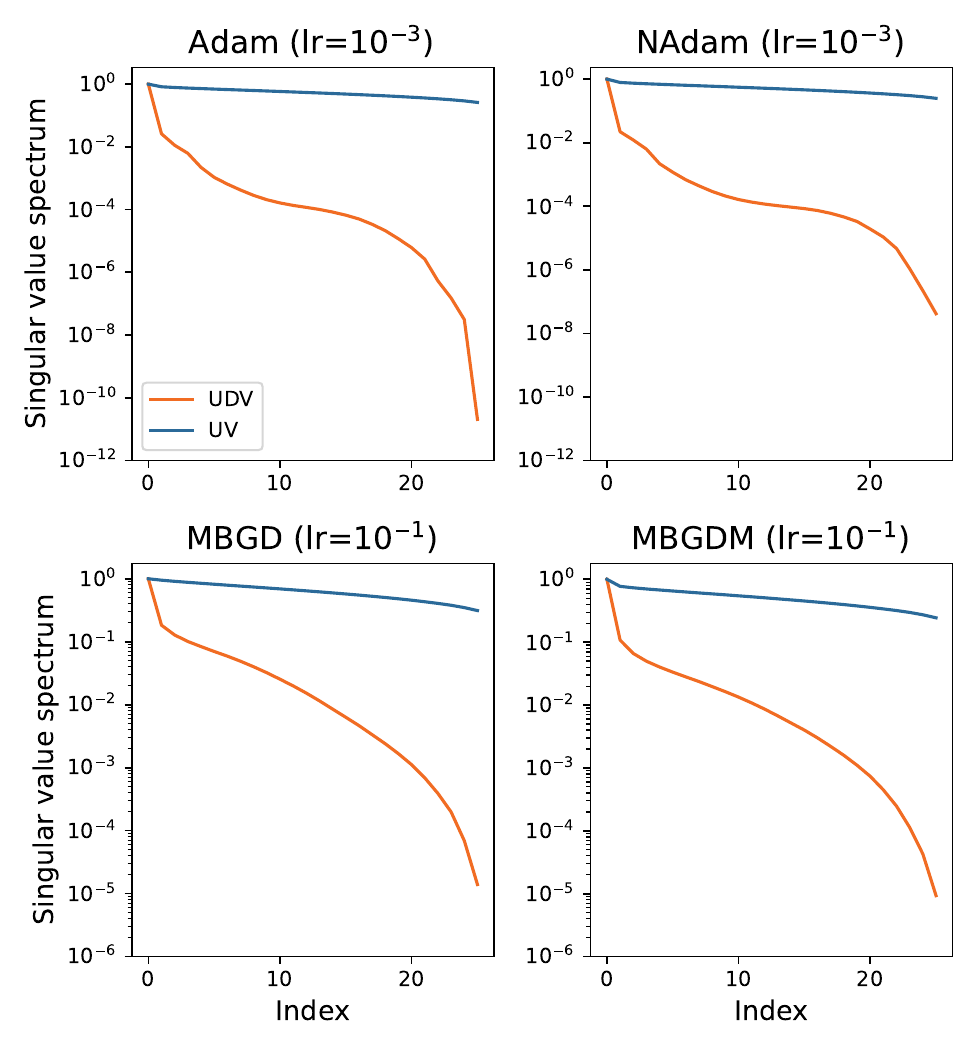}\vspace{-0.5em}
  \captionsetup{font=small}
  \caption{Singular value spectra of the solutions obtained with UV and UDV formulations on the HPART regression task under different optimization algorithms.}
  \vspace{-0.5em}
  \label{fig:svals-hpart}
\end{figure}

\subsubsection{Reducing Network Size via SVD-based Pruning}
\label{sec:SVD-based pruning}

Efficient and lightweight feed-forward layers are crucial for real-world applications. 
For instance, the Apple Intelligence Foundation Models \citep{gunter2024apple} recently reported that pruning hidden dimensions in feed-forward layers yields the most significant gains in their foundation models. 
Building on this insight, we leverage the low-rank bias of the UDV architecture through an SVD-based pruning strategy to produce compact networks without sacrificing performance.

A low-rank solution was observed when applying SVD to $UD$ layers: 
\begin{equation*}
	UD = \texttt{USV}^\top, \quad \texttt{U} \in \mathbb{R}^{d \times m}, \quad \texttt{S} \in \mathbb{R}^{m \times m}, \quad \texttt{V}^\top \in \mathbb{R}^{m \times m}.
\end{equation*} 
By dropping small singular values in $\texttt{S}$, these matrices can be truncated to $\smash{\bar{\texttt{U}} \in \mathbb{R}^{d \times r}}$, $\smash{\bar{\texttt{S}} \in \mathbb{R}^{r \times r}}$ and $\bar{\texttt{V}}^\top \in \mathbb{R}^{r \times m}$, where $0 < r < m$.
Consequently, ($m - r$) neurons can be pruned, and new weight matrices are assigned: \vspace{-0.25em}
\begin{equation*}
	\bar{U} = \bar{\texttt{U}} \in \mathbb{R}^{d \times r}, \quad \bar{D} = \bar{\texttt{S}} \in \mathbb{R}^{r \times r}, \quad \bar{V} = \bar{\texttt{V}}^{\top}V \in \mathbb{R}^{r \times c}. 
\end{equation*}

\vspace{-0.25em}

The bottom row of \Cref{fig:AIS_main_SV_Pruning} shows test accuracy of the post-training pruned models as a function of the number of remaining neurons. The x-axis is cropped for clarity (the full model has 2520 neurons before pruning). 
The UDV-based models maintain high accuracy even under aggressive pruning and consistently outperform the UV baseline at comparable compression levels. 
This demonstrates that UDV induces low-rank representations that are more amenable to compression without compromising performance. 
Although the pruned networks show that models with significantly fewer parameters can still generalize well, training such small models from scratch typically results in lower test accuracy, consistent with prior findings in the literature \citep{arora2018stronger,00_overPbenefits}.

\begin{figure*}[!htbp] 
  \centering
  \includegraphics[clip, trim=0cm 0cm 0cm 0cm, width=0.99\textwidth]{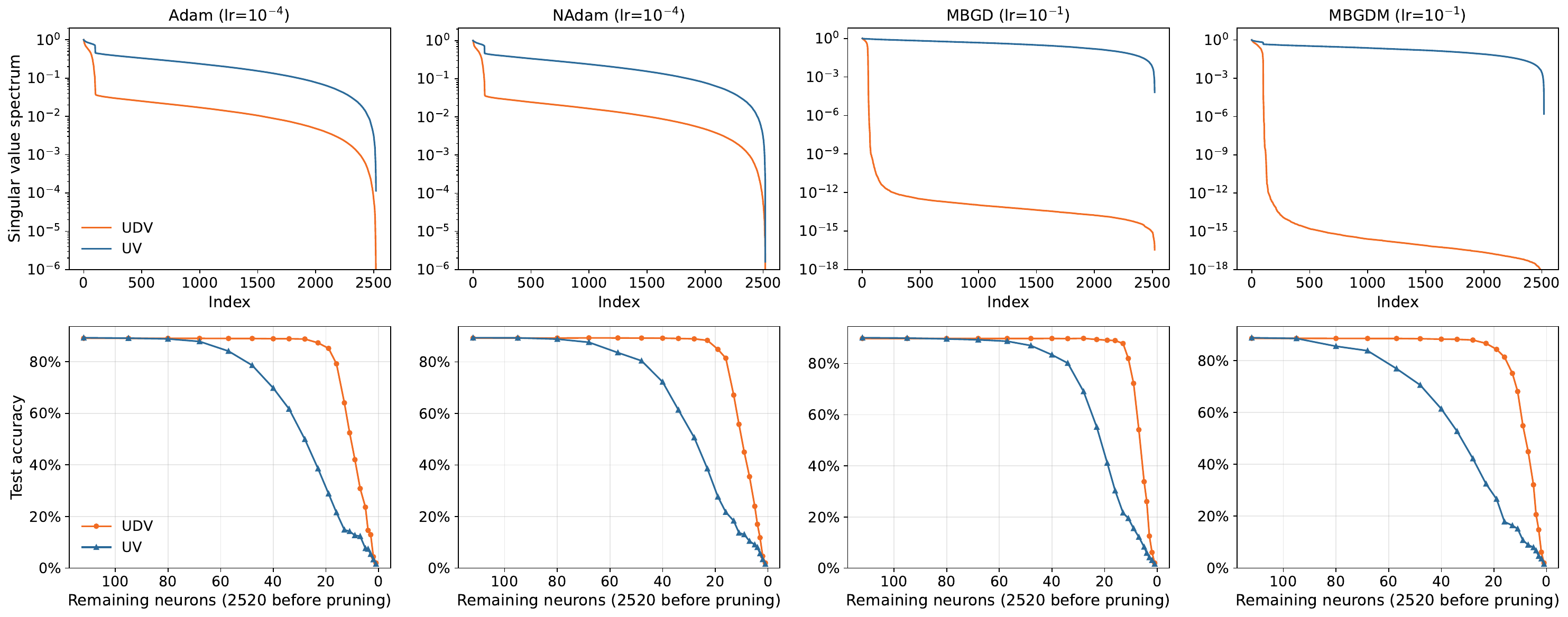}\vspace{-0.25em}
  \captionsetup{font=small}
  \caption{Singular value spectra and post-training SVD-based pruning results for RegNetX-32GF on CIFAR-100 under different optimization algorithms. [\textit{Top}] UDV induces faster spectral decay than the baseline UV model. [\textit{Bottom}] Test accuracy as a function of remaining neurons after pruning (full model has 2520 neurons, x-axis cropped for clarity). UDV-based networks retain high accuracy under aggressive pruning}
  \label{fig:AIS_main_SV_Pruning}
  \vspace{-0.5em}
\end{figure*}

\subsubsection{Further Details and Discussions}

Our findings in experiments on neural networks align with the results observed in matrix factorization. 
A key distinction, however, was the use of different optimization algorithms, including stochastic gradients and momentum steps, in neural network experiments. 
Despite these differences, the UDV architecture consistently demonstrated a strong bias toward low-rank solutions. 
Additional experiments can be found in the Appendices, with key results summarized below. 

\Cref{app:full_nnReuslts} provides comprehensive results complementing those in \Cref{sec:numerical-experiments-neural-networks-low-rank,sec:SVD-based pruning}, covering other configurations (MaxViT-T and EfficientNet-B0) and datasets (CIFAR-10/100). The low-rank patterns persist across settings, and the performance gap in pruning between UDV and UV becomes more pronounced as task complexity increases, highlighting the growing advantage of low-rank structure. 

Early in this work we explored four UDV variants with different constraints. All produced low-rank solutions, but the formulation in \eqref{eqn:extra-layer-relu} showed the strongest spectral decay, and we adopt it as the main model. The remaining variants are described in \Cref{app:udv-variants}. Additional experiments on SVD-based pruning with these UDV variants are provided in \Cref{app:pruning}. Furthermore, we analyzed the effect of learning rate on the singular value spectrum, similar to the analysis in \Cref{fig:BM-fact}, but applied to neural network experiments. This analysis confirms that the UDV framework produces low-rank solutions across a broad range of learning rates.

\Cref{app:udv-relu} extends the UDV framework by incorporating ReLU activation. Preliminary experiments suggest that UDV-ReLU also promotes low-rank solutions, consistent with the original UDV model.

Prior work on implicit bias in neural networks suggests that network depth can enhances the tendency toward low-rank solutions \citep{01_deepMF,04_rankdiminishing}, raising the question of whether the pronounced bias in UDV is just a consequence of the extra diagonal layer. To test this, we compared UDV with fully connected three-layer networks in \Cref{app:UFV}. 
We also examined an unconstrained UDV variant. 
The results show that the bias cannot be explained by depth alone, highlighting the importance of explicit constraints.

\Cref{app:weight_decay} compares the spectral decay of the UDV network to that induced by classical weight decay regularization in two- and three-layer networks. While weight decay promotes singular value decay, it can not reproduce the strong decay observed in the UDV~model.

Finally, \Cref{app:llm_task} presents a toy example applying the UDV structure within the LoRA framework to fine-tune a pre-trained LLaMA-2 model on a causal language modeling task. As a proof of concept, the pruning results suggest that UDV blocks could serve as drop-in replacements for linear layers, showing promising performance under compression.

Most importantly, our comprehensive set of experiments shows that the UDV framework consistently induces a strong implicit bias toward low-rank solutions across diverse architectures, datasets, training setups and algorithms.

\section{Conclusions}

We introduced a new matrix factorization framework that constrains the factors within Euclidean norm balls and incorporates a diagonal middle factor to avoid restricting the search space. 
Numerical experiments show that this formulation enhances the implicit bias toward low-rank solutions.

To explore the broader applicability of our findings, we designed an analogous neural network architecture with three layers, constraining the fully connected layers and adding a diagonal hidden layer, referred to as UDV. 
Extensive experiments (including those in Appendices) show that the UDV architecture induces a strong bias toward low-rank representations. 
Additionally, we demonstrated the utility of this low-rank structure by applying an SVD-based pruning strategy, as a proof of  concept, illustrating how it can be leveraged to construct compact networks that are more efficient for downstream tasks. 

The proposed model exhibits reduced rank regression behavior, where the training process gradually increases the model rank. 
We believe these our should be of broader interest to research on implicit bias. 
Although we provide some theoretical insights, developing a more complete theory and designing algorithms that fully exploit the model’s regularization capabilities remain important directions for future work.

\subsection*{Acknowledgments}
Alp Yurtsever and Yikun Hou were supported by the Wallenberg AI, Autonomous Systems and Software Program (WASP) funded by the Knut and Alice Wallenberg Foundation. Part of computations were enabled by the Berzelius resource provided by the Knut and Alice Wallenberg Foundation at the National Supercomputer Centre. Another part of computations were conducted using the resources of High Performance Computing Center North (HPC2N) which is mainly funded by The Kempe Foundations and the Knut and Alice Wallenberg Foundation. Suvrit Sra acknowledges generous support from the Alexander von Humboldt (AvH) foundation.

\renewcommand{\refname}{\large\bfseries References}
\bibliography{arxivref}
\bibliographystyle{iclr2025_conference}

\clearpage

\appendix

\renewcommand{\thefigure}{SM\arabic{figure}} 
\renewcommand{\thetable}{SM\arabic{table}}   
\setcounter{figure}{0}
\setcounter{table}{0}

\section{Additional Details on Matrix Factorization Experiments}
\label{app:additional-numerics-matrix-factorization}

\subsection{Matrix Completion under Gaussian Noise}
\label{app:matrix-completion-noisy}

We conducted similar experiments to those in \Cref{sec:matrix-factorization-experiments} also using noisy measurements: 
Let $b^\natural = \mathcal{A}(X^\natural)$ represent the true measurements, and assume $b = b^\natural + \omega$, where $\omega \in \mathbb{R}^n$ is zero-mean Gaussian noise with a standard deviation of $\sigma = 10^{-2} \|b\|_2$. 
The results, shown in \Cref{fig:BM-fact-noisy}, remain consistent with the noiseless case. 
$UDU^\top$ exhibits implicit bias toward truly low-rank solutions, while the classical BM factorization yields approximately low-rank solutions, with the spectral decay influenced by initialization. 

\begin{figure}[h]
\centering

  \includegraphics[width=0.9\linewidth]{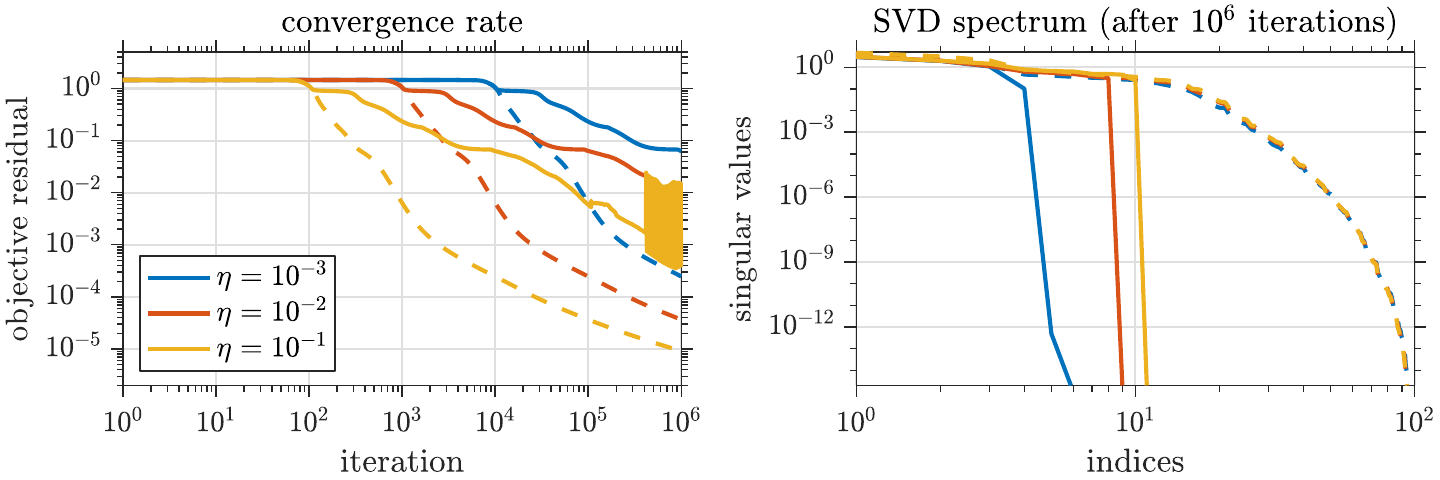} \\[0.1em]
  {Impact of \textbf{step-size} $(\eta)$, under \textbf{Gaussian noise}, with fixed initialization.} 

  \vspace{0.25em}

  \includegraphics[width=0.9\linewidth]{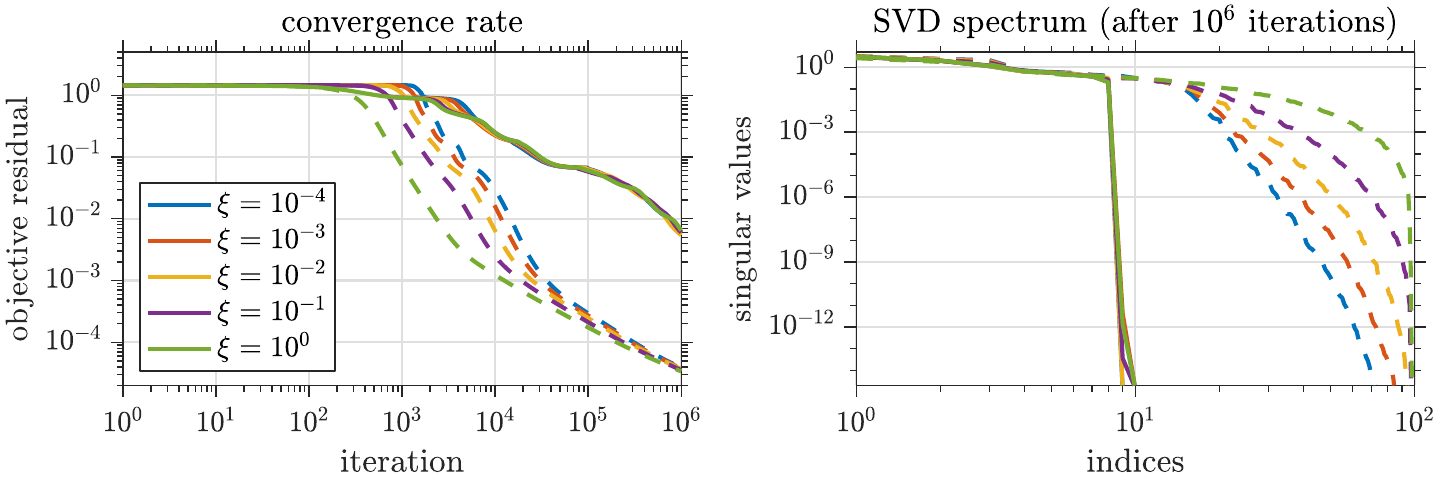} \\[0.1em]
  {Impact of \textbf{initial distance to origin} $(\xi)$, under \textbf{Gaussian noise}, with fixed step-size.} 

    \captionsetup{font=small}
  \caption{Impact of step-size and initialization on implicit bias. \textbf{Solid lines represent our UDU factorization}, while \textbf{dashed lines denote the classical BM factorization}. [\textit{Left}] Objective residual vs. iterations. [\textit{Right}] Singular value spectrum after $10^{6}$ iterations.}
\label{fig:BM-fact-noisy}
\end{figure}

\subsection{Phase Retrieval}
\label{app:phase-retrieval}

\subsubsection{Phase Retrieval on Synthetic Data}

The experiment with real data in \Cref{sec:fourier-ptychography} is large-scale and thus impractical for testing full-dimensional factorizations with $r=d$. To complement it, we perform a small-scale synthetic test bed for phase retrieval.

We consider the matrix sensing formulation of phase retrieval \citep{candes2013phaselift}, where the goal is to recover a signal from quadratic measurements of the form $y_i = |\ip{a_i}{x}|^2$. 
Although the standard maximum likelihood estimators lead to a non-convex optimization problem due to the quadratic terms, the problem can be reformulated as minimizing a convex function under a rank constraint in a lifted space. 
By denoting the lifted variable as $X = xx^\top$, the measurements can be expressed as
\begin{equation*}
    \smash{y_i = \ip{a_i^\top x}{a_i^\top x} = \ip{a_i a_i^\top}{x x^\top} := \ip{A_i}{X},}
\end{equation*}
which allows us to reformulate the problem as 
\begin{equation*}
    \smash{\min_{X \in \Sym_+^{d \times d}} \quad \frac{1}{2} \|\mathcal{A}(X) - b \|_2^2 \quad \mathrm{s.t.} \quad \mathrm{rank}(X) \leq 1.}
\end{equation*}

In this experiment, we use a gray-scale image of size $16 \times 16$ pixels, selected from the Pixel Art dataset \citep{Image_recovery}, as the signal $x \in \mathbb{R}^n$ to recover, where $n=256$.  
We generate a synthetic measurement system by sampling $a_1,\ldots,a_m$ from the standard Gaussian distribution, with $m=2n$. 
We then solve problems \eqref{eqn:uu-ls} and \eqref{eqn:udu-ls}. 
We initialized $U_0\in \R^{n \times n}$ with entries drawn \textit{iid} from the standard Gaussian distribution, then rescaled to have a unit Frobenius norm. 
For $D_0\in \R^{n \times n}$, we used the identity matrix. 
We used step-size $\eta = \frac{1}{L}$ where $L$ denotes the smoothness constant. 
After solving the problem, we recover $x \in \mathbb{R}^n$ from the lifted variable $X \in \mathbb{R}^{n \times n}$ by selecting its dominant eigenvector. 

\Cref{appendix-fig:recovered_singularvalue} demonstrates that the proposed method consistently identifies low-rank solutions, in line with the results of our other experiments.
\Cref{appendix-fig:recovered_image} displays the recovered image, demonstrating that the proposed method achieves higher quality than BM factorization within the same number of iterations, which can be attributed to the low-rank structure of the solution.

\begin{figure}[!htbp] 
  \centering
  \includegraphics[clip, trim=0cm 0.5cm 0cm 0.2cm, width=0.45\textwidth]{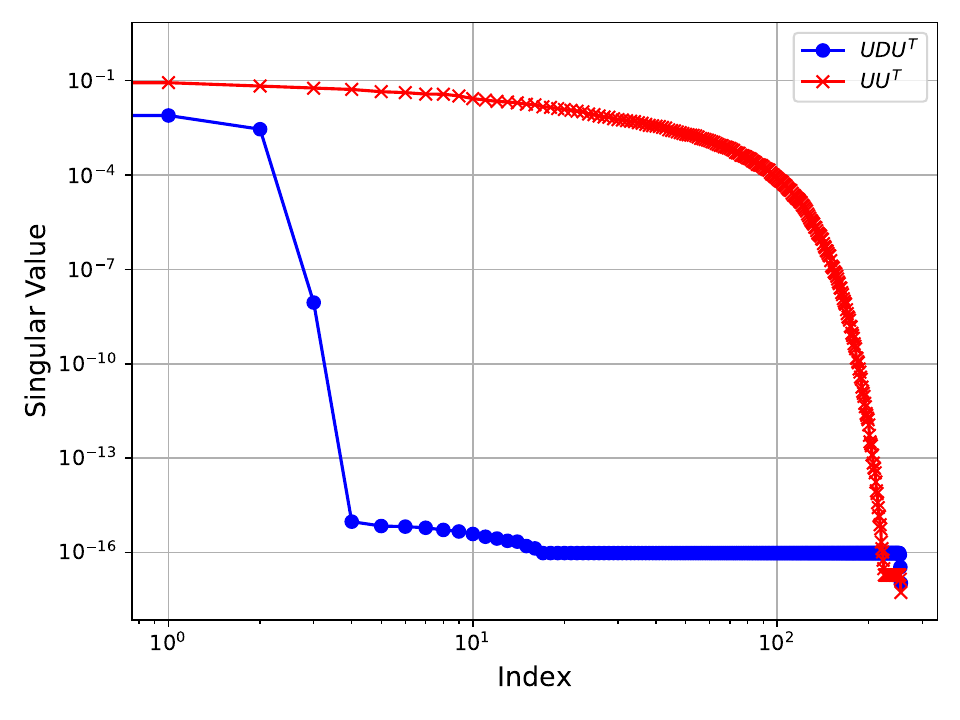}
  \captionsetup{font=small}
  \caption{Comparison of singular value spectrum in recovered image between methods based on $UDU^T$ and $UU^T$.}
  \label{appendix-fig:recovered_singularvalue}
\end{figure}

\begin{figure}[!htbp] 
  \centering
  \includegraphics[clip, trim=0cm 3cm 0cm 2.15cm, width=0.85\textwidth]{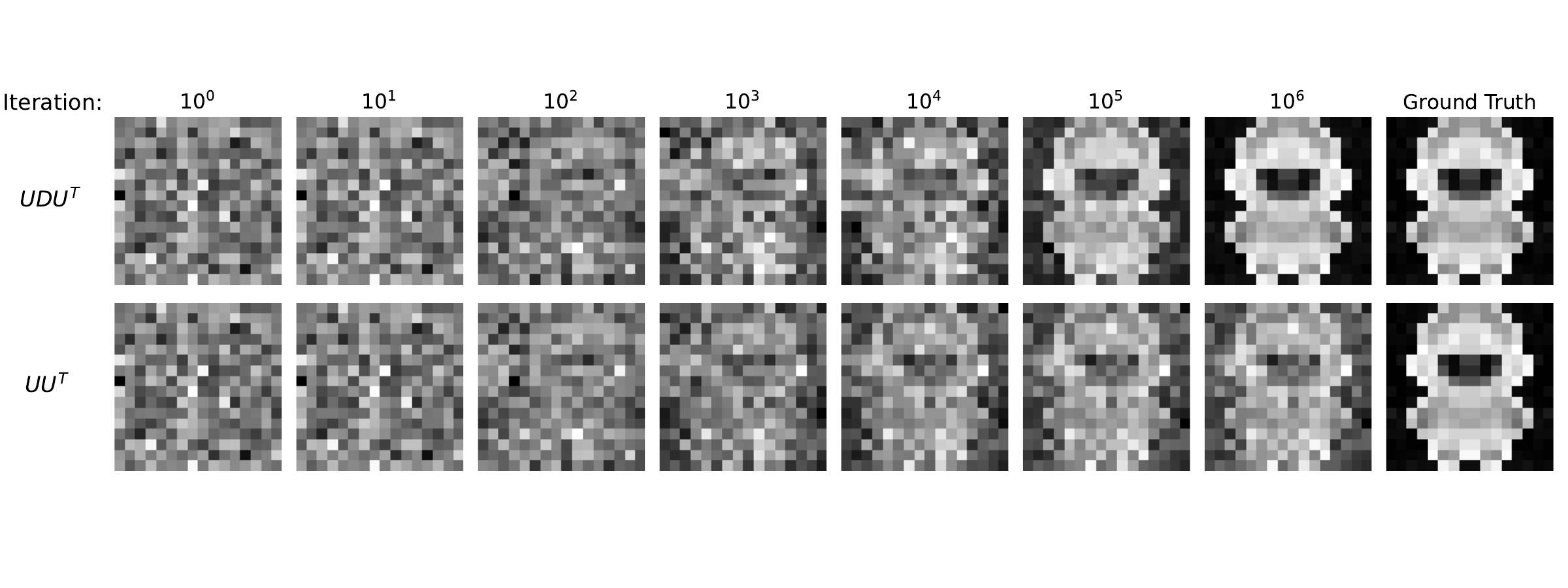}
  \captionsetup{font=small}
  \caption{Comparison of recovered image between methods based on $UDU^T$ and $UU^T$.}
  \label{appendix-fig:recovered_image}
\end{figure}

\subsubsection{Intermediate Reconstructions in Fourier Ptychography Imaging}

We provide additional details for the Fourier ptychography imaging experiment in \Cref{sec:fourier-ptychography}, showing the reconstructions at intermediate iterations (\Cref{fig:ptychography-full}).

\begin{figure*}[!htbp]
    \centering
    \captionsetup[subfigure]{font=scriptsize}
    \begin{subfigure}[b]{0.15\textwidth}
        \includegraphics[width=\textwidth]{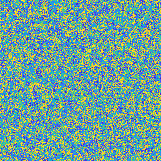}
    \end{subfigure}
    \begin{subfigure}[b]{0.15\textwidth}
        \includegraphics[width=\textwidth]{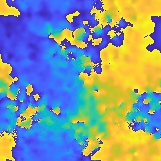}
    \end{subfigure}
    \begin{subfigure}[b]{0.15\textwidth}
        \includegraphics[width=\textwidth]{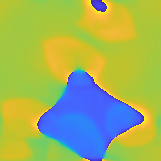}
    \end{subfigure}
    \begin{subfigure}[b]{0.15\textwidth}
        \includegraphics[width=\textwidth]{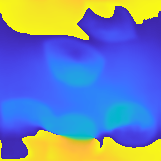}
    \end{subfigure}
    \begin{subfigure}[b]{0.15\textwidth}
        \includegraphics[width=\textwidth]{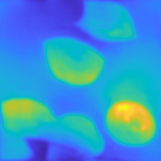}
    \end{subfigure} 
    \\[0.5em]
    \begin{subfigure}[b]{0.15\textwidth}
        \includegraphics[width=\textwidth]{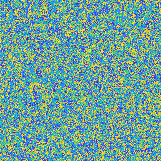}
        \caption*{$k=0$}
    \end{subfigure}
    \begin{subfigure}[b]{0.15\textwidth}
        \includegraphics[width=\textwidth]{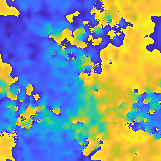}
        \caption*{$k=10$}
    \end{subfigure}
    \begin{subfigure}[b]{0.15\textwidth}
        \includegraphics[width=\textwidth]{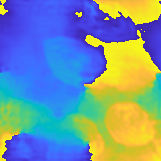}
        \caption*{$k=100$}
    \end{subfigure}
    \begin{subfigure}[b]{0.15\textwidth}
        \includegraphics[width=\textwidth]{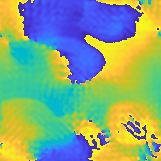}
        \caption*{$k=1000$}
    \end{subfigure}
    \begin{subfigure}[b]{0.15\textwidth}
        \includegraphics[width=\textwidth]{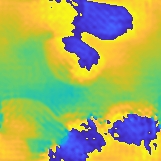}
        \caption*{$k=10000$}
    \end{subfigure} 
    \caption{Reconstructed images obtained from the proposed UDU factorization (top row) and the standard BM factorization (bottom row) at different iterations, starting from the same random initialization. While the UDU method converges to a clean and accurate recovery, the BM reconstructions display severe artifacts.}
    \label{fig:ptychography-full}
\end{figure*}

\subsection{UDU Factorization Produces Rank-Revealing Solutions}
\label{app:rank-revealing}

The proposed factorization method naturally produces rank-revealing solutions.
The columns of $U$ converge to zero in certain directions, resulting in truly low-rank solutions. 
Specifically, $U$ tends to grow along specific directions while the rescaling induced by projection onto the bounded constraint shrinks other coordinates. 
This behavior resembles the mechanism of the power method and provides insights into its connection with divergent forces. 

In \Cref{fig:norm-U-factor-UDU}, we illustrate the evolution of the column norms of $U$ from the matrix completion experiment described in \Cref{sec:matrix-factorization-experiments}. 
\Cref{fig:D-factor-UDU} displays the corresponding entries in the diagonal factor $D$. 
For comparison, \Cref{fig:norm-U-factor-BM} shows the column norms of $U$ obtained from the same experiment using the standard BM factorization. 
The results are shown for $\xi = 10^{-1}$ and $\eta = 10^{-1}$ over $10^5$ iterations, but the findings are similar across other parameter settings. 
We also repeat the experiment with the noisy measurements described in \Cref{app:matrix-completion-noisy}. The results are shown in \Cref{fig:norm-U-factor-UDU-noisy,fig:D-factor-UDU-noisy,fig:norm-U-factor-BM-noisy}.

\begin{figure}[p]
\centering
  \includegraphics[width=0.85\linewidth]{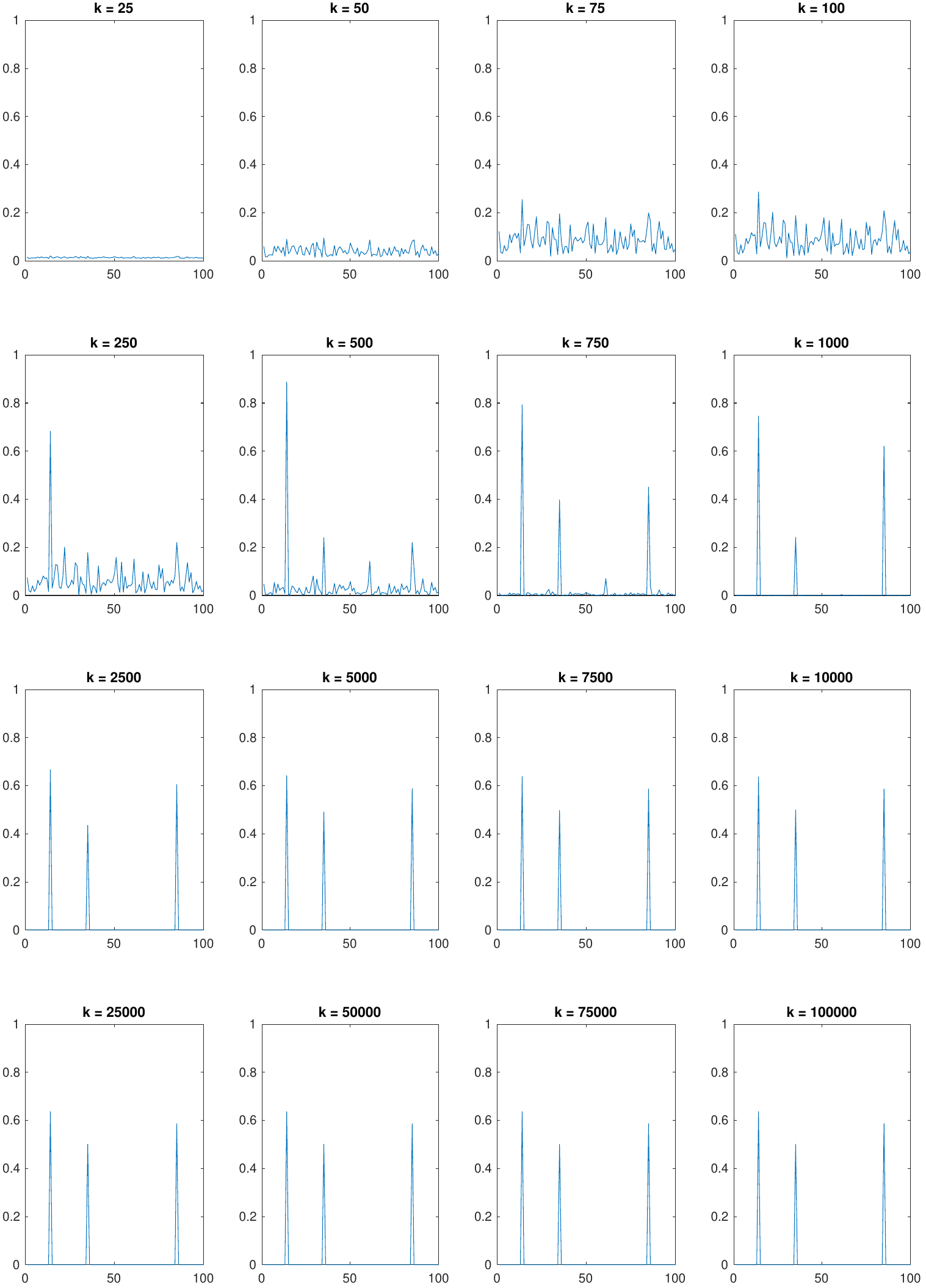} 
  \captionsetup{font=small}
    \caption{Evolution of the column norms of $U$ during the matrix completion experiment using the UDU factorization. The x-axis represents column indices, and the y-axis shows the Euclidean norms of the columns.}
    \label{fig:norm-U-factor-UDU}
\end{figure}

\begin{figure}[p]
\centering
  \includegraphics[width=0.85\linewidth]{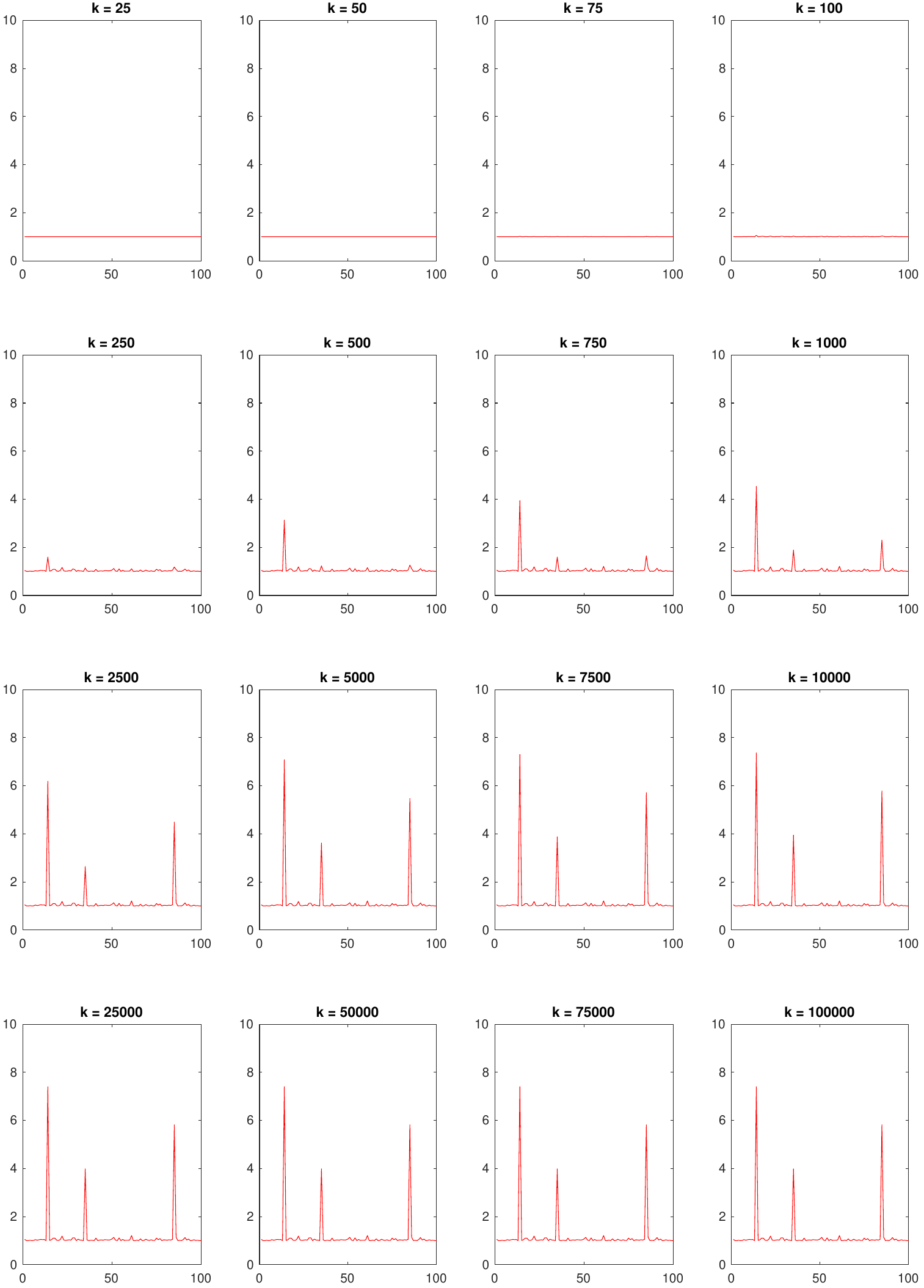} 
  \captionsetup{font=small}
    \caption{Evolution of the diagonal entries of $D$ during the matrix completion experiment using the UDU factorization. The x-axis represents indices. }
    \label{fig:D-factor-UDU}
\end{figure}

\begin{figure}[p]
\centering
  \includegraphics[width=0.85\linewidth]{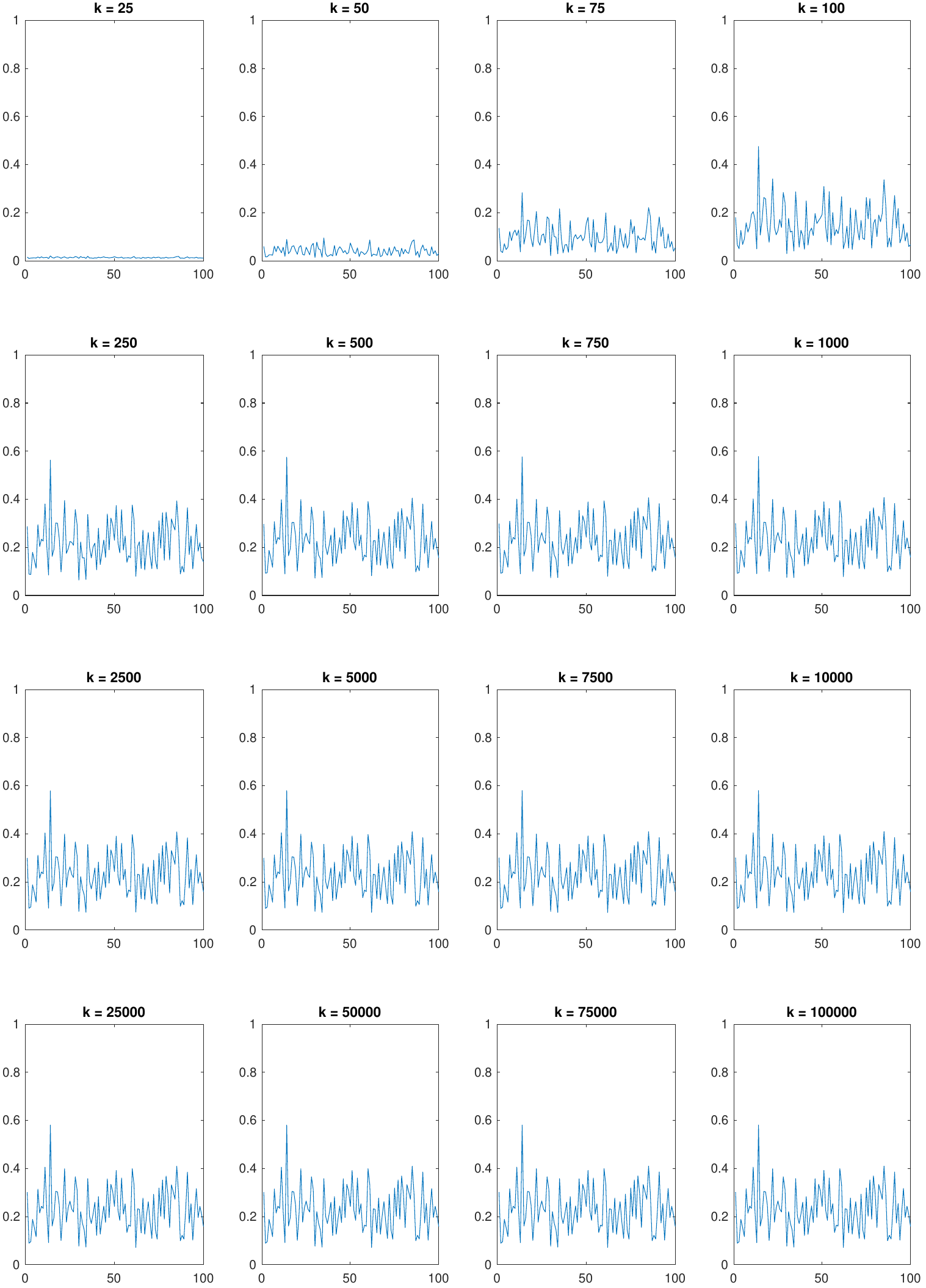} 
  \captionsetup{font=small}
    \caption{Evolution of the column norms of $U$ during the matrix completion experiment using the standard BM factorization, shown for comparison. The x-axis represents column indices, and the y-axis shows the Euclidean norms of the columns.}
    \label{fig:norm-U-factor-BM}
\end{figure}

\begin{figure}[p]
\centering
  \includegraphics[width=0.85\linewidth]{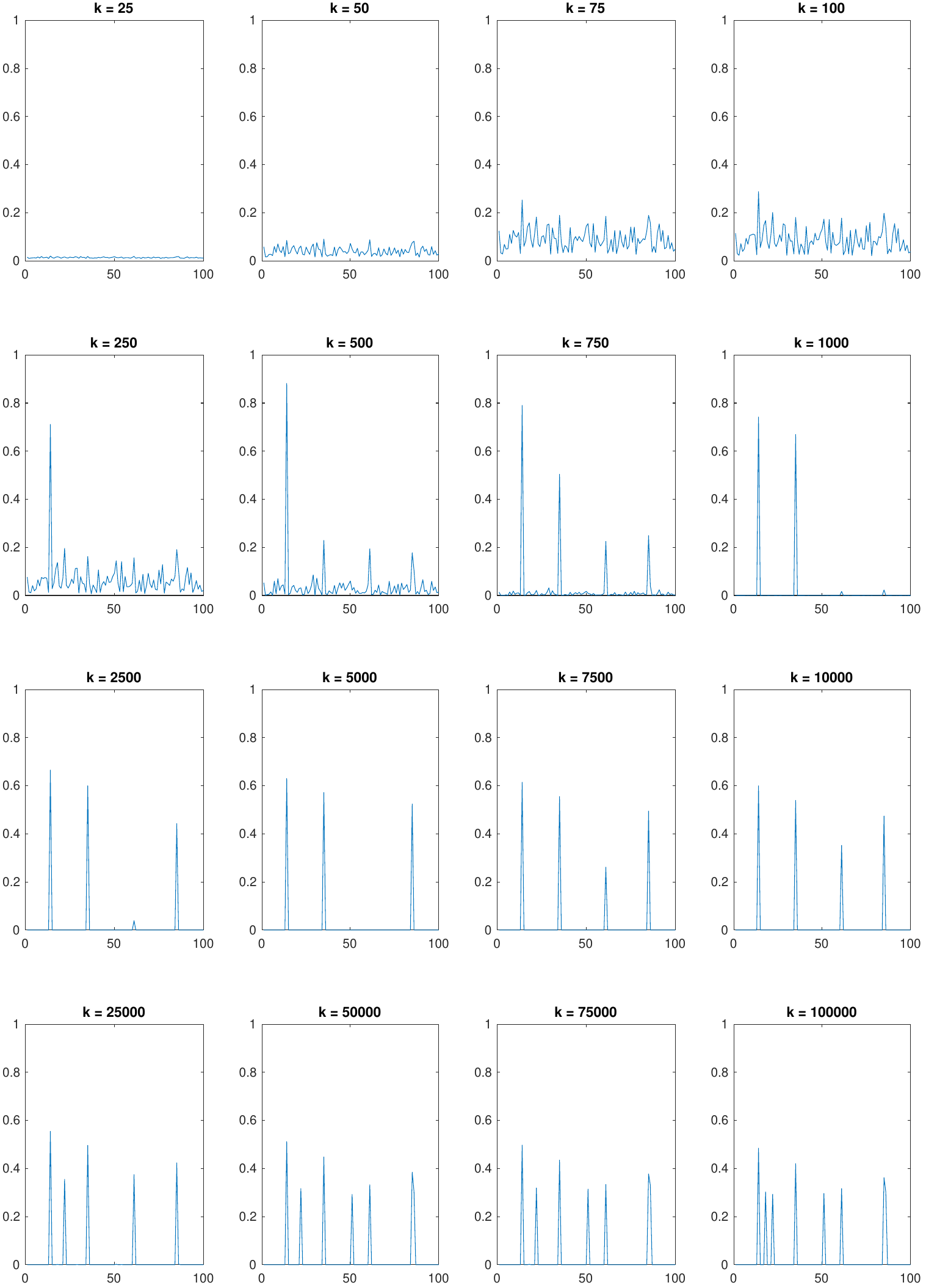} 
  \captionsetup{font=small}
    \caption{Evolution of the column norms of $U$ during the matrix completion experiment using the UDU factorization with \textbf{noisy} measurements. The x-axis represents column indices, and the y-axis shows the Euclidean norms of the columns.}

    \label{fig:norm-U-factor-UDU-noisy}
\end{figure}

\begin{figure}[p]
\centering
  \includegraphics[width=0.85\linewidth]{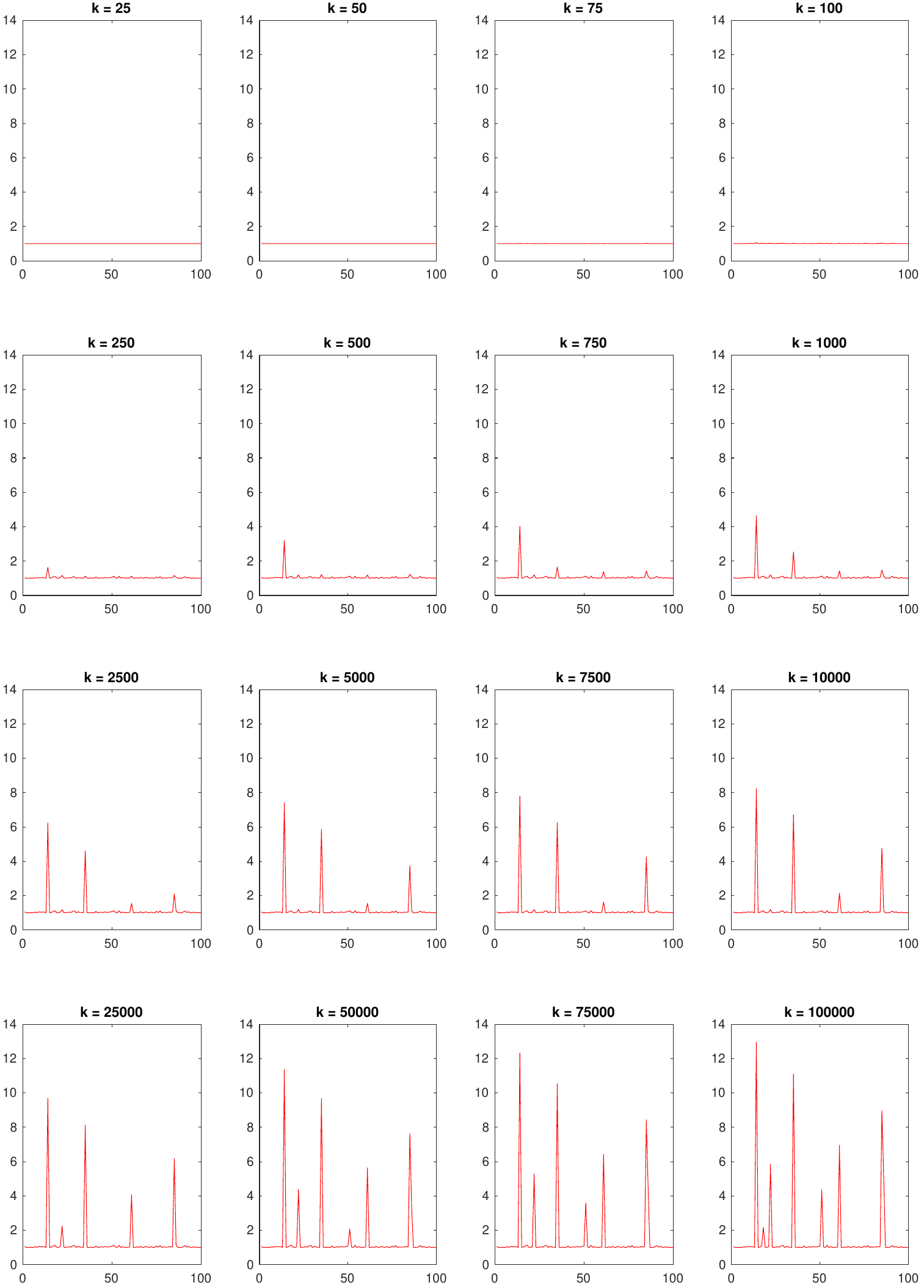} 
  \captionsetup{font=small}
    \caption{Evolution of the diagonal entries of $D$ during the matrix completion experiment using the UDU factorization with \textbf{noisy} measurements. The x-axis represents indices. }
    \label{fig:D-factor-UDU-noisy}
\end{figure}

\begin{figure}[p]
\centering
  \includegraphics[width=0.85\linewidth]{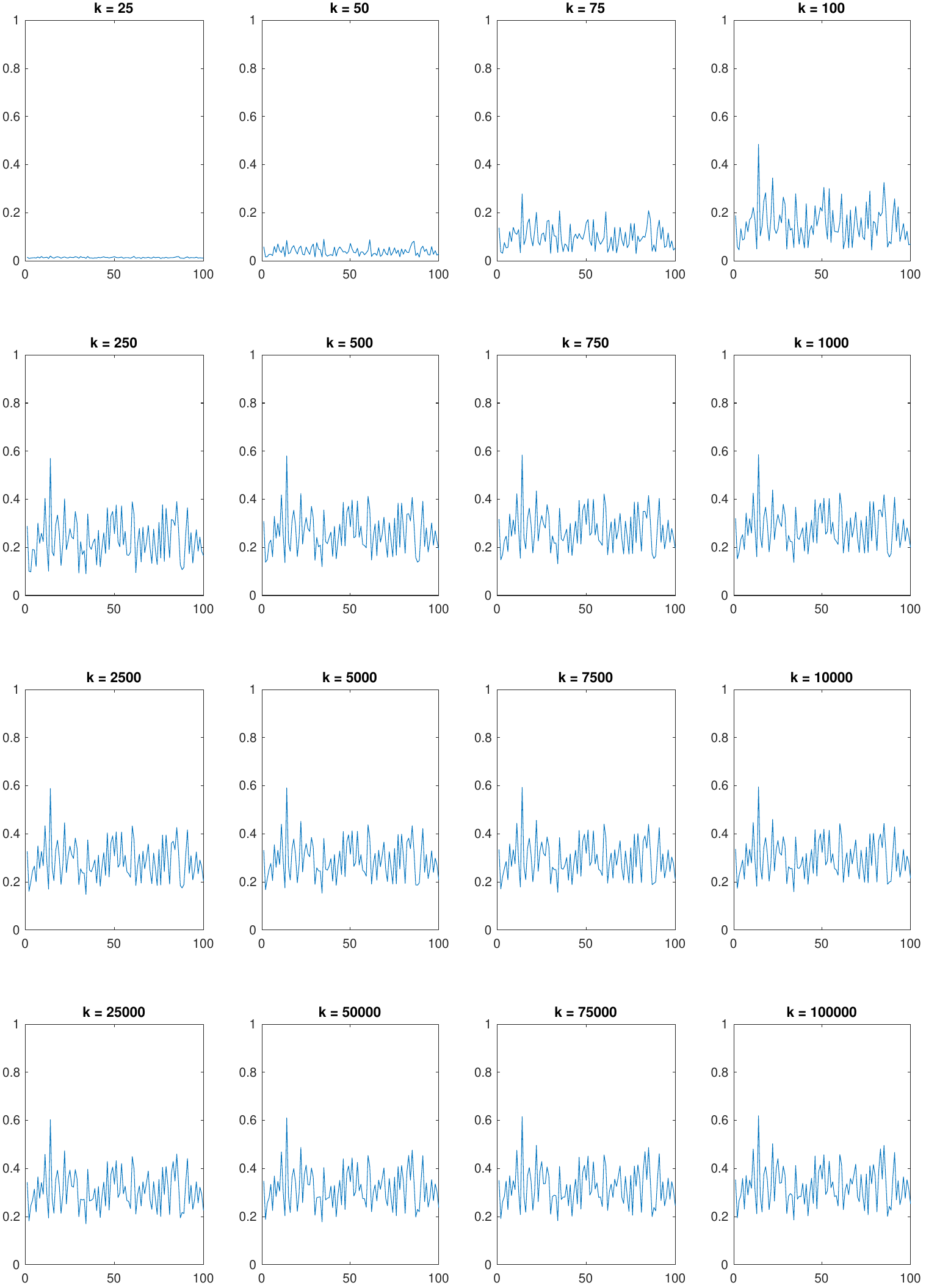} 
  \captionsetup{font=small}
    \caption{Evolution of the column norms of $U$ during the matrix completion experiment using the standard BM factorization with \textbf{noisy} measurements, shown for comparison. The x-axis represents column indices, and the y-axis shows the Euclidean norms of the columns.}
    \label{fig:norm-U-factor-BM-noisy}
\end{figure}

\clearpage 

\subsection{Additional Details on the Fixed Point Analysis}
\label{sec:appendix-fixed-point-analysis}

In \Cref{sec:theoretical-insights}, we provided insights into the inner workings of the UDU framework on matrix factorization problems based on a fixed-point analysis.
Recall the algorithm steps:
\begin{equation*}
\begin{aligned}
U_{k+1} & = \Pi_{U} (\bar{U}_{k+1}) & & \text{where} \quad \bar{U}_{k+1} = U_k - 2 \eta \nabla f(U_kD_kU_k^\top) U_k D_k \\
D_{k+1} & = \Pi_{D} (\bar{D}_{k+1}) & & \text{where} \quad \bar{D}_{k+1} = D_k - \eta U_k^\top \nabla f(U_{k}D_kU_{k}^\top) U_k  
\end{aligned}
\end{equation*}

Let $(U,D)$ denote a fixed point of this algorithm and let $X = UDU^\top$. 
Suppose $\norm{\bar{U}} \leq \alpha$. 
In this case, the projection step $\Pi_U$ acts as an identity:
\begin{equation} \label{eqn:fixed-U-1}
    U = \bar{U} = U - 2 \eta \nabla f(X) U D 
    \quad \implies \quad
    \nabla f(X) U D = 0
\end{equation}

Otherwise, if we assume $\norm{\bar{U}} \geq \alpha$, then the projection $\Pi_U$ has a normalization effect:
\begin{equation*}
    U = \frac{\alpha}{\norm{\bar{U}}} \bar{U} = \frac{\alpha}{\norm{\bar{U}}} \big(U - 2 \eta \nabla f(X) U D\big)
\end{equation*}
which we can reorganize as
\begin{equation} \label{eqn:fixed-U-2}
    \nabla f(X) U D = -\beta U \quad \text{where} \quad \beta := \frac{\norm{\bar{U}}-\alpha}{2\alpha\eta}.
\end{equation}

Denoting the columns of $U$ by $u_j$ and the diagonal entries of $D$ by $\lambda_j$, we express the conditions \eqref{eqn:fixed-U-1} and \eqref{eqn:fixed-U-2} in terms of the individual components as presented in conditions $(a)$ and $(b)$ in \Cref{sec:theoretical-insights}. 

Next, we investigate the $D$ component. 
Note that the projection $\Pi_D$ operates in an entry-wise separable manner; it maps all non-diagonal entries to $0$, and takes the positive part for the diagonal entries. 
It maps all off-diagonal entries to zero and applies the positive-part operator to the diagonal entries. 
Hence, it suffices to examine only the diagonal elements. 

Suppose $\bar{\lambda}_j \leq 0$. 
Then $\lambda_j = \max(\bar{\lambda}_j,0) = 0$, and
\begin{equation} \label{eqn:fixed-D-1}
    \bar{\lambda}_j = \lambda_j - \eta \, u_j^\top \nabla f(X) u_j = - \eta \, u_j^\top \nabla f(X) u_j \quad \implies \quad u_j^\top \nabla f(X) u_j \geq 0.
\end{equation}
Otherwise, if $\bar{\lambda}_j > 0$, then $\lambda_j = \max(\bar{\lambda}_j,0) = \bar{\lambda}_j$, hence
\begin{equation} \label{eqn:fixed-D-2}
    \lambda_j = \bar{\lambda}_j = \lambda_j - \eta \, u_j^\top \nabla f(X) u_j \quad \implies \quad u_j^\top \nabla f(X) u_j = 0.
\end{equation}

The conditions in \eqref{eqn:fixed-D-1} and \eqref{eqn:fixed-D-2} correspond to cases $(c)$ and $(d)$ in \Cref{sec:theoretical-insights}.

\newpage
\section{Additional Details on Neural Network Experiments}
\label{app:additional-numerics-neural-networks}

\noindent \textbf{Computing environment.} All classification tasks were conducted on an NVIDIA A100 GPU with four cores of the AMD Epyc 7742 processor, while regression tasks were conducted on a single core of an Intel Xeon Gold 6132 processor.
We used Python 3.12.3 and PyTorch 2.5.1. 

\noindent \textbf{Robustness Evaluation.} All reported results in the tables were averaged across 1000 seeds on HPART, 100 on NYCTTD, and 10 on CIFAR-10, CIFAR-100, and MNIST.

\subsection{Full Results of Neural Networks Reported in the Main Text}
\label{app:full_nnReuslts}
Here we provide all results related to the image classification experiments discussed in \Cref{sec: NNwithDiaonal}.
\Cref{tab:cifar10_full} and \Cref{tab:cifar100_full} present the complete results corresponding to \Cref{tab:results_baseline} in the main text, while \Cref{fig: app_10M,fig: app_10E,fig: app_10R,fig: app_100M,fig: app_100E,fig: app_100R} show all singular value spectrum and pruning results. 
The regression experiments were part of the earlier stage of this study, and their details are reported in \Cref{app:udv-variants}.
Meanwhile, the early experiments reduced the need for extensive hyperparameter selection in the current study. Therefore, for Adam and NAdam we used learning rates of $10^{-3}$ and $10^{-4}$, whereas for MBGD and MBGDM we used $10^{-1}$ and $10^{-2}$.

\begin{table}[]
\centering
\caption{Comparison of model performance on CIFAR-10. Note ``--'' denotes results that are not applicable.}
\label{tab:cifar10_full}
\resizebox{\textwidth}{!}{
\begin{tabular}{ccccccccccccc}
\hline
Model &
  \multicolumn{4}{c}{MaxVit-T [M]} &
  \multicolumn{4}{c}{EfficientNet-B0 [E]} &
  \multicolumn{4}{c}{RegNetX-32GF [R]} \\ \hline
Optimizer &
  Adam &
  NAdam &
  MBGD &
  MBGDM &
  Adam &
  NAdam &
  MBGD &
  MBGDM &
  Adam &
  NAdam &
  MBGD &
  MBGDM \\ \hline
LR: $10^{-4}$ &
  \multicolumn{1}{r}{\begin{tabular}[c]{@{}r@{}}UDV:99.14\\ UV:99.11\end{tabular}} &
  \multicolumn{1}{r}{\begin{tabular}[c]{@{}r@{}}UDV:\textbf{99.17}\\ UV:\textbf{99.11}\end{tabular}} &
  - &
  - &
  \multicolumn{1}{r}{\begin{tabular}[c]{@{}r@{}}UDV:\textbf{97.87}\\ UV:97.98\end{tabular}} &
  \multicolumn{1}{r}{\begin{tabular}[c]{@{}r@{}}UDV:97.86\\ UV:\textbf{97.99}\end{tabular}} &
  - &
  - &
  \multicolumn{1}{r}{\begin{tabular}[c]{@{}r@{}}UDV:\textbf{98.67}\\ UV:98.64\end{tabular}} &
  \multicolumn{1}{r}{\begin{tabular}[c]{@{}r@{}}UDV:98.55\\ UV:98.56\end{tabular}} &
  - &
  - \\ \hline
LR: $10^{-3}$ &
  \multicolumn{1}{r}{\begin{tabular}[c]{@{}r@{}}UDV:98.61\\ UV:98.76\end{tabular}} &
  \multicolumn{1}{r}{\begin{tabular}[c]{@{}r@{}}UDV:98.72\\ UV:98.69\end{tabular}} &
  - &
  - &
  \multicolumn{1}{r}{\begin{tabular}[c]{@{}r@{}}UDV:97.61\\ UV:97.77\end{tabular}} &
  \multicolumn{1}{r}{\begin{tabular}[c]{@{}r@{}}UDV:97.68\\ UV:97.69\end{tabular}} &
  - &
  - &
  \multicolumn{1}{r}{\begin{tabular}[c]{@{}r@{}}UDV:97.84\\ UV:97.60\end{tabular}} &
  \multicolumn{1}{r}{\begin{tabular}[c]{@{}r@{}}UDV:97.78\\ UV:97.68\end{tabular}} &
  - &
  - \\ \hline
LR: $10^{-2}$ &
  - &
  - &
  \multicolumn{1}{r}{\begin{tabular}[c]{@{}r@{}}UDV:98.04\\ UV:98.25\end{tabular}} &
  \multicolumn{1}{r}{\begin{tabular}[c]{@{}r@{}}UDV:98.93\\ UV:98.94\end{tabular}} &
  - &
  - &
  \multicolumn{1}{r}{\begin{tabular}[c]{@{}r@{}}UDV:95.98 \\ UV:96.65\end{tabular}} &
  \multicolumn{1}{r}{\begin{tabular}[c]{@{}r@{}}UDV:97.79 \\ UV:97.90\end{tabular}} &
  - &
  - &
  \multicolumn{1}{r}{\begin{tabular}[c]{@{}r@{}}UDV:98.00\\ UV:98.18\end{tabular}} &
  \multicolumn{1}{r}{\begin{tabular}[c]{@{}r@{}}UDV:98.61\\ UV:\textbf{98.65}\end{tabular}} \\ \hline
LR: $10^{-1}$ &
  - &
  - &
  \multicolumn{1}{r}{\begin{tabular}[c]{@{}r@{}}UDV:98.93\\ UV:99.01\end{tabular}} &
  \multicolumn{1}{r}{\begin{tabular}[c]{@{}r@{}}UDV:99.05\\ UV:99.03\end{tabular}} &
  - &
  - &
  \multicolumn{1}{r}{\begin{tabular}[c]{@{}r@{}}UDV:97.86 \\ UV:97.82\end{tabular}} &
  \multicolumn{1}{r}{\begin{tabular}[c]{@{}r@{}}UDV:97.87 \\ UV:97.94\end{tabular}} &
  - &
  - &
  \multicolumn{1}{r}{\begin{tabular}[c]{@{}r@{}}UDV:98.67\\ UV:98.64\end{tabular}} &
  \multicolumn{1}{r}{\begin{tabular}[c]{@{}r@{}}UDV:98.59\\ UV:98.63\end{tabular}} \\ \hline
\end{tabular}
}
\end{table}

\begin{table}[]
\centering
\caption{Comparison of model performance on CIFAR-100. Note ``--'' denotes results that are not applicable and ``$\times$'' denotes results that did not converge.}
\label{tab:cifar100_full}
\resizebox{\textwidth}{!}{
\begin{tabular}{ccccccccccccc}
\hline
Model &
  \multicolumn{4}{c}{MaxVit-T [M]} &
  \multicolumn{4}{c}{EfficientNet-B0 [E]} &
  \multicolumn{4}{c}{RegNetX-32GF [R]} \\ \hline
Optimizer &
  Adam &
  NAdam &
  MBGD &
  MBGDM &
  Adam &
  NAdam &
  MBGD &
  MBGDM &
  Adam &
  NAdam &
  MBGD &
  MBGDM \\ \hline
LR: $10^{-4}$ &
  \multicolumn{1}{r}{\begin{tabular}[c]{@{}r@{}}UDV:84.69\\ UV:\textbf{87.75}\end{tabular}} &
  \multicolumn{1}{r}{\begin{tabular}[c]{@{}r@{}}UDV:84.69\\ UV:87.72\end{tabular}} &
  - &
  - &
  \multicolumn{1}{r}{\begin{tabular}[c]{@{}r@{}}UDV:91.58\\ UV:91.90\end{tabular}} &
  \multicolumn{1}{r}{\begin{tabular}[c]{@{}r@{}}UDV:\textbf{91.63}\\ UV:\textbf{92.03}\end{tabular}} &
  - &
  - &
  \multicolumn{1}{r}{\begin{tabular}[c]{@{}r@{}}UDV:89.17\\ UV:89.22\end{tabular}} &
  \multicolumn{1}{r}{\begin{tabular}[c]{@{}r@{}}UDV:89.31\\ UV:89.40\end{tabular}} &
  - &
  - \\ \hline
LR: $10^{-3}$ &
  \multicolumn{1}{r}{\begin{tabular}[c]{@{}r@{}}UDV:84.70\\ UV:86.36\end{tabular}} &
  \multicolumn{1}{r}{\begin{tabular}[c]{@{}r@{}}UDV:84.76\\ UV:86.09\end{tabular}} &
  - &
  - &
  \multicolumn{1}{r}{\begin{tabular}[c]{@{}r@{}}UDV:89.05\\ UV:90.55\end{tabular}} &
  \multicolumn{1}{r}{\begin{tabular}[c]{@{}r@{}}UDV:89.49\\ UV:90.60\end{tabular}} &
  - &
  - &
  \multicolumn{1}{r}{\begin{tabular}[c]{@{}r@{}}UDV:85.41\\ UV:85.62\end{tabular}} &
  \multicolumn{1}{r}{\begin{tabular}[c]{@{}r@{}}UDV:85.73\\ UV:85.44\end{tabular}} &
  - &
  - \\ \hline
LR: $10^{-2}$ &
  - &
  - &
  \multicolumn{1}{r}{\begin{tabular}[c]{@{}r@{}}UDV: $\times$ \\ UV:84.57\end{tabular}} &
  \multicolumn{1}{r}{\begin{tabular}[c]{@{}r@{}}UDV:85.23\\ UV:87.73\end{tabular}} &
  - &
  - &
  \multicolumn{1}{r}{\begin{tabular}[c]{@{}r@{}}UDV: $\times$ \\ UV:88.23\end{tabular}} &
  \multicolumn{1}{r}{\begin{tabular}[c]{@{}r@{}}UDV:89.20\\ UV:91.55\end{tabular}} &
  - &
  - &
  \multicolumn{1}{r}{\begin{tabular}[c]{@{}r@{}}UDV: $\times$ \\ UV:88.82\end{tabular}} &
  \multicolumn{1}{r}{\begin{tabular}[c]{@{}r@{}}UDV:89.75\\ UV:\textbf{90.33}\end{tabular}} \\ \hline
LR: $10^{-1}$ &
  - &
  - &
  \multicolumn{1}{r}{\begin{tabular}[c]{@{}r@{}}UDV:85.45\\ UV:87.74\end{tabular}} &
  \multicolumn{1}{r}{\begin{tabular}[c]{@{}r@{}}UDV:\textbf{86.18}\\ UV:86.67\end{tabular}} &
  - &
  - &
  \multicolumn{1}{r}{\begin{tabular}[c]{@{}r@{}}UDV:89.88\\ UV:91.56\end{tabular}} &
  \multicolumn{1}{r}{\begin{tabular}[c]{@{}r@{}}UDV:91.35\\ UV:91.70\end{tabular}} &
  - &
  - &
  \multicolumn{1}{r}{\begin{tabular}[c]{@{}r@{}}UDV:\textbf{89.79}\\ UV:90.30\end{tabular}} &
  \multicolumn{1}{r}{\begin{tabular}[c]{@{}r@{}}UDV:88.53\\ UV:88.69\end{tabular}} \\ \hline
\end{tabular}
}
\end{table}

\begin{figure}[!htbp] 
  \centering
  \includegraphics[clip, trim=0cm 0cm 0cm 0cm, width=0.75\textwidth]{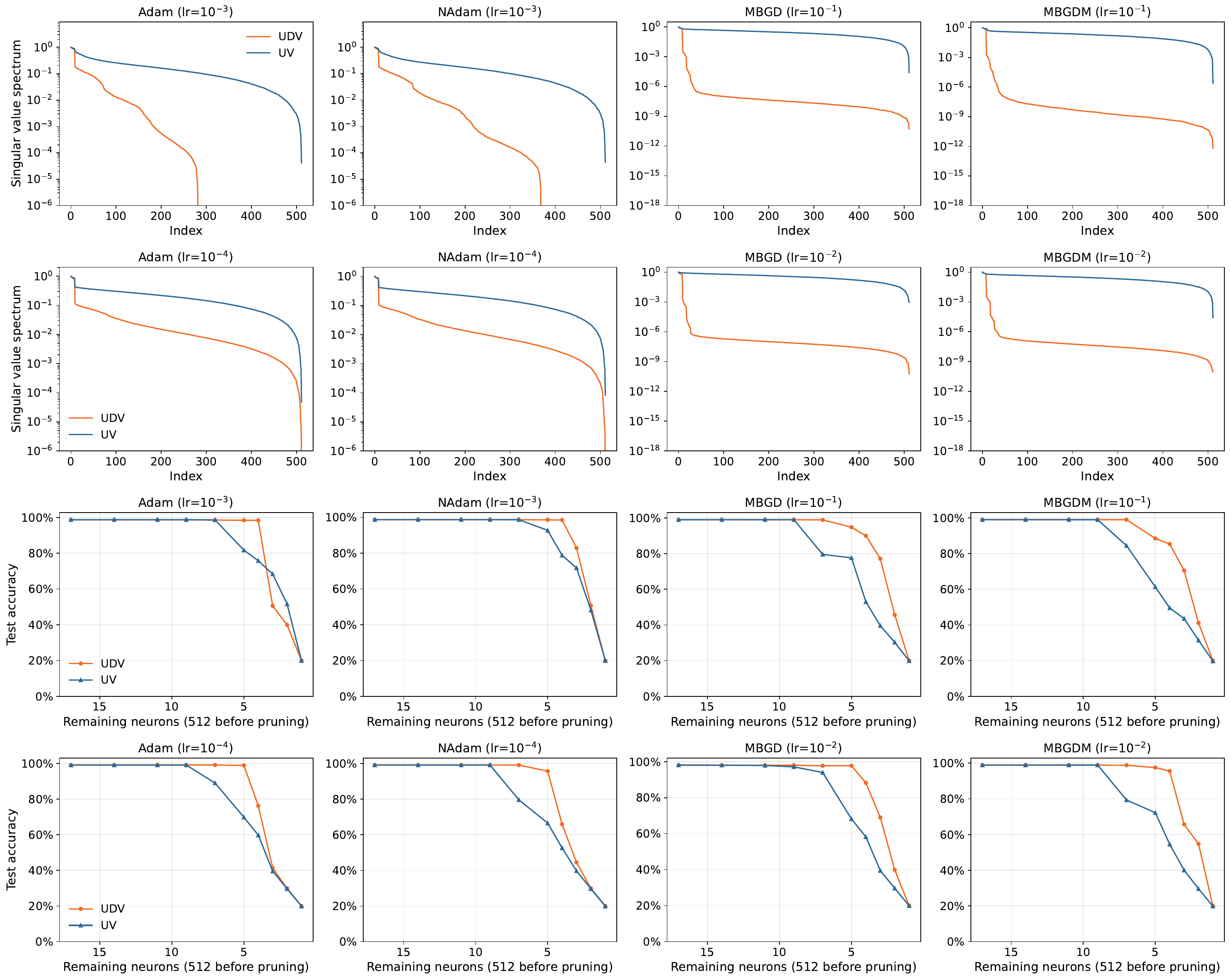}
  \captionsetup{font=small}
  \caption{Singular value spectrum and pruning performance of MaxVit-T on CIFAR-10.}
  \label{fig: app_10M}
\end{figure}

\begin{figure}[!htbp] 
  \centering
  \includegraphics[clip, trim=0cm 0cm 0cm 0cm, width=0.75\textwidth]{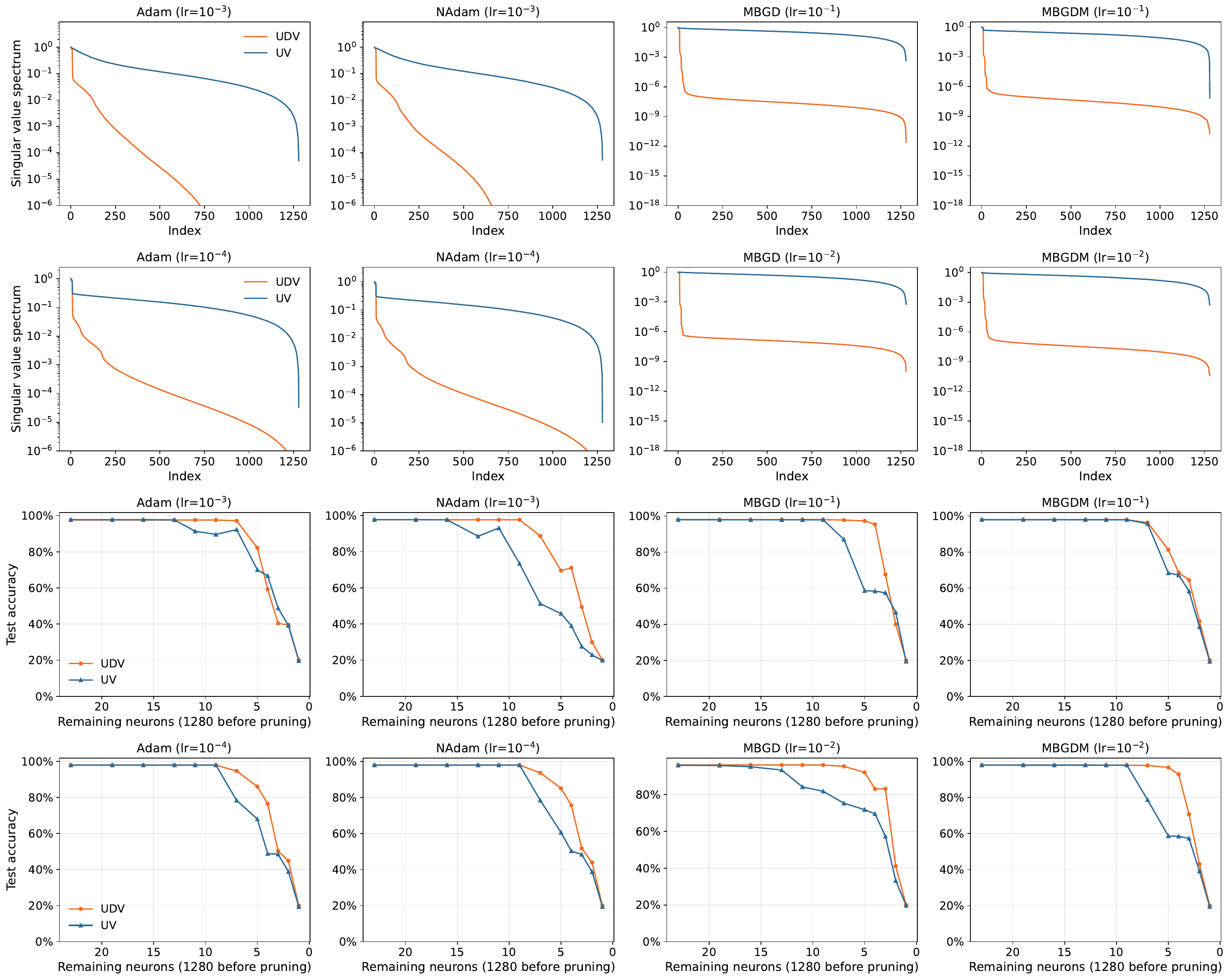}
  \captionsetup{font=small}
  \caption{Singular value spectrum and pruning performance of EfficientNet-B0 on CIFAR-10.}
  \label{fig: app_10E}
\end{figure}

\begin{figure}[!htbp] 
  \centering
  \includegraphics[clip, trim=0cm 0cm 0cm 0cm, width=0.75\textwidth]{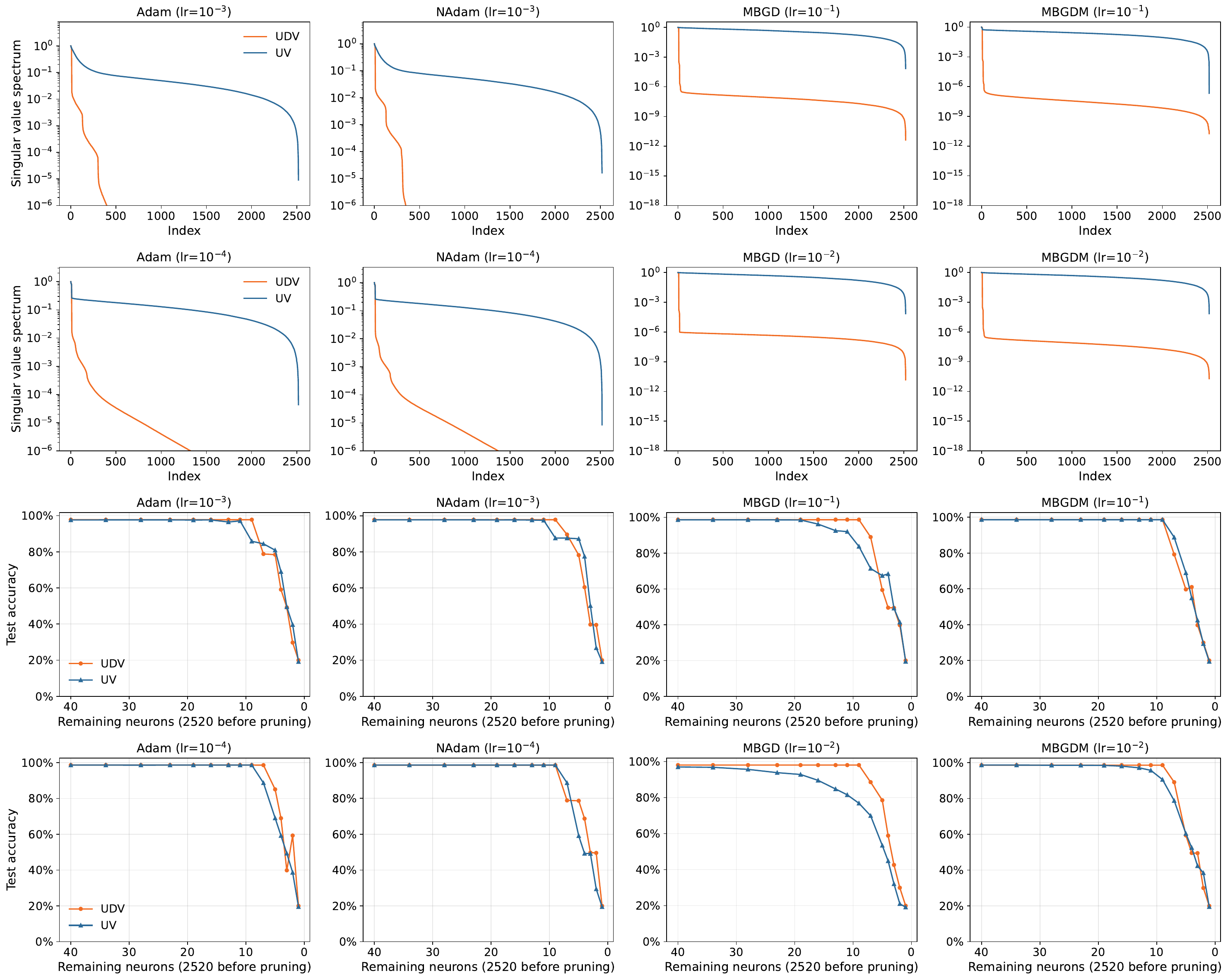}
  \captionsetup{font=small}
  \caption{Singular value spectrum and pruning performance of RegNetX-32GF on CIFAR-10.}
  \label{fig: app_10R}
\end{figure}

\begin{figure}[!htbp] 
  \centering
  \includegraphics[clip, trim=0cm 0cm 0cm 0cm, width=0.75\textwidth]{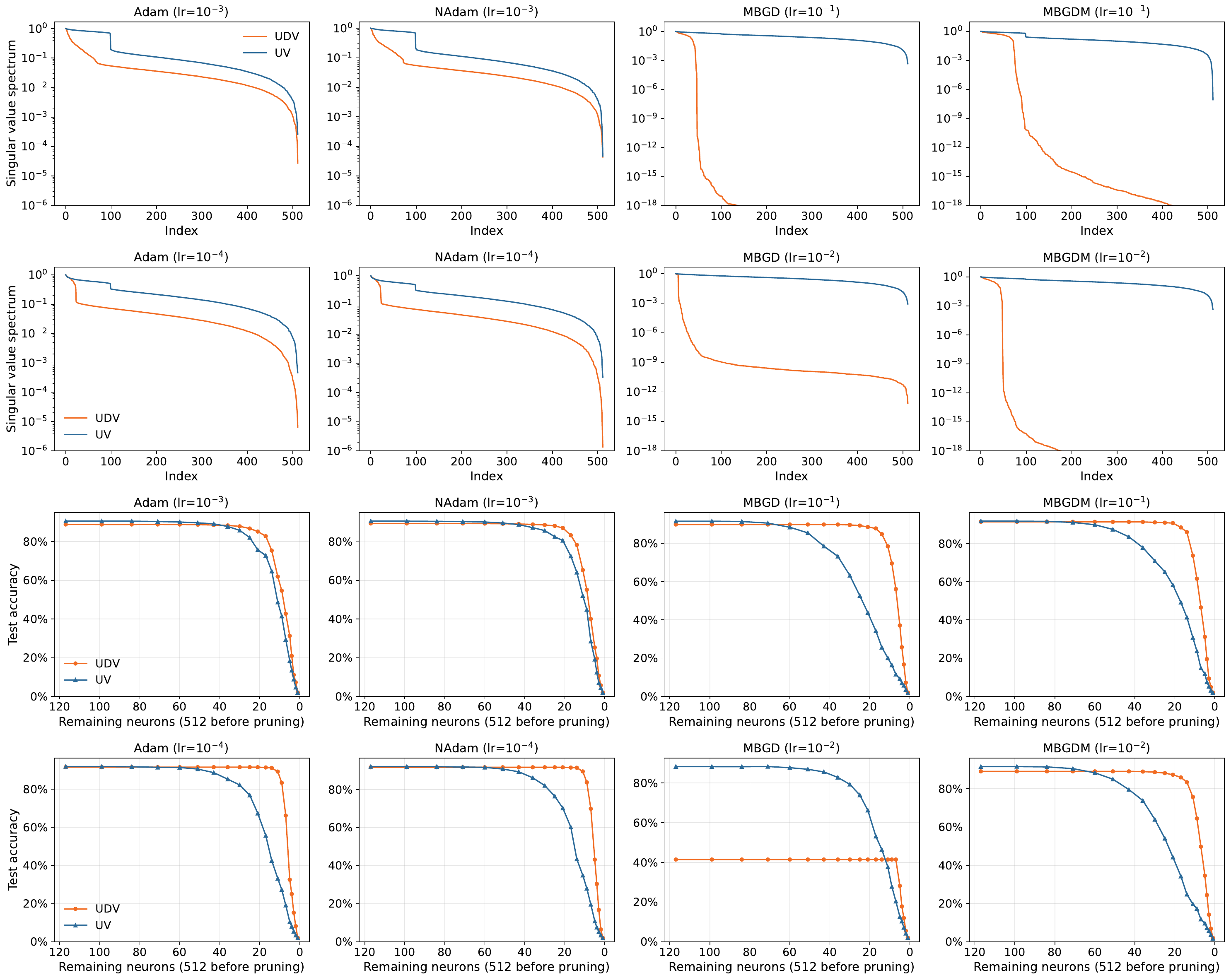}
  \captionsetup{font=small}
  \caption{Singular value spectrum and pruning performance of MaxVit-T on CIFAR-100. The learning rate of $10^{-2}$ for MBGD was unsuitable for our model, but no other hyperparameters were fine-tuned to ensure fairness.}
  \label{fig: app_100M}
\end{figure}

\begin{figure}[!htbp] 
  \centering
  \includegraphics[clip, trim=0cm 0cm 0cm 0cm, width=0.75\textwidth]{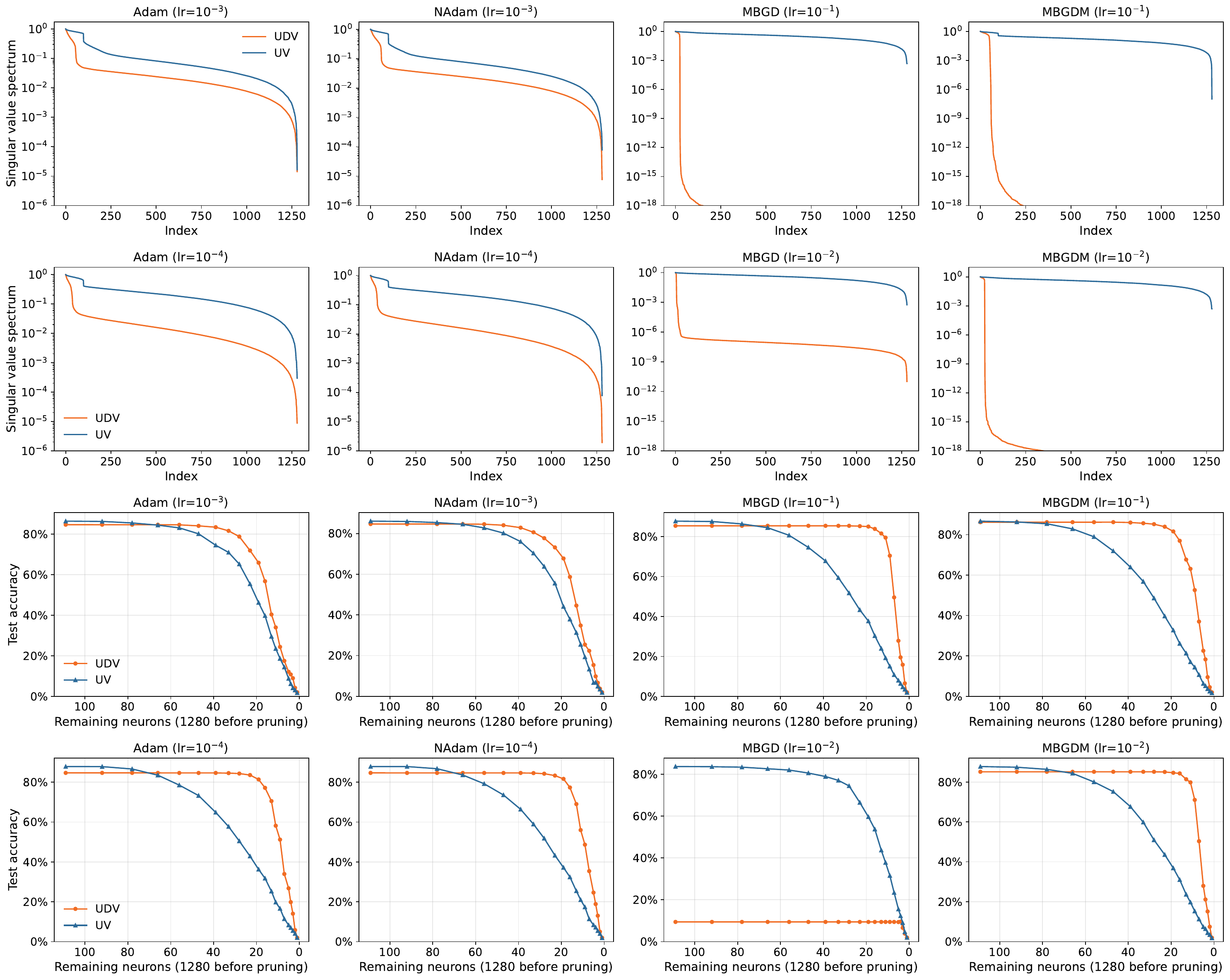}
  \captionsetup{font=small}
  \caption{Singular value spectrum and pruning performance of EfficientNet-B0 on CIFAR-100. The learning rate of $10^{-2}$ for MBGD was unsuitable for our model, but no other hyperparameters were fine-tuned to ensure fairness.}
  \label{fig: app_100E}
\end{figure}

\begin{figure}[!htbp] 
  \centering
  \includegraphics[clip, trim=0cm 0cm 0cm 0cm, width=0.75\textwidth]{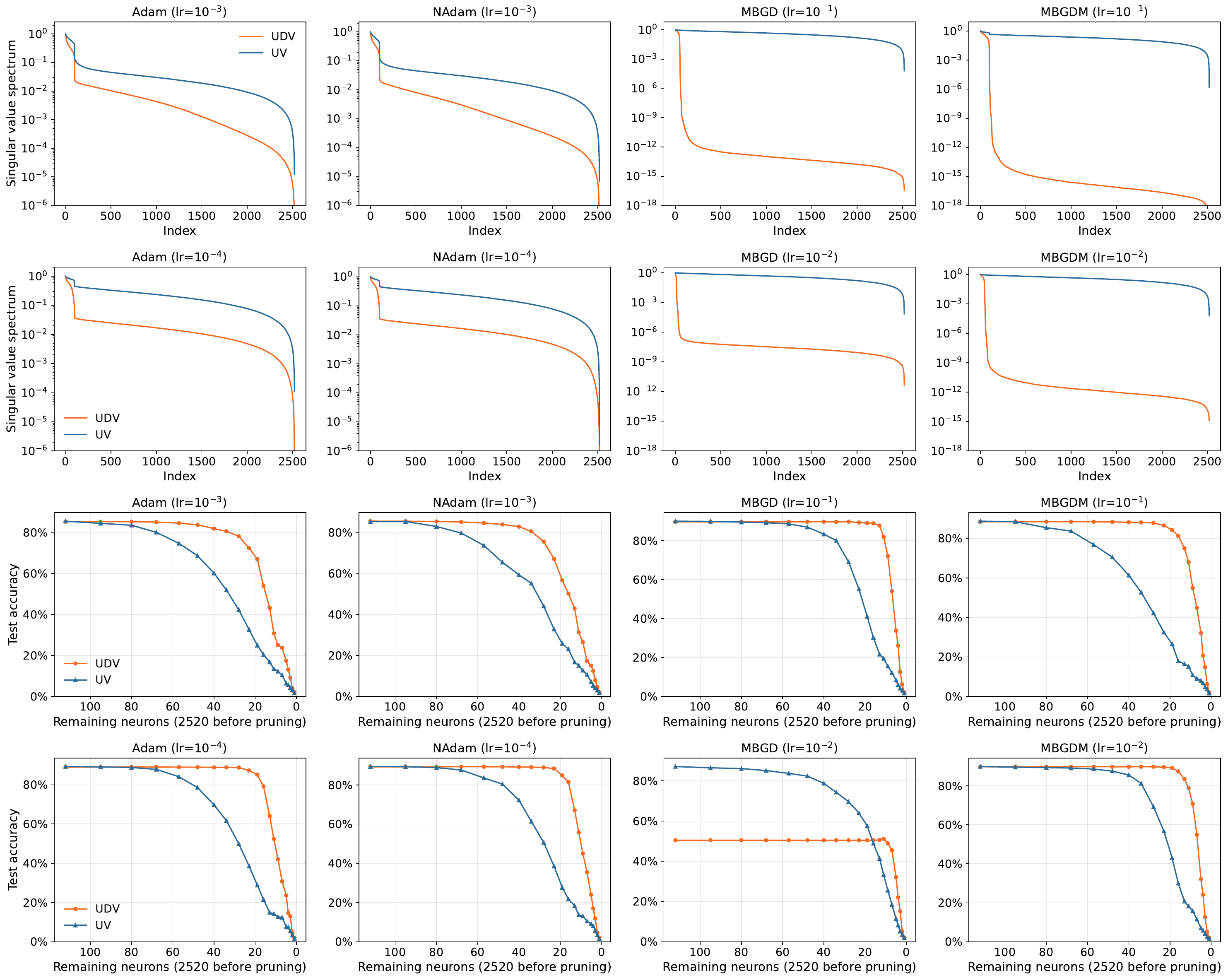}
  \captionsetup{font=small}
  \caption{Singular value spectrum and pruning performance of RegNetX-32GF on CIFAR-100. The learning rate of $10^{-2}$ for MBGD was unsuitable for our model, but no other hyperparameters were fine-tuned to ensure fairness.}
  \label{fig: app_100R}
\end{figure}

\subsection{Design Variants of the UDV Architecture}
\label{app:udv-variants}

When designing our network architecture, we considered four variants, including the one presented in \Cref{sec: NNwithDiaonal}. 
The three additional models were defined by the following sets of constraints:
\begin{gather}
        \sum_{j=1}^m \|\mathbf{u}_j \|_2^2 \leq 1, ~~ \sum_{j=1}^m \|\mathbf{v}_j \|_2^2 \leq 1 
       \tag{UDV-s} \label{appendix-eqn:udv-s} \\[0.5em]
        \|\mathbf{u}_j \|_2^2 \leq 1, ~~ \|\mathbf{v}_j \|_2^2 \leq 1 ~~~ w_j \geq 0; \quad \text{for all}~~j = 1,\ldots,m
       \tag{UDV-v1} \label{appendix-eqn:udv-v1} \\[1em]
        \|\mathbf{u}_j \|_2^2 \leq 1, ~~ \|\mathbf{v}_j \|_2^2 \leq 1 ; \quad \text{for all}~~j = 1,\ldots,m 
       \tag{UDV-v2} \label{appendix-eqn:udv-v2} 
\end{gather}

In detail, UDV-s is identical to UDV but omits the non-negativity constraints on the diagonal layer. 
UDV-v1, on the other hand, enforces row/column-wise norm constraints instead of the Frobenius norm used in UDV, while retaining the non-negativity constraints on the diagonal elements. 
Finally, UDV-v2 is identical to UDV-v1 but without the non-negativity constraints on the diagonal layer. 

Early image classification experiments were performed on the normalized MNIST dataset \citep{04_MNIST}, which allowed rapid iteration and efficient validation of the proposed method.
The replacement of the classifier follows the procedure described in \Cref{sec:numerical-experiments-neural-networks}.
The number of hidden neurons in the diagonal layer was defined as $m = \texttt{floor}(\tfrac{2}{3}d)$,and this results in a UDV network structure ($d$-$m$-$c$) of 512-341-10 for MaxViT-T, 1280-853-10 for EfficientNet-B0, and 2520-1680-10 for RegNetX-32GF.

Four optimizers, Adam, NAdam, MBGD, and MBGDM, were applied to all tasks. 
Specifically, for regression tasks, we tested learning rates of ($10^{-4}$, $10^{-3}$, $10^{-2}$, $10^{-1}$, $1$, $2$ $3$). 
Larger learning rates of ($1$, $2$, $3$) were often excluded for the UV model due to divergence. 
For classification tasks, we tested learning rates of ($10^{-6}$, $10^{-5}$, $10^{-4}$, $10^{-3}$, $10^{-2}$, $10^{-1}$, $1$) with Adam and NAdam, and ($10^{-3}$, $10^{-2}$, $10^{-1}$, $1$, $2$, $3$, $5$) with MBGD and MBGDM.

Model performances are summarized in \Cref{tab:full_results_adam,tab:full_results_nadam,tab:full_results_MBGD,tab:full_results_mbgdm}, and \Cref{appendix-fig:SVD_examples_reg,appendix-fig:SVD_examples_class} illustrate that UDV consistently converges to a low-rank solution.
Generally, UDV exhibits the most pronounced decaying pattern in singular values. 
Additionally, we extended our experiments to include full-batch training on the HPART and NYCTTD datasets. 
Although full-batch training converges more slowly than stochastic (mini-batch) methods, it exhibits a similar singular value decay pattern, confirming our earlier observations.

\begin{figure}[!htbp] 
  \centering
  \includegraphics[width=0.95\textwidth]{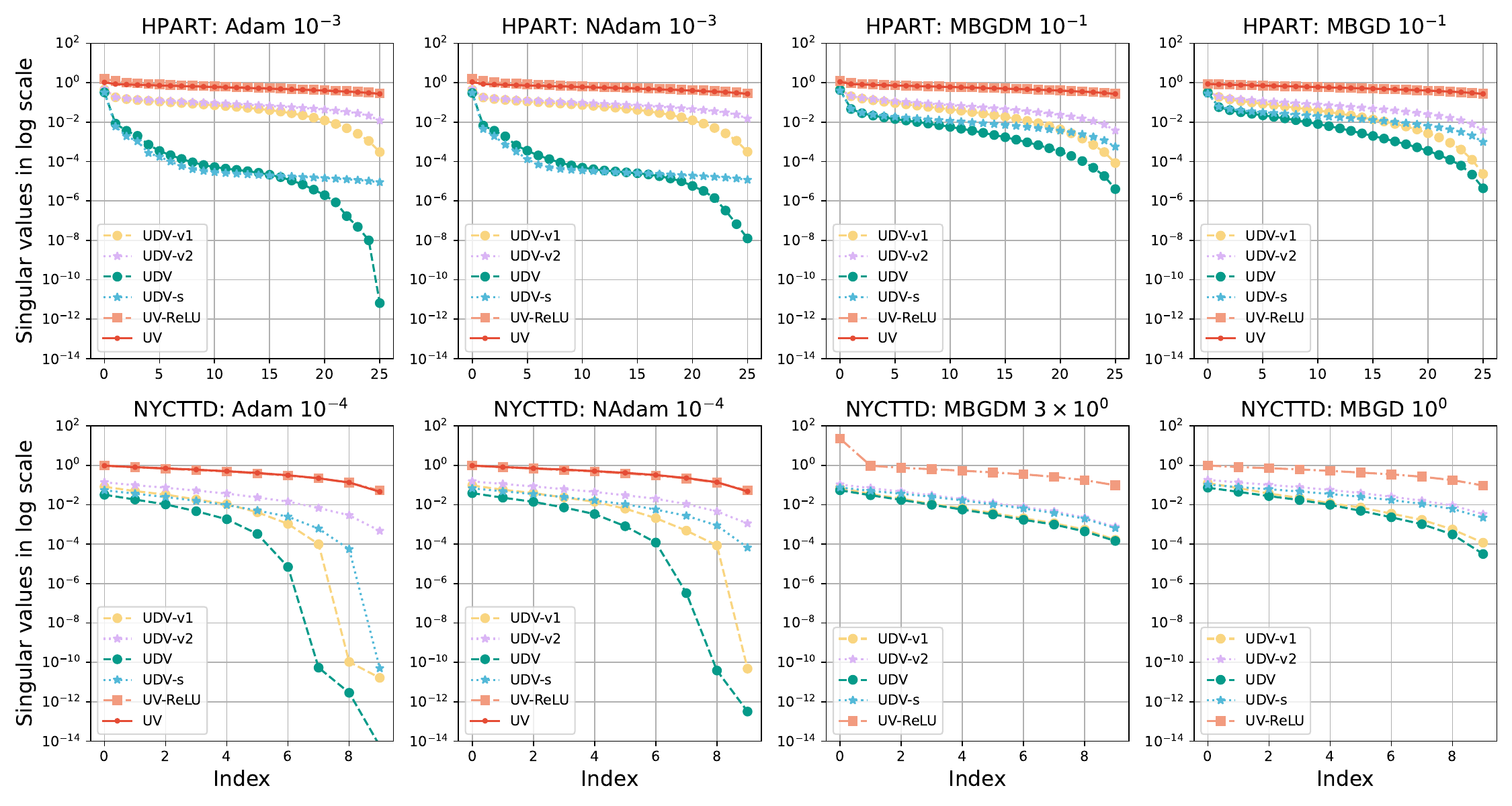}
  \captionsetup{font=small}
  \caption{Comparison of singular value pattern among all UDV variants, UV-ReLU and UV on the regression tasks.}
  \label{appendix-fig:SVD_examples_reg}
\end{figure}

\begin{figure}[p] 
  \centering
  \includegraphics[width=0.95\textwidth]{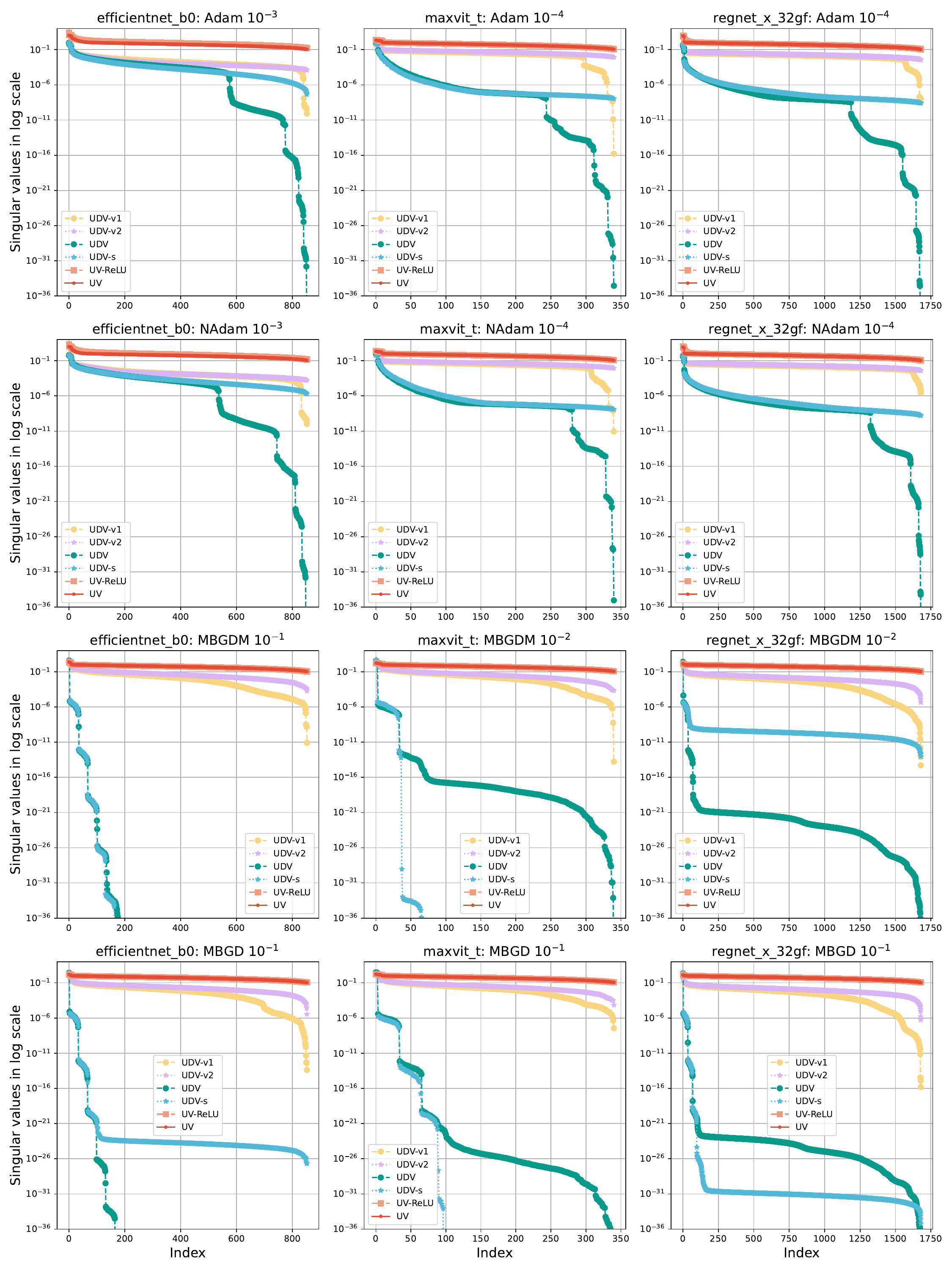}
  \captionsetup{font=small}
  \caption{Comparison of singular value pattern among all UDV variants, UV-ReLU and UV on the MNIST dataset.}
  \label{appendix-fig:SVD_examples_class}
\end{figure}

\subsection{Additional Details on the Pruning Experiment}
\label{app:pruning}

The SVD-based pruning experiments yield conclusions consistent with the singular value decay pattern.
UDV typically exhibits the fastest decay, leading to more compact models while maintaining performance, as shown in \Cref{appendix-fig:Pruning_Performance}.

\begin{figure}[!htbp] 
  \centering
  \includegraphics[height=0.85\textheight]{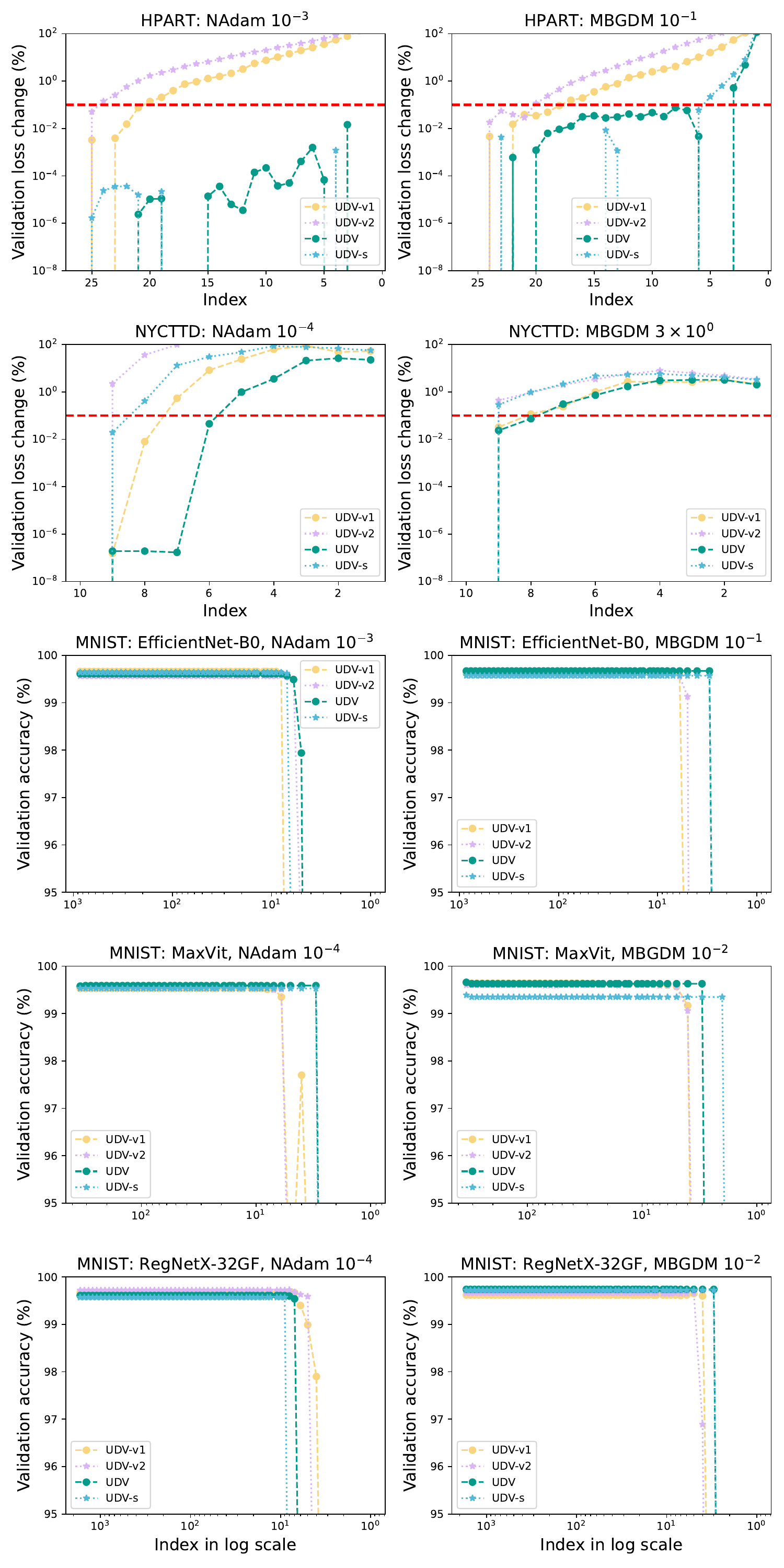}
  \captionsetup{font=small}
  \caption{The performance of SVD-based pruning. The index in x-axis represents the number of neurons in the diagonal layer after pruning. For the HPART and NYCTTD datasets, the test loss change indicates how much worse the pruned model performs compared to the baseline (the model before pruning), expressed as a percentage ($\frac{loss_{pruned}-loss_{baseline}}{loss_{baseline}}\times100\%$). Note that the pruned model may outperform the baseline, but negative values cannot be displayed on a logarithmic scale. The red dashed line denotes the 0.1\% threshold, indicating negligible performance sacrifice. For the MNIST dataset, the results show the test accuracy after pruning the model.}
  \label{appendix-fig:Pruning_Performance}
\end{figure}

We observed that the learning rate can have some impact on the singular value decay pattern. 
In particular, a very large learning rate may cause oscillations in both training and test loss, yet may result in a rapid decay of the spectrum. 
Conversely, a small learning rate may lead to a less pronounced spectral decay but still can yield a comparable test loss to that of the optimal learning rates. 
For example, the difference in test losses between the Adam optimizer with learning rates of $10^{-3}$ and $10^{-4}$ was negligible, yet the singular value spectra differed (see \Cref{appendix-fig:SVD_examples_sensitive}). 
This discrepancy also affects the performance of the pruning experiments (see \Cref{appendix-fig:Pruning_Sensitive}). 

\begin{figure}[!tbp] 
  \centering
  \includegraphics[width=0.85\textwidth]{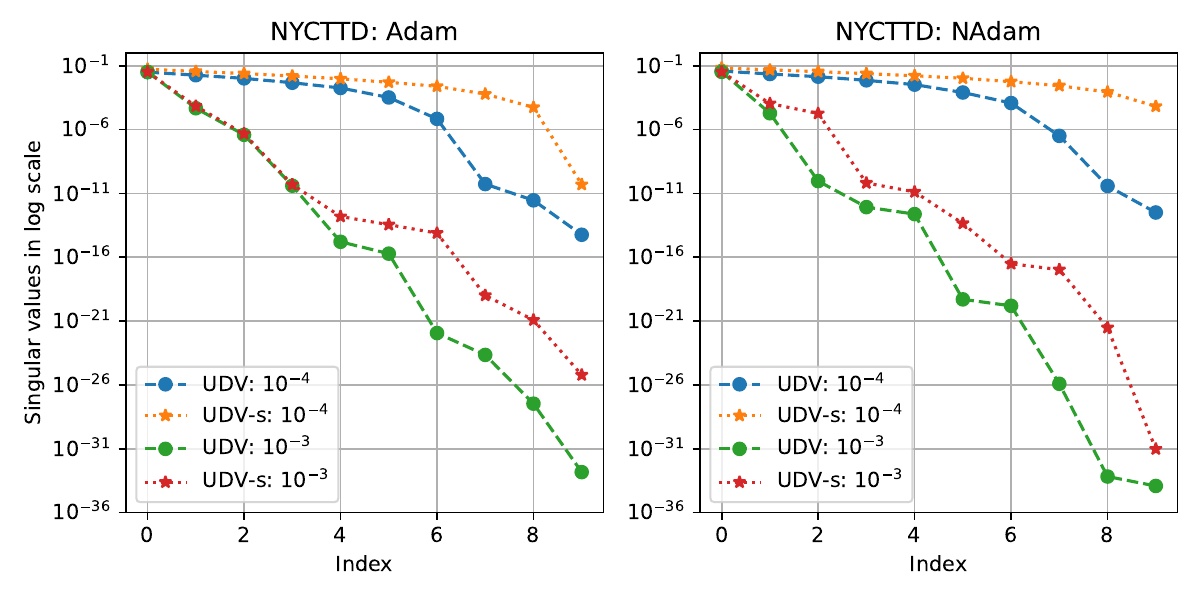}\vspace{-0.75em}
  \captionsetup{font=small}
  \caption{Differences in singular value patterns across varying learning rates.}\vspace{0.5em}
  \label{appendix-fig:SVD_examples_sensitive}
\end{figure}

\begin{figure}[!tbp] 
  \centering
  \includegraphics[width=0.95\textwidth]{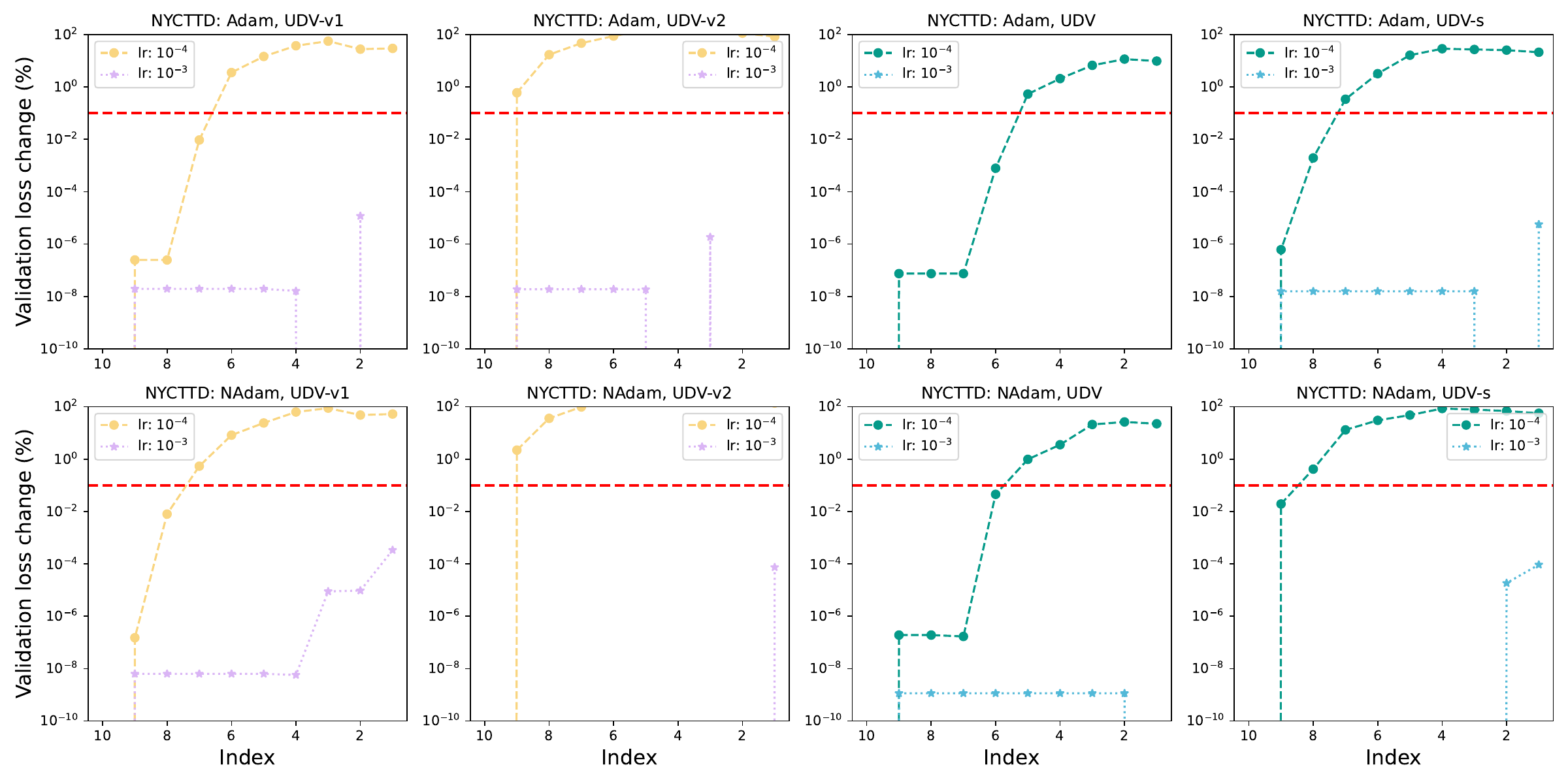}
  \captionsetup{font=small}
  \caption{Different learning rates lead different pruning performance.  the test loss change indicates how much worse the pruned model performs compared to the baseline (the model before pruning), expressed as a percentage ($\frac{loss_{pruned}-loss_{baseline}}{loss_{baseline}}\times100\%$). Note that the pruned model may outperform the baseline, but negative values cannot be displayed on a logarithmic scale. The red dashed line denotes the 0.1\% threshold, indicating negligible performance sacrifice.}
  \label{appendix-fig:Pruning_Sensitive}
\end{figure}

When the learning rate is close to the optimal value, a `stair-step' pattern (see \Cref{appendix-fig:StairStep}) can be observed on the loss or accuracy curve in the classification task.
It could be attributed to the model continuously searching for a low-rank solution, increasing the rank gradually.

\begin{figure}[!tbp] 
  \centering
  \includegraphics[width=0.75\textwidth]{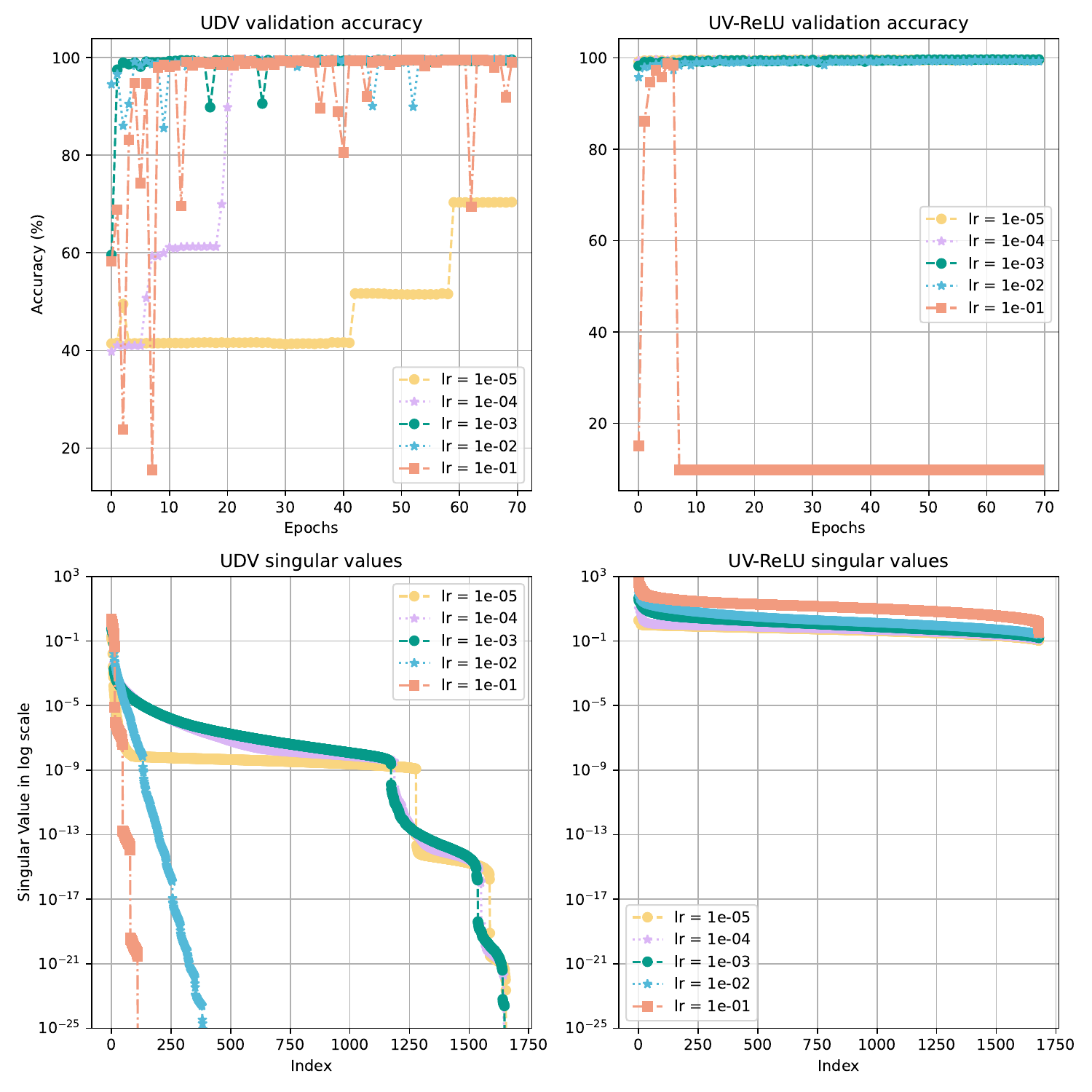}
  \captionsetup{font=small}
  \caption{The `stair-step' accuracy curve in classification (RegNetX-32GF with Adam). The sub-figures in the second row correspond to the respective singular value spectra, indicating that UDV consistently seeks low-rank solutions regardless of the learning rate.}
  \label{appendix-fig:StairStep}
\end{figure}

\subsection{Incorporating ReLU into UDV}
\label{app:udv-relu}

To explore whether non-linear activation functions, particularly ReLU, can be used in our UDV framework, we conducted an additional experiment. 

Building on the standard ReLU, we first introduced norm constraints, $\sum_{j = 1}^m \|\mathbf{u}_j \|_2^2 \leq 1$ and $\sum_{j = 1}^m \|\mathbf{v}_j \|_2^2 \leq 1$, to the layers both before and after ReLU, denoted as \textit{ReLU(constrained)}.
Next, we integrated ReLU into the diagonal layer of UDV. 
When the non-negativity constraints are applied on the diagonal layer, we refer to the model as UDV-ReLU; otherwise, it is called UDV-ReLU-s.

We conducted a preliminary experiment to verify the feasibility of incorporating ReLU into the UDV, leaving a comprehensive study to future work. 
We derived the optimizers and learning rates from \Cref{tab:full_results_adam,tab:full_results_nadam,tab:full_results_MBGD,tab:full_results_mbgdm}, but only used the HPART dataset for regression and the MNIST for classification.

\Cref{appendix-fig:UDV-ReLU-Reg} and \Cref{appendix-fig:UDV-ReLU-Class} showed that UDV-ReLU and UDV-ReLU-s exhibit a similar (or even more pronounced) decaying pattern in singular values as observed in the proposed UDV.
The `stair-step' accuracy curve was also observed in these models.

\begin{figure}[!t]
  \centering
  \includegraphics[width=0.95\textwidth]{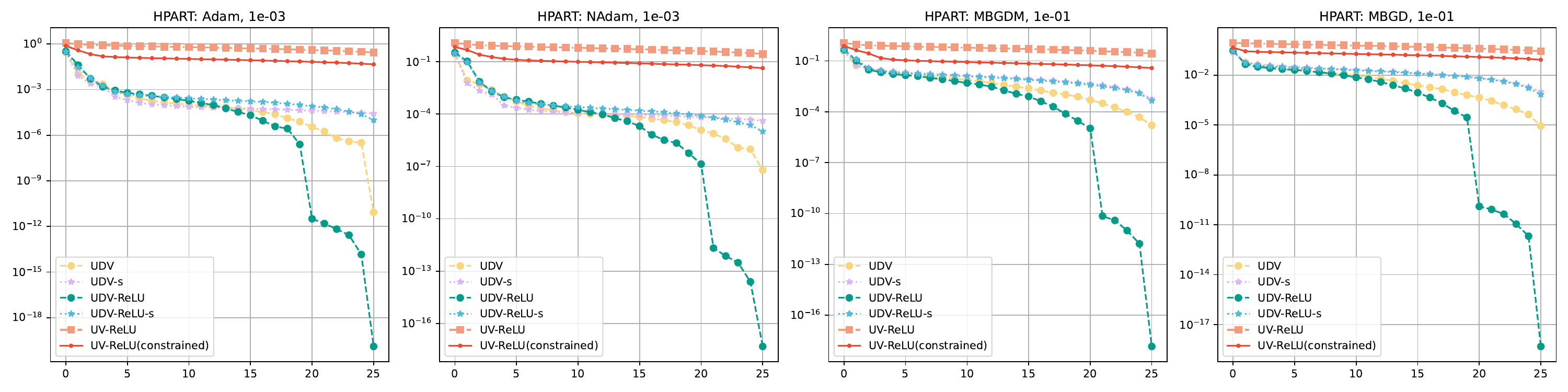}
  \captionsetup{font=small}
  \caption{ReLU in UDV: Singular value pattern on HPART dataset.}
  \label{appendix-fig:UDV-ReLU-Reg}
\end{figure}
\begin{figure}[!t]
  \centering
  \includegraphics[width=0.95\textwidth]{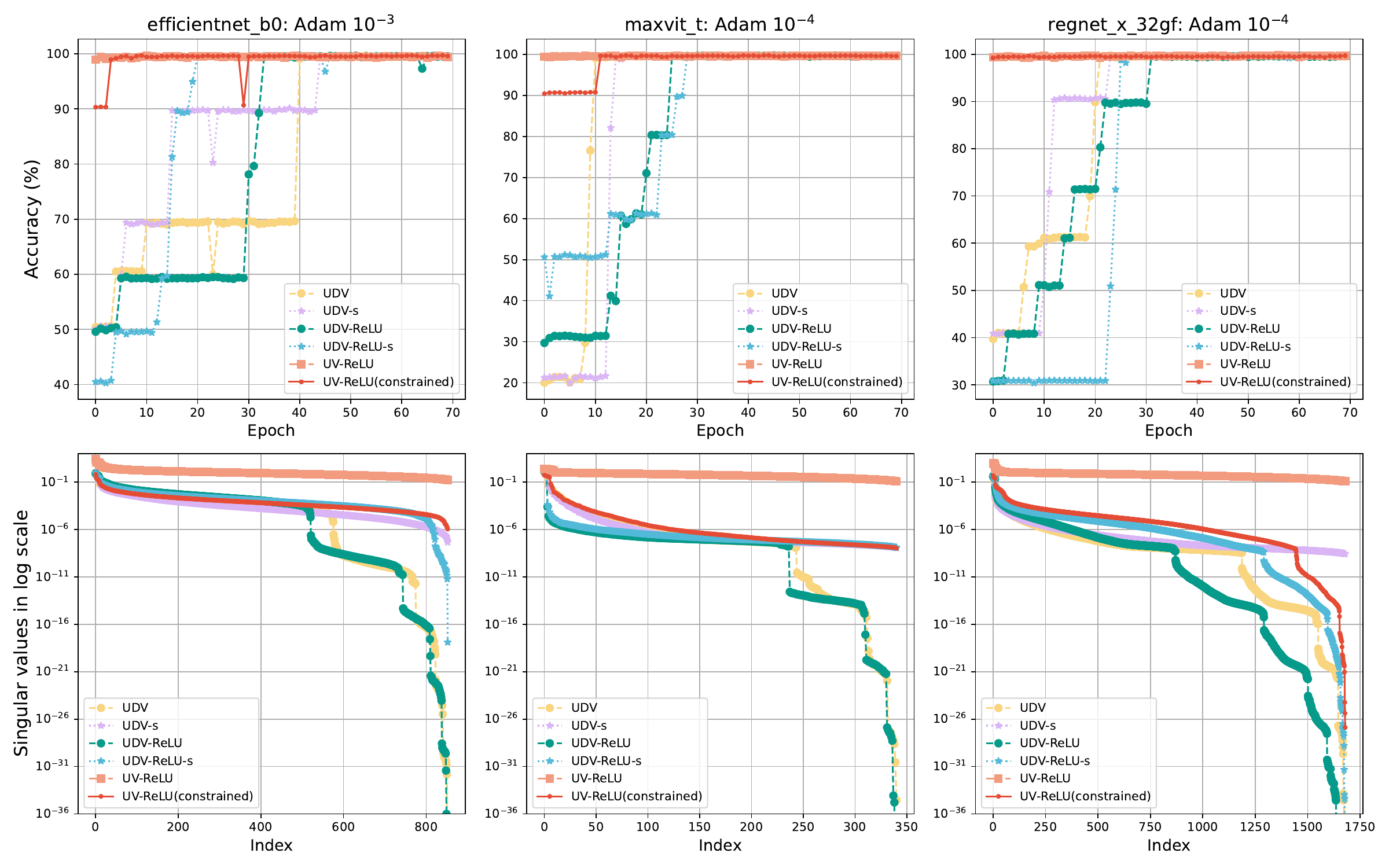}
  \captionsetup{font=small}
  \caption{ReLU in UDV: test accuracy (the first row) and the singular value pattern (the second row) on MNIST dataset. Faster convergence can be achieved by fine-tuning the learning rate.}
  \label{appendix-fig:UDV-ReLU-Class}
\end{figure}

\subsection{Comparison of UDV and Three-Layer Fully Connected Networks}
\label{app:UFV}

We extend the experiments to compare the UDV framework with three-layer fully connected neural networks. 
To highlight the critical role of the UDV structure and its constraints, we include the following models in the comparison: the proposed UDV model, a model based on the UDV structure but without the constraints (UDV unconstrained), a fully connected three-layer neural network (UFV), and the standard UV model as the baseline.

The results, also included in \Cref{tab:full_results_adam,tab:full_results_nadam,tab:full_results_MBGD,tab:full_results_mbgdm}, are shown in \Cref{appendix-fig:compare_UFV_reg} and \Cref{appendix-fig:compare_UFV_class}. 
These findings demonstrate that the singular value decay in UDV unconstrained, UFV, and UV is significantly slower than in the proposed UDV framework.  
This underscores the critical role of the diagonal layer and constraints in facilitating low-rank solution identification within the proposed structure.

\begin{figure}[!htbp] 
  \centering
  \includegraphics[clip, trim=0cm 0cm 0cm 0cm, width=0.95\textwidth]{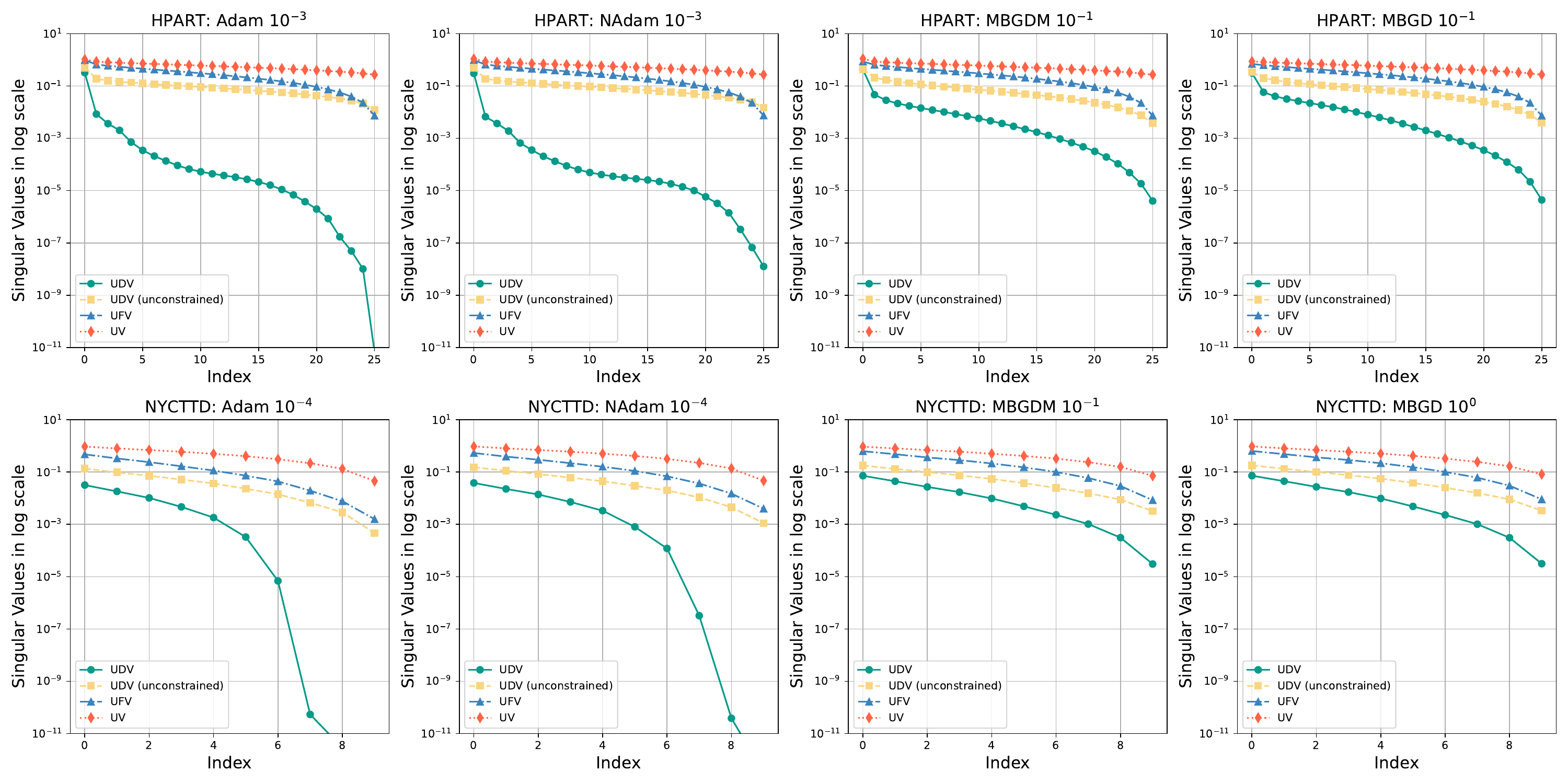}
  \captionsetup{font=small}
  \caption{Comparison of singular value spectrum among UDV, UDV (unconstrained),UFV and UV on the regression tasks.}
  \label{appendix-fig:compare_UFV_reg}
\end{figure}

\begin{figure}[!htbp] 
  \centering
  \includegraphics[clip, trim=0cm 0cm 0cm 0cm, width=0.95\textwidth]{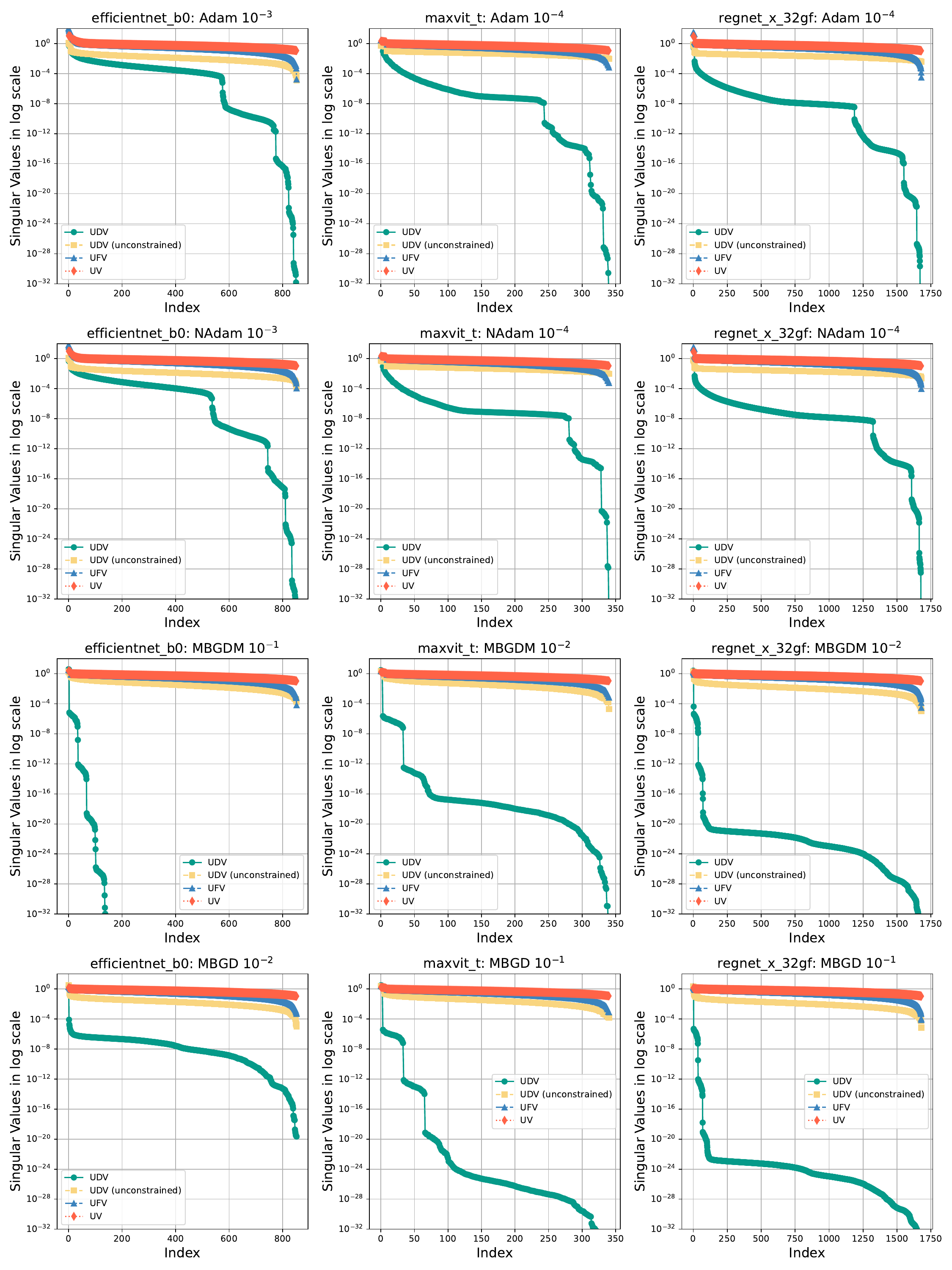}
  \captionsetup{font=small}
  \caption{Comparison of singular value spectrum among UDV, UDV (unconstrained), UFV and UV on the classification datasets (MNIST).}
  \label{appendix-fig:compare_UFV_class}
\end{figure}

\subsection{Comparing the Effects of UDV and Weight Decay Regularization}
\label{app:weight_decay}

We conducted a preliminary experiment to compare the singular value decay in two-layer (UV) and three-layer (UFV) fully connected network blocks with weight decay regularization to that observed in the UDV framework.  
We focused on the regression task with the HPART dataset.  
For the regularization parameter ($\gamma$), we tested values $10^{-6}$, $10^{-5}$, $\dots$, $10^{-1}$.  

Increasing $\gamma$ led to a faster decay in the singular value spectrum but at the cost of higher training and test losses.  
\Cref{appendix-fig:weight_decay} illustrates an example where we selected the largest $\gamma$ values for which the test loss remained within 10\% of the UDV baseline.  
The resulting singular value decay in the UV and UFV models was less pronounced than in the UDV model.  
Further increasing $\gamma$ to match or exceed the spectral decay of the UDV model resulted in significantly worse test loss compared to the UDV network.  
The complete results are presented in \Cref{appendix-fig:weight_decay_UFV,appendix-fig:weight_decay_UV}.

\begin{figure}[!htbp] 
  \centering
  \includegraphics[clip, trim=0cm 0cm 0cm 0cm, width=\textwidth]{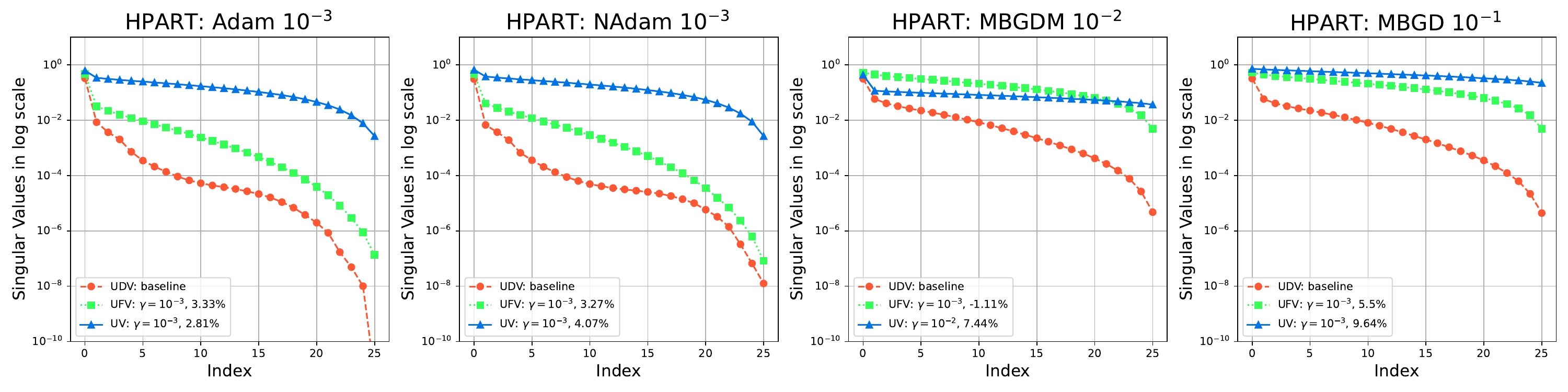}
  \captionsetup{font=small}
  \caption{Comparison of singular value spectrum among the proposed UDV, UFV and UV (with weight decay) on the HPART dataset. In addition to the UDV baseline, the UFV and UV models include the regularization parameter $\gamma$ and the relative change in test performance. The relative change is defined as ($\frac{loss_{model}-loss_{baseline}}{loss_{baseline}}\times 100\%$), representing how much worse the model performs compared to the UDV baseline.}
  \label{appendix-fig:weight_decay}
\end{figure}

\begin{figure}[!htbp] 
  \centering
  \includegraphics[clip, trim=0cm 0cm 0cm 0cm, width=\textwidth]{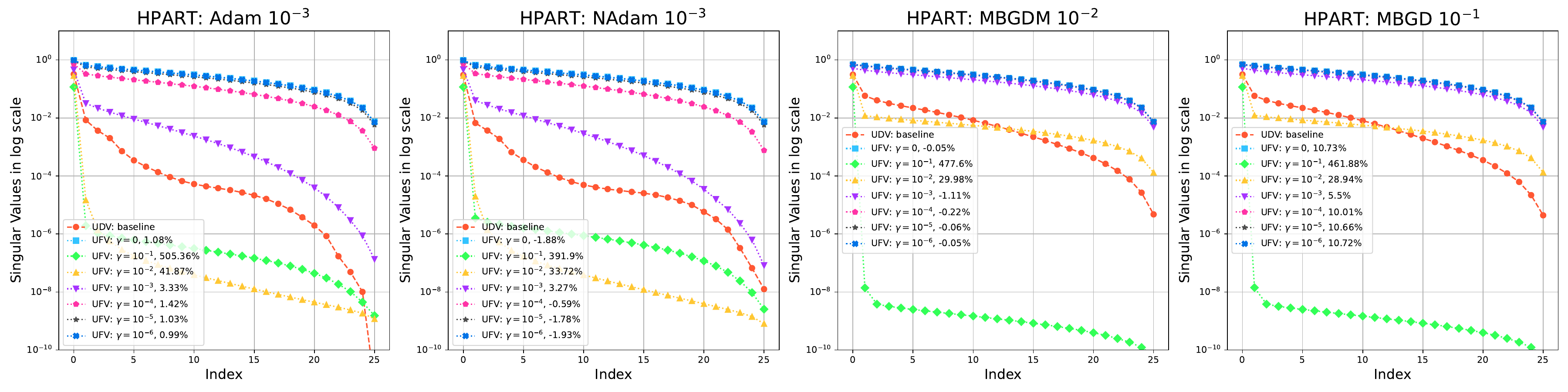}
  \captionsetup{font=small}
  \caption{Extension of \Cref{appendix-fig:weight_decay}. Comparison of singular value spectrum between the proposed UDV and UFV (with weight decay) on the HPART dataset. Regularization parameter $\gamma = 0$ indicates that no weight decay is applied.}
  \label{appendix-fig:weight_decay_UFV}
\end{figure}

\begin{figure}[!htbp] 
  \centering
  \includegraphics[clip, trim=0cm 0cm 0cm 0cm, width=\textwidth]{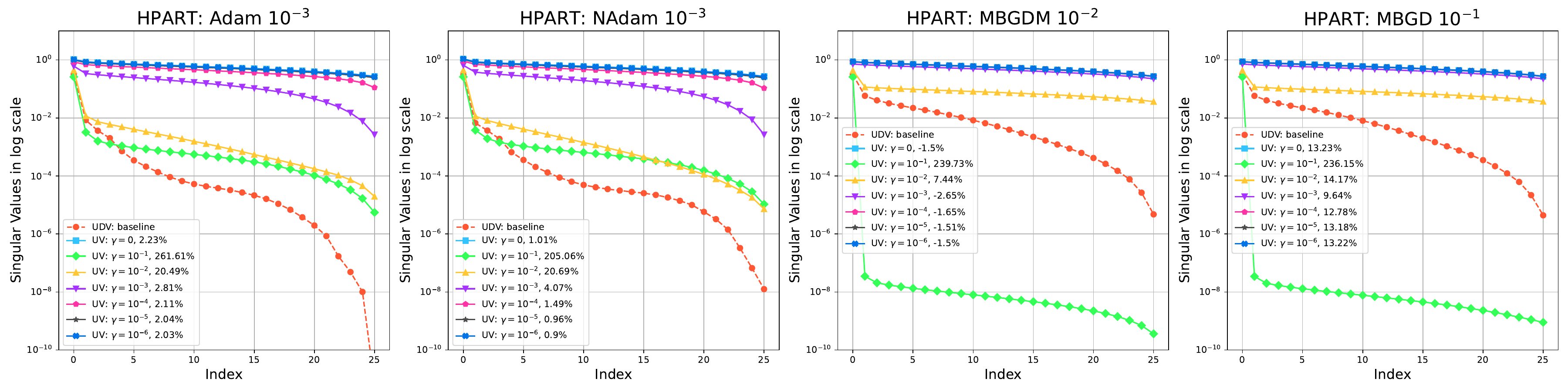}
  \captionsetup{font=small}
  \caption{Extension of \Cref{appendix-fig:weight_decay}. Comparison of singular value spectrum between the proposed UDV and UV (with weight decay) on the HPART dataset. Regularization parameter $\gamma = 0$ indicates that no weight decay is applied.}
  \label{appendix-fig:weight_decay_UV}
\end{figure}

\clearpage
\newpage

\subsection{A Toy Example of Causal Language Modeling Using a Pre-trained LLaMA-2 Model with LoRA-Based UDV Fine-Tuning}
\label{app:llm_task}

To further broaden the application scope of the proposed method, we integrate UDV into the pre-trained model using the LoRA framework and evaluate its effectiveness on the causal language modeling task (i.e.,  a generative task where a model predicts the next token in a sequence based solely on past tokens).
Experiments are conducted using the LLaMA-2-7B model \citep{touvron2023llama} and the WikiText-2 dataset \citep{merity2016pointer}, which is a common benchmark test.

LoRA (Low-Rank Adaptation) is a parameter-efficient fine-tuning technique that inserts trainable low-rank matrices into pre-trained models \citep{hu2022lora}. 
By freezing the original model weights and only training a small number of additional parameters, LoRA enables efficient adaptation to downstream tasks while significantly reducing computational and storage costs.
Specifically, LoRA introduces trainable low-rank matrices $B\in \mathbb{R}^{d\times r}$ and $A\in \mathbb{R}^{r\times k}$ as a side path to the frozen pre-trained weight matrix $W_0$, and expresses the forward pass as $h=W_0x+BAx$, where $x$ is the input.
Here we consider $d=k$ for simplicity.

By inserting a diagonal matrix $D$ between the matrices $B$ and $A$, the LoRA structure naturally extends to the LoRA-based UDV structure, as illustrated in \Cref{appendix-fig:UDV_lora}.
After the weight update of $U, D$ and $V$, constraints described in \Cref{eqn:extra-layer-relu} are applied to the LoRA-based UDV structure.
In practice, the diagonal matrix can be efficiently implemented using a vector $d\in\mathbb{R}^{1\times r}$, further reducing computational and memory overhead.
Consequently, the LoRA-based UDV introduces a negligible increase in the number of trainable parameters compared to the original LoRA structure, as it adds only $r$ parameters per LoRA layer replacement.

\begin{figure}[!htbp] 
  \centering
  \includegraphics[clip, trim=0cm 0cm 0cm 0cm, width=0.95\textwidth]{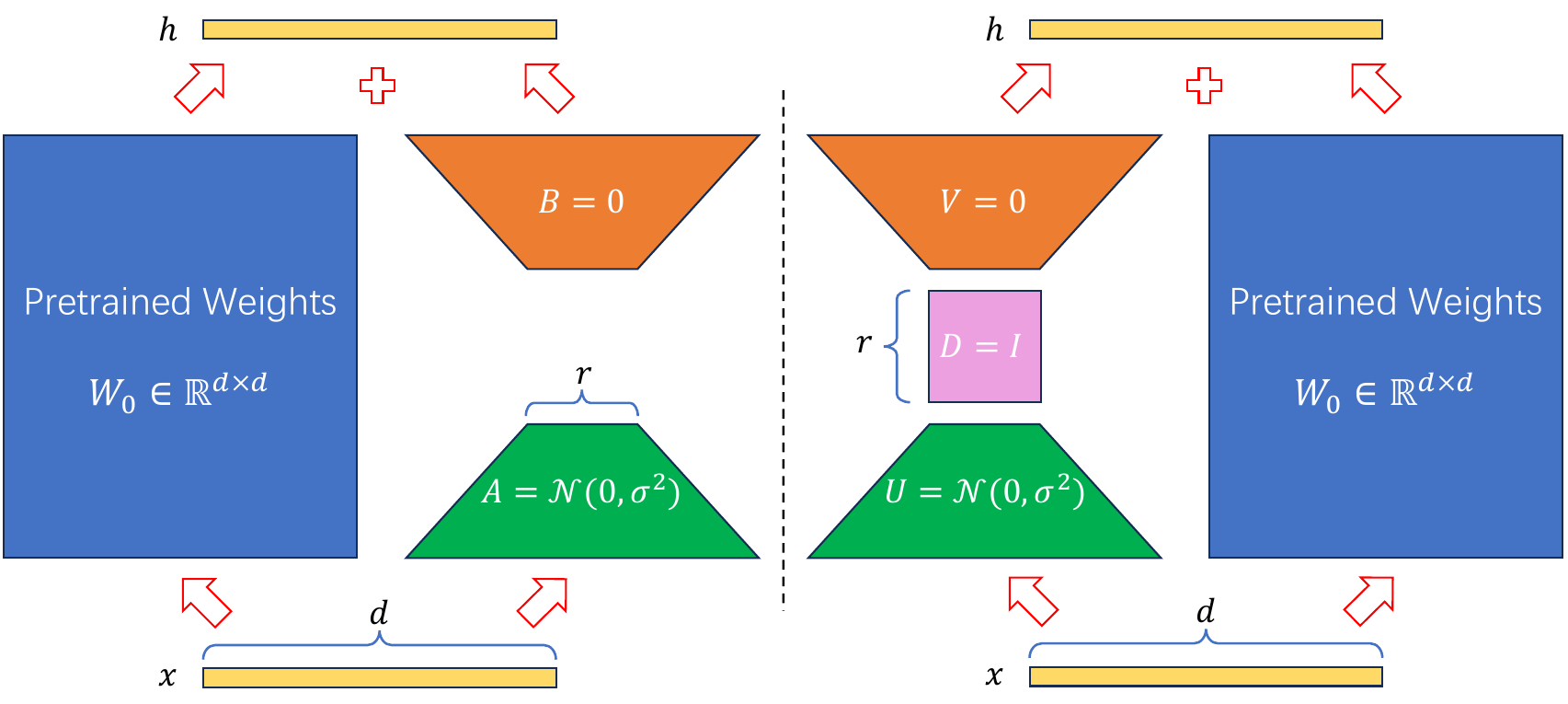}
  \captionsetup{font=small}
  \caption{LoRA forward pass (Left) and LoRA-based UDV forward pass (Right). The inputs and pre-trained weights are identical in both cases. Matrices with the same fill color indicate matching shapes.}
  \label{appendix-fig:UDV_lora}
\end{figure}

We initialized both the LoRA and LoRA-based UDV structures from the same starting point: $B=V=0$, $A=U=\mathcal{N}(0, 0.01)$ and $D = I$, where $I$ is the identity matrix.
Mini-batch gradient descent (MBGD) with learning rates $lr = 0.001$,  using an effective batch size of 32, was applied to better observe the effects of the imposed constraints.
Following the standard LoRA setup, only the linear layers associated with the Query ($Q$) and Value ($V$) projections in the attention mechanism were replaced with LoRA and LoRA-based UDV structures.
We evaluated the models using rank values $r \in \{4,8,16,32,64\}$ , where the dimensionality of $d$ depends on the shape of the original layer.
Both structures were fine-tuned for 30 epochs, and then pruning was applied as described in \Cref{sec:SVD-based pruning}. 

Instead of applying uniform pruning across all LoRA-based UDV structures, we set several thresholds on the retained energy—defined as the sum of the squared singular values—to dynamically prune each UDV block based on its individual importance.
Note that the pruning can be applied to the LoRA structure in the same manner as the LoRA-based UDV structure by temporarily inserting an frozen identity matrix as a diagonal layer.

Perplexity (i.e., the exponential of the average negative log-likelihood) is used as the primary evaluation metric for this experiment, and lower perplexity indicates better language modeling performance.
Furthermore, the pruning performance was demonstrated by the perplexity change in terms of the compression ratio (i.e., the ratio between the number of trainable parameters of the compressed structure and the original structure)

Since the hard constraints slower the convergence, and its fine-tuning performance is slightly lower than that of the original LoRA structure within the same number of epochs, as shown in \Cref{appendix-fig:LoRA_ppl_5} shown.
However, the pruning results (see \Cref{appendix-fig:LoRA_pruning_5}) demonstrate that the LoRA-based UDV is significantly more robust to high compression ratios. 
Even at a compression ratio of $0.1$, the LoRA-based UDV is able to maintain performance effectively.

\begin{figure}[!htbp] 
  \centering
  \includegraphics[clip, trim=0cm 0cm 0cm 0cm, width=0.95\textwidth]{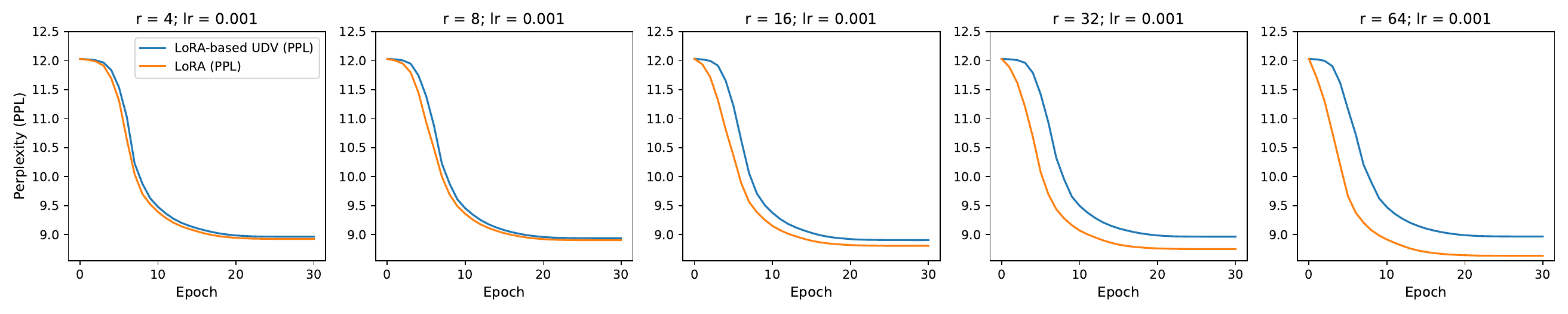}
  \captionsetup{font=small}
  \caption{Comparison of perplexity between LoRA and LoRA-based UDV during training. The perplexity at epoch $0$ reflects the performance of the pre-trained model before any fine-tuning.}
  \label{appendix-fig:LoRA_ppl_5}
\end{figure}

\begin{figure}[!htbp] 
  \centering
  \includegraphics[clip, trim=0cm 0cm 0cm 0cm, width=0.95\textwidth]{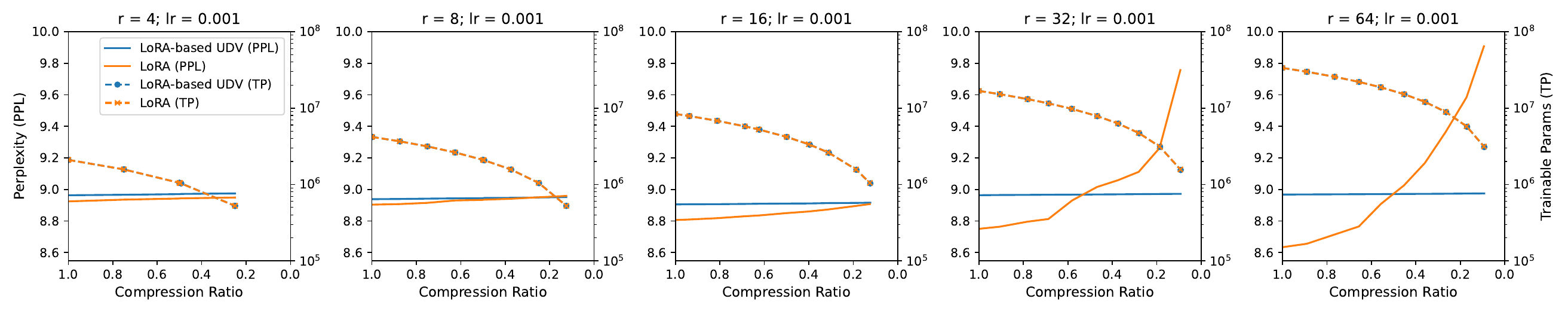}
  \captionsetup{font=small}
  \caption{Comparison of pruning performance between LoRA and LoRA-based UDV during training.}
  \label{appendix-fig:LoRA_pruning_5}
\end{figure}

This toy example is not intended to present a mature pruning method, but rather to highlight the potential of the proposed UDV block, which can be applied to a wide range of models and demonstrates promising pruning performance.

\section{Extended Discussion on the Related Work}

\textbf{Burer-Monteiro factorization.}
BM factorization was proposed for solving semidefinite programs \citep{burer2003nonlinear, 01_BM_first2, boumal2016non, boumal2020deterministic, cifuentes2021burer} and has been recognized for its efficiency in addressing low-rank optimization problems, e.g.,~\citep{sun2016guaranteed, bhojanapalli2016dropping, park2017non, park2018finding, hsieh2018non, sahin2019inexact, 01_BM_solving2, 01_BM_solving}. Building on the connections between training problems for (non-linear) two-layer neural networks and convex (copositive) formulations (see \citep{pilanci2020neural,ergen2020implicit,sahiner2021vector} and the references therein), \citet{01_closeWork} recently introduced BM factorization to solve convex formulations of various architectures, including fully connected networks with ReLU and gated ReLU \citep{01_GaLU} activations, two-layer convnets~\citep{01_CNN}, and self-attention mechanisms based on Transformers \citep{01_attention}. 

\textbf{Implicit regularization.} 
One promising line of research that aims to explain the successful generalization abilities of neural networks follows the concept of `implicit regularization,' which results from the optimization methods and formulations used during training \citep{neyshabur2014search, neyshabur2017exploring, 01_IRandNN}. 
Several studies explore matrix factorization to investigate implicit bias \citep{gunasekar2017implicit, arora2018stronger, razin2020implicit, belabbas2020implicit,li2021towards}. 
\citet{01_IRhelpGD} extend these results to the final linear layers of nonlinear ReLU-activated feedforward networks with fully connected layers and skip connections. 
More recently \citet{01_RelationMFIR} investigated implicit regularization in ReLU networks. 
Much of the existing work focuses on gradient flow dynamics in the limit of infinitesimal learning rates. 
In particular, \citet{gidel2019implicit} examined discrete gradient dynamics in two-layer linear neural networks, showing that the dynamics progressively learn solutions of reduced-rank regression with a gradually increasing rank. 

\textbf{Constrained neural networks.}
Regularizers are frequently used in neural network training to prevent overfitting and improve generalization \citep{01_regularization}, or to achieve structural benefits such as sparse and compact network architectures \citep{scardapane2017group}. 
However, it is conventional to apply these regularizers as penalty functions in the objective rather than as hard constraints, addressing them within gradient-based optimization via their (sub)gradients. 
This approach is likely favored due to the ease of implementation, as pre-built functions are readily available in common neural network packages. 
Several independent studies have also applied proximal methods across different networks and applications, finding that proximal-based methods tend to yield solutions with more pronounced structures \citep{bai2019proxquant, yang2020proxsgd, yun2021adaptive, yurtsever2021three, yang2022pathprox}. 
Structural regularization in the form of hard constraints, however, appears to be rare in neural network training. 
One notable exception is in the context of neural network training with the Frank-Wolfe algorithm \citep{pokutta2020deep, zimmer2022compression, macdonald2022interpretable}.
Recently, \citet{pethick2025training} revealed parallels between Frank-Wolfe on constrained networks and algorithms that post-process update steps, such as Muon \citep{jordan2024muon}, which achieves state-of-the-art results on nanoGPT by orthogonalizing the update directions before applying them. 

There are other examples of constraints in neural networks. For instance, constraints can be applied to ensure adherence to real-world conditions in applications where they are necessary \citep{pathak2015constrained, marquez2017imposing, jia2017constrained, kervadec2019constrained, weber2021constrained}, or to incorporate physical laws in physics-informed neural networks \citep{01_PINN, lu2021physics, patel2022physics}. 
In these cases, the feasible set is typically complex and difficult to project onto; therefore, optimization algorithms like primal-dual methods or augmented Lagrangian techniques are used. 
Constraints are also present in lifted neural networks, a framework where the training problem is reformulated in a higher-dimensional `lifted' space. In this space, conventional activation functions like ReLU, used in the original formulation, are expressed as constraints \citep{askari2018lifted,sahiner2021vector,bartan2021neural,prakhya2024convex}. 

\textbf{Pruning.}
Neural networks are overparameterized, which can enhance generalization and avoid poor local minima. But such models then suffer from excessive memory and computational demands, making them less efficient for deployment in real-world applications \citep{00_overPbenefits}. 
Several compression techniques have been studied in the literature, including parameter quantization \citep{00_quantizing}, knowledge distillation \citep{00_knowledgeDistill}, and pruning.

Pruning reduces the number of parameters, resulting in more compact and efficient models that are easier to deploy. 
The literature on pruning is extensive, with diverse methods proposed with distinct characteristics. 
Various criteria are used to determine which weights to prune, including second-order derivative-based methods \citep{lecun1989optimal,hassibi1992second}, magnitude-based pruning \citep{janowsky1989pruning,han2015learning}, saliency heuristics \citep{mozer1988skeletonization,lee2018snip}, and matrix or tensor factorization-based techniques \citep{00_EarlyMD02,sainath2013low,jaderberg2014speeding,lebedev2015speeding,swaminathan2020sparse}, among others. 
Pruning by singular value thresholding has shown promising results, particularly in natural language processing \citep{chen2021drone}, and is often used along with various enhancements such as importance weights and data whitening for effective compression of large language models \citep{hsu2022language, yuan2023asvd, wang2024svd}. 
A comprehensive review of pruning methods is beyond the scope of this paper due to space limitations and the diversity of approaches. 
For more detailed reviews, we refer to \citep{reed1993pruning,00_purningSurvey,cheng2024survey}, and the references therein. 

\clearpage
\newpage

\begin{table}[]
\centering
\captionsetup{font=small}
\caption{Experiments using Adam optimizer. Not applicable or results with obvious oscillations or divergence are denoted as `--'.}
\label{tab:full_results_adam}
\resizebox{0.95\textwidth}{!}{
\begin{tabular}{clllll}
\hline
Tasks &
  \multicolumn{2}{c}{Regression (Test Loss)} &
  \multicolumn{3}{c}{Classification (Test Accuracy)} \\ \hline
\begin{tabular}[c]{@{}c@{}}Dataset \\ (Transferred model)\\ \textnormal{ }\end{tabular} &
  \multicolumn{1}{c}{\begin{tabular}[c]{@{}c@{}}HPART \\ --\\ ($\times 10^{-3}$)\end{tabular}} &
  \multicolumn{1}{c}{\begin{tabular}[c]{@{}c@{}}NYCTTD \\ --\\ ($\times 10^{-6}$)\end{tabular}} &
  \multicolumn{3}{c}{\begin{tabular}[c]{@{}c@{}}MNIST\\  (MaxVit-T [M] $|$ EfficientNet-B0 [E] $|$ RegNetX-32GF [R]) \\ ($\times 100\%$)\end{tabular}} \\ \hline
LR: $10^{-6}$ &
  \multicolumn{1}{c}{--} &
  \multicolumn{1}{c}{--} &
  \begin{tabular}[c]{@{}l@{}}UDV: --\\ UDV-s: --\\ UDV-v1: 99.56\\ UDV-v2: 99.53\\ UDV-ReLU: --\\ UDV(unconstrained): 99.53\\ UFV: 99.53\\ UV-ReLU: 99.59 \\ UV: 99.54\\ M: 99.42\end{tabular} &
  \begin{tabular}[c]{@{}l@{}}UDV: --\\ UDV-s: --\\ UDV-v1: --\\ UDV-v2: --\\ UDV-ReLU: --\\ UDV(unconstrained): --\\ UFV: 99.25\\ UV-ReLU: 99.29\\ UV: 99.24\\ E: 98.67\end{tabular} &
  \begin{tabular}[c]{@{}l@{}}UDV: --\\ UDV-s: --\\ UDV-v1: 99.46 \\ UDV-v2: 99.48\\ UDV-ReLU: --\\ UDV(unconstrained): 99.48\\ UFV: 99.38\\ UV-ReLU: 99.41\\ UV: 99.36\\ R: 99.32\end{tabular} \\ \hline
LR: $10^{-5}$ &
  \multicolumn{1}{c}{--} &
  \multicolumn{1}{c}{--} &
  \begin{tabular}[c]{@{}l@{}}UDV: --\\ UDV-s: --\\ UDV-v1: 99.61\\ UDV-v2: 99.58\\ UDV-ReLU: --\\ UDV(unconstrained): 99.62\\ UFV: 99.60\\ UV-ReLU: 99.57 \\ UV: 99.61\\ M: 99.59\end{tabular} &
  \begin{tabular}[c]{@{}l@{}}UDV: --\\ UDV-s: --\\ UDV-v1: 99.49\\ UDV-v2: 99.53\\ UDV-ReLU: --\\ UDV(unconstrained): 99.54\\ UFV: 99.43\\ UV-ReLU: 99.52\\ UV: 99.52\\ E: 99.51\end{tabular} &
  \begin{tabular}[c]{@{}l@{}}UDV: --\\ UDV-s: --\\ UDV-v1: \textbf{99.58}\\ UDV-v2: 99.50\\ UDV-ReLU: --\\ UDV(unconstrained): 99.50\\ UFV: 99.58\\ UV-ReLU: 99.59\\ UV: 99.57\\ R: \textbf{99.59}\end{tabular} \\ \hline
LR: $10^{-4}$ &
  \begin{tabular}[c]{@{}l@{}}UDV: 2.304\\ UDV-s: 1.912\\ UDV-v1: 1.823\\ UDV-v2: 1.738\\ UDV-ReLU: 2.731\\ UDV(unconstrained): 1.738\\ UFV: 1.351\\ UV-ReLU: 1.376\\ UV: 1.475\end{tabular} &
  \begin{tabular}[c]{@{}l@{}}UDV: \textbf{5.248}\\ UDV-s: 5.248\\ UDV-v1: \textbf{5.248}\\ UDV-v2: 5.250\\ UDV-ReLU: \textbf{5.251}\\ UDV(unconstrained): 5.250\\ UFV: 5.254\\ UV-ReLU: \textbf{5.263}\\ UV: 5.275\end{tabular} &
  \begin{tabular}[c]{@{}l@{}}UDV: \textbf{99.66}\\ UDV-s: 99.58\\ UDV-v1: \textbf{99.65}\\ UDV-v2: \textbf{99.69}\\ UDV-ReLU: \textbf{99.63}\\ UDV(unconstrained): \textbf{99.66}\\ UFV: \textbf{99.69}\\ UV-ReLU: \textbf{99.63}\\ UV: \textbf{99.64}\\ M: \textbf{99.65}\end{tabular} &
  \begin{tabular}[c]{@{}l@{}}UDV: --\\ UDV-s: --\\ UDV-v1: \textbf{99.63}\\ UDV-v2: \textbf{99.58}\\ UDV-ReLU: --\\ UDV(unconstrained): \textbf{99.63}\\ UFV: \textbf{99.59}\\ UV-ReLU: \textbf{99.54}\\ UV: 99.56\\ E: 99.60\end{tabular} &
  \begin{tabular}[c]{@{}l@{}}UDV: \textbf{99.55}\\ UDV-s: \textbf{99.55}\\ UDV-v1: 99.56\\ UDV-v2: \textbf{99.68}\\ UDV-ReLU: \textbf{99.55}\\ UDV(unconstrained): \textbf{99.59}\\ UFV: \textbf{99.67}\\ UV-ReLU: 99.64\\ UV:\textbf{99.60}\\ R: 99.59\end{tabular} \\ \hline
LR: $10^{-3}$ &
  \begin{tabular}[c]{@{}l@{}}UDV: \textbf{1.304}\\ UDV-s: \textbf{1.316}\\ UDV-v1: \textbf{1.267}\\ UDV-v2: \textbf{1.268}\\ UDV-ReLU: \textbf{1.320}\\ UDV(unconstrained): \textbf{1.268}\\ UFV: \textbf{1.318}\\ UV-ReLU: \textbf{1.167}\\ UV: \textbf{1.333}\end{tabular} &
  \begin{tabular}[c]{@{}l@{}}UDV: 5.248\\ UDV-s: \textbf{5.248}\\ UDV-v1: 5.248\\ UDV-v2: \textbf{5.248}\\ UDV-ReLU: 5.251\\ UDV(unconstrained): \textbf{5.248}\\ UFV: \textbf{5.248}\\ UV-ReLU: 5.306\\ UV: \textbf{5.251}\end{tabular} &
  \begin{tabular}[c]{@{}l@{}}UDV: 99.57\\ UDV-s: \textbf{99.59}\\ UDV-v1: 99.57\\ UDV-v2: 99.58\\ UDV-ReLU: 99.51\\ UDV(unconstrained): 99.52\\ UFV: 99.59\\ UV-ReLU: 99.60\\ UV: 99.60\\ M: 99.63\end{tabular} &
  \begin{tabular}[c]{@{}l@{}}UDV: \textbf{99.55}\\ UDV-s: \textbf{99.59}\\ UDV-v1: 99.59\\ UDV-v2: 99.57\\ UDV-ReLU: \textbf{99.54}\\ UDV(unconstrained): 99.57\\ UFV: 99.54\\ UV-ReLU: 99.49\\ UV: \textbf{99.60}\\ E: \textbf{99.61}\end{tabular} &
  \begin{tabular}[c]{@{}l@{}}UDV: 99.50\\ UDV-s: 99.49\\ UDV-v1: 99.47\\ UDV-v2: 99.47\\ UDV-ReLU: 99.42\\ UDV(unconstrained): 99.58\\ UFV: 99.41\\ UV-ReLU: \textbf{99.66}\\ UV:99.46\\ R: 99.49\end{tabular} \\ \hline
LR: $10^{-2}$ &
  \begin{tabular}[c]{@{}l@{}}UDV: 1.877\\ UDV-s: 1.998\\ UDV-v1: 1.409\\ UDV-v2: 1.500\\ UDV-ReLU: 1.699\\ UDV(unconstrained): 1.402\\ UFV: 1.483\\ UV-ReLU: 1.467\\ UV: 1.430\end{tabular} &
  \begin{tabular}[c]{@{}l@{}}UDV: 5.257\\ UDV-s:5.258\\ UDV-v1: 5.256\\ UDV-v2: 5.257\\ UDV-ReLU: 5.323\\ UDV(unconstrained): 5.263\\ UFV: 7.048\\ UV-ReLU: 5.323\\ UV: 7.369\end{tabular} &
  \begin{tabular}[c]{@{}l@{}}UDV: --\\ UDV-s: --\\ UDV-v1: --\\ UDV-v2: --\\ UDV-ReLU: --\\ UDV(unconstrained): --\\ UFV: --\\ UV-ReLU: --\\ UV: --\\ M: --\end{tabular} &
  \begin{tabular}[c]{@{}l@{}}UDV: 99.34\\ UDV-s: 99.56\\ UDV-v1: 99.48\\ UDV-v2: 99.44\\ UDV-ReLU: 96.95\\ UDV(unconstrained): 99.52\\ UFV: --\\ UV-ReLU: 99.37\\ UV: 99.32\\ E: 98.42\end{tabular} &
  \begin{tabular}[c]{@{}l@{}}UDV: 99.53\\ UDV-s: 99.43\\ UDV-v1: 99.35\\ UDV-v2: 99.28\\ UDV-ReLU: 99.37\\ UDV(unconstrained): 99.10\\ UFV: --\\ UV-ReLU: 99.32\\ UV: 98.90\\ R: 99.39\end{tabular} \\ \hline
LR: $10^{-1}$ &
  \begin{tabular}[c]{@{}l@{}}UDV: 4.188\\ UDV-s: --\\ UDV-v1: --\\ UDV-v2: --\\ UDV-ReLU: 18.26\\ UDV(unconstrained): --\\ UFV: 1.745\\ UV-ReLU: 42.01\\ UV: 1.614\end{tabular} &
  \begin{tabular}[c]{@{}l@{}}UDV: 5.321\\ UDV-s: 114.6\\ UDV-v1: 23.12\\ UDV-v2: 21.22\\ UDV-ReLU: 5.323\\ UDV(unconstrained): 126.1\\ UFV: --\\ UV-ReLU: 5.323\\ UV: --\end{tabular} &
  \begin{tabular}[c]{@{}l@{}}UDV: --\\ UDV-s: --\\ UDV-v1: --\\ UDV-v2: --\\ UDV-ReLU: --\\ UDV(unconstrained): --\\ UFV: --\\ UV-ReLU: --\\ UV: --\\ M: --\end{tabular} &
  \begin{tabular}[c]{@{}l@{}}UDV: 99.05\\ UDV-s: 95.62\\ UDV-v1: --\\ UDV-v2: --\\ UDV-ReLU: --\\ UDV(unconstrained): --\\ UFV: --\\ UV-ReLU:  --\\ UV: 97.69\\ E: 99.10\end{tabular} &
  \begin{tabular}[c]{@{}l@{}}UDV: 97.54\\ UDV-s: 99.31\\ UDV-v1: --\\ UDV-v2: --\\ UDV-ReLU: --\\ UDV(unconstrained): --\\ UFV: --\\ UV-ReLU: --\\ UV: --\\ R: 99.22\end{tabular} \\ \hline
LR: $10^{0}$ &
  \begin{tabular}[c]{@{}l@{}}UDV: 38.71\\ UDV-s: --\\ UDV-v1: 2.413\\ UDV-v2: 4.633\\ UDV-ReLU: 48.62\\ UDV(unconstrained): 14.05\\ UFV: --\\ UV-ReLU: 48.24\\ UV: --\end{tabular} &
  \begin{tabular}[c]{@{}l@{}}UDV: 5.327\\ UDV-s: --\\ UDV-v1: 16.19\\ UDV-v2: --\\ UDV-ReLU: 5.323\\ UDV(unconstrained): --\\ UFV: --\\ UV-ReLU: 5.323\\ UV: --\end{tabular} &
  \begin{tabular}[c]{@{}l@{}}UDV: --\\ UDV-s: --\\ UDV-v1: --\\ UDV-v2: --\\ UDV-ReLU: --\\ UDV(unconstrained): --\\ UFV: --\\ UV-ReLU: --\\ UV: --\\ M: --\end{tabular} &
  \begin{tabular}[c]{@{}l@{}}UDV: --\\ UDV-s: --\\ UDV-v1: --\\ UDV-v2: --UDV-ReLU: --\\ UDV(unconstrained): --\\ UFV: --\\ UV-ReLU: --\\ UV: --\\ E: --\end{tabular} &
  \begin{tabular}[c]{@{}l@{}}UDV: --\\ UDV-s: --\\ UDV-v1: --\\ UDV-v2: --\\ UDV-ReLU: --\\ UDV(unconstrained): --\\ UFV: --\\ UV-ReLU: --\\ UV: --\\ R: 95.65\end{tabular} \\ \hline
LR: $2 \times 10^{0}$ &
  \begin{tabular}[c]{@{}l@{}}UDV: 60.46\\ UDV-s: --\\ UDV-v1: 6.651\\ UDV-v2: 7.366\\ UDV-ReLU: 48.62\\ UDV(unconstrained): --\\ UFV: --\\ UV-ReLU: 49.43\\ UV: --\end{tabular} &
  \begin{tabular}[c]{@{}l@{}}UDV: 5.322\\ UDV-s: --\\ UDV-v1: --\\ UDV-v2: --\\ UDV-ReLU: 5.323\\ UDV(unconstrained): --\\ UFV: --\\ UV-ReLU: 5.323\\ UV: --\end{tabular} &
  \multicolumn{1}{c}{--} &
  \multicolumn{1}{c}{--} &
  \multicolumn{1}{c}{--} \\ \hline
LR: $3 \times 10^{0}$ &
  \begin{tabular}[c]{@{}l@{}}UDV: 106.9\\ UDV-s: --\\ UDV-v1: 6.606\\ UDV-v2: 7.398\\ UDV-ReLU: 48.62\\ UDV(unconstrained): --\\ UFV: --\\ UV-ReLU: 48.50\\ UV: --\end{tabular} &
  \begin{tabular}[c]{@{}l@{}}UDV: 5.335\\ UDV-s: --\\ UDV-v1: 6.380\\ UDV-v2: --\\ UDV-ReLU: 5.323\\ UDV(unconstrained): --\\ UFV: --\\ UV-ReLU: 5.323\\ UV: --\end{tabular} &
  \multicolumn{1}{c}{--} &
  \multicolumn{1}{c}{--} &
  \multicolumn{1}{c}{--} \\ \hline
LR: $5\times 10^{0}$ &
  \multicolumn{1}{c}{--} &
  \multicolumn{1}{c}{--} &
  \multicolumn{1}{c}{--} &
  \multicolumn{1}{c}{--} &
  \multicolumn{1}{c}{--} \\ \hline
\end{tabular}
}
\end{table}

\begin{table}[]
\centering
\captionsetup{font=small}
\caption{Experiments using NAdam optimizer. Not applicable or results with obvious oscillations or divergence are denoted as `--'.}
\label{tab:full_results_nadam}
\resizebox{0.95\textwidth}{!}{
\begin{tabular}{clllll}
\hline
Tasks &
  \multicolumn{2}{c}{Regression (Test Loss)} &
  \multicolumn{3}{c}{Classification (Test Accuracy)} \\ \hline
\begin{tabular}[c]{@{}c@{}}Dataset \\ (Transferred model)\\ \textnormal{ }\end{tabular} &
  \multicolumn{1}{c}{\begin{tabular}[c]{@{}c@{}}HPART \\ --\\ ($\times 10^{-3}$)\end{tabular}} &
  \multicolumn{1}{c}{\begin{tabular}[c]{@{}c@{}}NYCTTD \\ --\\ ($\times 10^{-6}$)\end{tabular}} &
  \multicolumn{3}{c}{\begin{tabular}[c]{@{}c@{}}MNIST \\ (MaxVit-T [M] $|$ EfficientNet-B0 [E] $|$ RegNetX-32GF [R]) \\ ($\times 100\%$)\end{tabular}} \\ \hline
LR: $10^{-6}$ &
  \multicolumn{1}{c}{--} &
  \multicolumn{1}{c}{--} &
  \begin{tabular}[c]{@{}l@{}}UDV: --\\ UDV-s: --\\ UDV-v1: 99.56\\ UDV-v2: 99.53\\ UDV-ReLU: --\\ UDV(unconstrained): 99.53\\ UFV: 99.54\\ UV-ReLU: 99.59\\ UV: 99.53\\ M: 99.41\end{tabular} &
  \begin{tabular}[c]{@{}l@{}}UDV: --\\ UDV-s: --\\ UDV-v1: --\\ UDV-v2: --\\ UDV-ReLU: --\\ UDV(unconstrained): --\\ UFV: 99.26\\ UV-ReLU: 99.30\\ UV: 99.24\\ E: 98.67\end{tabular} &
  \begin{tabular}[c]{@{}l@{}}UDV: --\\ UDV-s: --\\ UDV-v1: 99.49\\ UDV-v2: 99.46\\ UDV-ReLU: --\\ UDV(unconstrained): 99.46\\ UFV: 99.32\\ UV-ReLU: 99.40\\ UV: 99.32\\ R: 99.37\end{tabular} \\ \hline
LR: $10^{-5}$ &
  \multicolumn{1}{c}{--} &
  \multicolumn{1}{c}{--} &
  \begin{tabular}[c]{@{}l@{}}UDV: --\\ UDV-s: --\\ UDV-v1: 99.63\\ UDV-v2: 99.59\\ UDV-ReLU: --\\ UDV(unconstrained): 99.60\\ UFV: 99.62\\ UV-ReLU: 99.61 \\ UV: 99.63\\ M: \textbf{99.62}\end{tabular} &
  \begin{tabular}[c]{@{}l@{}}UDV: --\\ UDV-s: --\\ UDV-v1: 99.50\\ UDV-v2: 99.52\\ UDV-ReLU: --\\ UDV(unconstrained): 99.52\\ UFV: 99.43\\ UV-ReLU: 99.51\\ UV: 99.52\\ E: 99.50\end{tabular} &
  \begin{tabular}[c]{@{}l@{}}UDV: --\\ UDV-s: --\\ UDV-v1: 99.66\\ UDV-v2: 99.56\\ UDV-ReLU: --\\ UDV(unconstrained): 99.56\\ UFV: 99.54\\ UV-ReLU: 99.59\\ UV: 99.61\\ R: \textbf{99.67}\end{tabular} \\ \hline
LR: $10^{-4}$ &
  \begin{tabular}[c]{@{}l@{}}UDV: 2.312\\ UDV-s: 1.916\\ UDV-v1: 1.837\\ UDV-v2: 1.752\\ UDV-ReLU: 2.665\\ UDV(unconstrained): 1.752\\ UFV: \textbf{1.381}\\ UV-ReLU: 1.398\\ UV: \textbf{1.512}\end{tabular} &
  \begin{tabular}[c]{@{}l@{}}UDV: \textbf{5.248}\\ UDV-s: 5.248\\ UDV-v1: 5.248\\ UDV-v2: 5.249\\ UDV-ReLU: \textbf{5.251}\\ UDV(unconstrained): 5.249\\ UFV: 5.256\\ UV-ReLU: 5.258\\ UV: 5.275\end{tabular} &
  \begin{tabular}[c]{@{}l@{}}UDV: \textbf{99.67}\\ UDV-s: \textbf{99.62}\\ UDV-v1: \textbf{99.67}\\ UDV-v2: \textbf{99.61}\\ UDV-ReLU: \textbf{99.63}\\ UDV(unconstrained): 99.62\\ UFV: \textbf{99.68}\\ UV-ReLU: \textbf{99.68}\\ UV: \textbf{99.65}\\ M: 99.59\end{tabular} &
  \begin{tabular}[c]{@{}l@{}}UDV: --\\ UDV-s: --\\ UDV-v1: 99.60\\ UDV-v2: \textbf{99.66}\\ UDV-ReLU: --\\ UDV(unconstrained): \textbf{99.60}\\ UFV: \textbf{99.54}\\ UV-ReLU: 99.53\\ UV: 99.53\\ E: \textbf{99.65}\end{tabular} &
  \begin{tabular}[c]{@{}l@{}}UDV: \textbf{99.52}\\ UDV-s: \textbf{99.55}\\ UDV-v1: \textbf{99.69}\\ UDV-v2: \textbf{99.72}\\ UDV-ReLU: 99.46\\ UDV(unconstrained): \textbf{99.63}\\ UFV: \textbf{99.55}\\ UV-ReLU: \textbf{99.63}\\ UV: \textbf{99.64}\\ R: 99.64\end{tabular} \\ \hline
LR: $10^{-3}$ &
  \begin{tabular}[c]{@{}l@{}}UDV: \textbf{1.638}\\ UDV-s: \textbf{1.691}\\ UDV-v1: \textbf{1.418}\\ UDV-v2: \textbf{1.440}\\ UDV-ReLU: \textbf{1.504}\\ UDV(unconstrained): \textbf{1.437}\\ UFV: 1.607\\ UV-ReLU: \textbf{1.367}\\ UV: 1.654\end{tabular} &
  \begin{tabular}[c]{@{}l@{}}UDV: 5.248\\ UDV-s: \textbf{5.248}\\ UDV-v1: \textbf{5.248}\\ UDV-v2: \textbf{5.248}\\ UDV-ReLU: 5.251\\ UDV(unconstrained): \textbf{5.248}\\ UFV: \textbf{5.248}\\ UV-ReLU: \textbf{5.249}\\ UV: \textbf{5.254}\end{tabular} &
  \begin{tabular}[c]{@{}l@{}}UDV: 99.53\\ UDV-s: 99.56\\ UDV-v1: 99.54\\ UDV-v2: 99.61\\ UDV-ReLU: 99.55\\ UDV(unconstrained): \textbf{99.66}\\ UFV: 99.57\\ UV-ReLU: 99.63\\ UV: 99.59\\ M: 99.58\end{tabular} &
  \begin{tabular}[c]{@{}l@{}}UDV: 97.56\\ UDV-s: \textbf{99.61}\\ UDV-v1: \textbf{99.60}\\ UDV-v2: 99.55\\ UDV-ReLU: \textbf{99.50}\\ UDV(unconstrained): 99.55\\ UFV: 99.28\\ UV-ReLU: \textbf{99.55}\\ UV: \textbf{99.55}\\ E: 99.58\end{tabular} &
  \begin{tabular}[c]{@{}l@{}}UDV: 99.45\\ UDV-s: 99.45\\ UDV-v1: 99.43\\ UDV-v2: 99.46\\ UDV-ReLU: 99.47\\ UDV(unconstrained): 99.50\\ UFV: 99.27\\ UV-ReLU: 99.53\\ UV:99.42\\ R: 99.65\end{tabular} \\ \hline
LR: $10^{-2}$ &
  \begin{tabular}[c]{@{}l@{}}UDV: 3.396\\ UDV-s: 3.287\\ UDV-v1: 2.297\\ UDV-v2: 2.315\\ UDV-ReLU: 2.403\\ UDV(unconstrained): 1.884\\ UFV: --\\ UV-ReLU: \textbf{1.954}\\ UV: --\end{tabular} &
  \begin{tabular}[c]{@{}l@{}}UDV: 5.258\\ UDV-s: 6.918\\ UDV-v1: 5.276\\ UDV-v2: 5.253\\ UDV-ReLU: 5.323\\ UDV(unconstrained): 5.253\\ UFV: 25.24\\ UV-ReLU: 5.323\\ UV: 6.262\end{tabular} &
  \begin{tabular}[c]{@{}l@{}}UDV: --\\ UDV-s: --\\ UDV-v1: --\\ UDV-v2: --\\ UDV-ReLU: --\\ UDV(unconstrained): --\\ UFV: --\\ UV-ReLU: --\\ UV: --\\ M: --\end{tabular} &
  \begin{tabular}[c]{@{}l@{}}UDV: \textbf{99.35}\\ UDV-s: --\\ UDV-v1: 99.31\\ UDV-v2: 99.01\\ UDV-ReLU: 99.42\\ UDV(unconstrained): 98.78\\ UFV: --\\ UV-ReLU: 98.87\\ UV: 98.60\\ E: 99.26\end{tabular} &
  \begin{tabular}[c]{@{}l@{}}UDV: 97.12\\ UDV-s: 99.28\\ UDV-v1: 99.37\\ UDV-v2: 99.38\\ UDV-ReLU: \textbf{99.49}\\ UDV(unconstrained): 99.38\\ UFV: --\\ UV-ReLU: 99.28\\ UV: 98.91\\ R: 99.59\end{tabular} \\ \hline
LR: $10^{-1}$ &
  \begin{tabular}[c]{@{}l@{}}UDV: 13.69\\ UDV-s: --\\ UDV-v1: --\\ UDV-v2: --\\ UDV-ReLU: 48.62\\ UDV(unconstrained): --\\ UFV: 4.918\\ UV-ReLU: 41.61\\ UV: \textbf{1.863}\end{tabular} &
  \begin{tabular}[c]{@{}l@{}}UDV: 5.323\\ UDV-s: --\\ UDV-v1: 6.627\\ UDV-v2: --\\ UDV-ReLU: 5.323\\ UDV(unconstrained): --\\ UFV: --\\ UV-ReLU: 5.323\\ UV: --\end{tabular} &
  \begin{tabular}[c]{@{}l@{}}UDV: --\\ UDV-s: --\\ UDV-v1: --\\ UDV-v2: --\\ UDV-ReLU: --\\ UDV(unconstrained): --\\ UFV: --\\ UV-ReLU: --\\ UV: --\\ M: --\end{tabular} &
  \begin{tabular}[c]{@{}l@{}}UDV: --\\ UDV-s: --\\ UDV-v1: --\\ UDV-v2: --\\ UDV-ReLU: --\\ UDV(unconstrained): --\\ UFV: --\\ UV-ReLU:  --\\ UV: 97.54\\ E: 98.75\end{tabular} &
  \begin{tabular}[c]{@{}l@{}}UDV: --\\ UDV-s: --\\ UDV-v1: --\\ UDV-v2: --\\ UDV-ReLU: --\\ UDV(unconstrained): --\\ UFV: --\\ UV-ReLU: --\\ UV: --\\ R: 99.19\end{tabular} \\ \hline
LR: $10^{0}$ &
  \begin{tabular}[c]{@{}l@{}}UDV: --\\ UDV-s: --\\ UDV-v1: 5.649\\ UDV-v2: \textbf{2.317}\\ UDV-ReLU: 48.62\\ UDV(unconstrained): --\\ UFV: --\\ UV-ReLU: 45.77\\ UV: --\end{tabular} &
  \begin{tabular}[c]{@{}l@{}}UDV: 5.328\\ UDV-s: 53.70\\ UDV-v1: 19.52\\ UDV-v2: --\\ UDV-ReLU: 5.323\\ UDV(unconstrained): --\\ UFV: --\\ UV-ReLU: 5.323\\ UV: --\end{tabular} &
  \begin{tabular}[c]{@{}l@{}}UDV: --\\ UDV-s: --\\ UDV-v1: --\\ UDV-v2: --\\ UDV-ReLU: --\\ UDV(unconstrained): --\\ UFV: --\\ UV-ReLU: --\\ UV: --\\ M: --\end{tabular} &
  \begin{tabular}[c]{@{}l@{}}UDV: --\\ UDV-s: --\\ UDV-v1: --\\ UDV-v2: --\\ UDV-ReLU: --\\ UDV(unconstrained): --\\ UFV: --\\ UV-ReLU: --\\ UV: --\\ E: --\end{tabular} &
  \begin{tabular}[c]{@{}l@{}}UDV: --\\ UDV-s: --\\ UDV-v1: --\\ UDV-v2: --\\ UDV-ReLU: --\\ UDV(unconstrained): --\\ UFV: --\\ UV-ReLU: --\\ UV: --\\ R: --\end{tabular} \\ \hline
LR: $2 \times 10^{0}$ &
  \begin{tabular}[c]{@{}l@{}}UDV: --\\ UDV-s: --\\ UDV-v1: 3.075\\ UDV-v2: \textbf{2.828}\\ UDV-ReLU: 48.62\\ UDV(unconstrained): --\\ UFV: --\\ UV-ReLU: 59.82\\ UV: --\end{tabular} &
  \begin{tabular}[c]{@{}l@{}}UDV:5.377\\ UDV-s: --\\ UDV-v1: 6.011\\ UDV-v2: --\\ UDV-ReLU: 5.323\\ UDV(unconstrained): --\\ UFV: --\\ UV-ReLU: 5.323\\ UV: --\end{tabular} &
  \multicolumn{1}{c}{--} &
  \multicolumn{1}{c}{--} &
  \multicolumn{1}{c}{--} \\ \hline
LR: $3 \times 10^{0}$ &
  \begin{tabular}[c]{@{}l@{}}UDV: --\\ UDV-s: --\\ UDV-v1: 3.882\\ UDV-v2: \textbf{3.128}\\ UDV-ReLU: 48.62\\ UDV(unconstrained): --\\ UFV: --\\ UV-ReLU: --\\ UV: --\end{tabular} &
  \begin{tabular}[c]{@{}l@{}}UDV: 5.401\\ UDV-s: --\\ UDV-v1: 15.27\\ UDV-v2: --\\ UDV-ReLU: 5.323\\ UDV(unconstrained): --\\ UFV: --\\ UV-ReLU: 5.323\\ UV: --\end{tabular} &
  \multicolumn{1}{c}{--} &
  \multicolumn{1}{c}{--} &
  \multicolumn{1}{c}{--} \\ \hline
LR: $5\times 10^{0}$ &
  \multicolumn{1}{c}{--} &
  \multicolumn{1}{c}{--} &
  \multicolumn{1}{c}{--} &
  \multicolumn{1}{c}{--} &
  \multicolumn{1}{c}{--} \\ \hline
\end{tabular}
}
\end{table}

\begin{table}[]
\centering
\captionsetup{font=small}
\caption{Experiments using MBGD optimizer. Not applicable or results with obvious oscillations or divergence are denoted as `--'.}
\label{tab:full_results_MBGD}
\resizebox{\textwidth}{!}{
\begin{tabular}{clllll}
\hline
Tasks &
  \multicolumn{2}{c}{Regression (Test Loss)} &
  \multicolumn{3}{c}{Classification (Test Accuracy)} \\ \hline
\begin{tabular}[c]{@{}c@{}}Dataset \\ (Transferred model)\\ \textnormal{ }\end{tabular} &
  \multicolumn{1}{c}{\begin{tabular}[c]{@{}c@{}}HPART \\ --\\ ($\times 10^{-3}$)\end{tabular}} &
  \multicolumn{1}{c}{\begin{tabular}[c]{@{}c@{}}NYCTTD \\ --\\ ($\times 10^{-6}$)\end{tabular}} &
  \multicolumn{3}{c}{\begin{tabular}[c]{@{}c@{}}MNIST \\ (MaxVit-T [M] $|$ EfficientNet-B0 [E] $|$ RegNetX-32GF [R]) \\ ($\times 100\%$)\end{tabular}} \\ \hline
LR: $10^{-6}$ &
  \multicolumn{1}{c}{--} &
  \multicolumn{1}{c}{--} &
  \multicolumn{1}{c}{--} &
  \multicolumn{1}{c}{--} &
  \multicolumn{1}{c}{--} \\ \hline
LR: $10^{-5}$ &
  \multicolumn{1}{c}{--} &
  \multicolumn{1}{c}{--} &
  \multicolumn{1}{c}{--} &
  \multicolumn{1}{c}{--} &
  \multicolumn{1}{c}{--} \\ \hline
LR: $10^{-4}$ &
  \begin{tabular}[c]{@{}l@{}}UDV: 46.90\\ UDV-s: 45.43\\ UDV-v1: 46.01\\ UDV-v2: 44.35\\ UDV-ReLU: 47.88\\ UDV(unconstrained): 44.35\\ UFV: 11.77\\ UV-ReLU: 12.59\\ UV: 11.98\end{tabular} &
  \begin{tabular}[c]{@{}l@{}}UDV: 26.60\\ UDV-s: 40.06\\ UDV-v1: 49.35\\ UDV-v2: 83.91\\ UDV-ReLU: 20.56\\ UDV(unconstrained): 83.91\\ UFV: 86.68\\ UV-ReLU: --\\ UV: 71.72\end{tabular} &
  \multicolumn{1}{c}{--} &
  \multicolumn{1}{c}{--} &
  \multicolumn{1}{c}{--} \\ \hline
LR: $10^{-3}$ &
  \begin{tabular}[c]{@{}l@{}}UDV: 30.38\\ UDV-s: 21.39\\ UDV-v1: 23.65\\ UDV-v2: 15.77\\ UDV-ReLU: 39.89\\ UDV(unconstrained): 15.77\\ UFV: 6.719\\ UV-ReLU: 6.478\\ UV: 6.345\end{tabular} &
  \begin{tabular}[c]{@{}l@{}}UDV: 9.407\\ UDV-s: 10.29\\ UDV-v1: 10.83\\ UDV-v2: 12.56\\ UDV-ReLU: 7.231\\ UDV(unconstrained): 12.56\\ UFV: 8.236\\ UV-ReLU: 12.67\\ UV: 6.763\end{tabular} &
  \begin{tabular}[c]{@{}l@{}}UDV: --\\ UDV-s: --\\ UDV-v1: 98.44\\ UDV-v2: 98.32\\ UDV-ReLU: --\\ UDV(unconstrained): 98.27\\ UFV: 99.34\\ UV-ReLU: 99.24\\ UV: 99.19\\ M: 99.29\end{tabular} &
  \begin{tabular}[c]{@{}l@{}}UDV: --\\ UDV-s: --\\ UDV-v1: --\\ UDV-v2: --\\ UDV-ReLU: --\\ UDV(unconstrained): --\\ UFV: 98.15\\ UV-ReLU: 98.48\\ UV: 98.55\\ E: 98.74\end{tabular} &
  \begin{tabular}[c]{@{}l@{}}UDV: --\\ UDV-s: --\\ UDV-v1: --\\ UDV-v2: --\\ UDV-ReLU: --\\ UDV(unconstrained): --\\ UFV: 99.20\\ UV-ReLU: 99.16\\ UV: 99.17\\ R: 99.16\end{tabular} \\ \hline
LR: $10^{-2}$ &
  \begin{tabular}[c]{@{}l@{}}UDV: 6.030\\ UDV-s: 6.014\\ UDV-v1: 6.125\\ UDV-v2: 5.877\\ UDV-ReLU: 6.234\\ UDV(unconstrained): 5.877\\ UFV: 5.575\\ UV-ReLU: 2.253\\ UV: 2.251\end{tabular} &
  \begin{tabular}[c]{@{}l@{}}UDV: 5.465\\ UDV-s: 5.421\\ UDV-v1: 5.514\\ UDV-v2: 5.532\\ UDV-ReLU: 5.431\\ UDV(unconstrained): 5.532\\ UFV: 5.293\\ UV-ReLU: 5.521\\ UV: 5.296\end{tabular} &
  \begin{tabular}[c]{@{}l@{}}UDV: 99.38\\ UDV-s: 99.32\\ UDV-v1: 99.47\\ UDV-v2: 99.50\\ UDV-ReLU: 99.53\\ UDV(unconstrained): 99.58\\ UFV: 99.53\\ UV-ReLU: 99.53\\ UV: 99.50\\ M: 99.59\end{tabular} &
  \begin{tabular}[c]{@{}l@{}}UDV: --\\ UDV-s: --\\ UDV-v1: 98.62\\ UDV-v2: 98.97\\ UDV-ReLU: --\\ UDV(unconstrained): 98.40\\ UFV: 99.05\\ UV-ReLU: 99.31\\ UV: 99.24\\ E: 99.43\end{tabular} &
  \begin{tabular}[c]{@{}l@{}}UDV: 99.40\\ UDV-s: 99.46\\ UDV-v1: 99.37\\ UDV-v2: 99.34\\ UDV-ReLU: \textbf{99.47}\\ UDV(unconstrained): 99.34\\ UFV: 99.36\\ UV-ReLU: 99.31\\ UV:99.38\\ R: 99.44\end{tabular} \\ \hline
LR: $10^{-1}$ &
  \begin{tabular}[c]{@{}l@{}}UDV: \textbf{1.398}\\ UDV-s: \textbf{1.407}\\ UDV-v1: \textbf{1.556}\\ UDV-v2: \textbf{1.565}\\ UDV-ReLU: \textbf{1.379}\\ UDV(unconstrained): \textbf{1.565}\\ UFV: \textbf{1.548}\\ UV-ReLU: \textbf{1.493}\\ UV: \textbf{1.583}\end{tabular} &
  \begin{tabular}[c]{@{}l@{}}UDV: 5.253\\ UDV-s: 5.250\\ UDV-v1: 5.256\\ UDV-v2: 5.252\\ UDV-ReLU: 5.277\\ UDV(unconstrained): 5.252\\ UFV: \textbf{5.255}\\ UV-ReLU: 5.323\\ UV: \textbf{5.291}\end{tabular} &
  \begin{tabular}[c]{@{}l@{}}UDV: \textbf{99.59}\\ UDV-s: \textbf{99.38}\\ UDV-v1: \textbf{99.61}\\ UDV-v2: \textbf{99.60}\\ UDV-ReLU: \textbf{99.65}\\ UDV(unconstrained): \textbf{99.59}\\ UFV: \textbf{99.61}\\ UV-ReLU: \textbf{99.60}\\ UV: \textbf{99.63}\\ M: \textbf{99.62}\end{tabular} &
  \begin{tabular}[c]{@{}l@{}}UDV: --\\ UDV-s: \textbf{99.36}\\ UDV-v1: 99.43\\ UDV-v2: 99.53\\ UDV-ReLU: 95.21\\ UDV(unconstrained): \textbf{95.29}\\ UFV: \textbf{99.55}\\ UV-ReLU: 99.52\\ UV: \textbf{99.55}\\ E: \textbf{99.67}\end{tabular} &
  \begin{tabular}[c]{@{}l@{}}UDV: \textbf{99.51}\\ UDV-s: \textbf{99.57}\\ UDV-v1: 99.60\\ UDV-v2: 99.62\\ UDV-ReLU: --\\ UDV(unconstrained): \textbf{99.61}\\ UFV: \textbf{99.59}\\ UV-ReLU: 99.59\\ UV: 98.56\\ R: 99.63\end{tabular} \\ \hline
LR: $10^{0}$ &
  \begin{tabular}[c]{@{}l@{}}UDV: 3.076\\ UDV-s: 2.870\\ UDV-v1: 1.935\\ UDV-v2: 1.830\\ UDV-ReLU: 4.573\\ UDV(unconstrained): 1.810\\ UFV: --\\ UV-ReLU: 48.62\\ UV: --\end{tabular} &
  \begin{tabular}[c]{@{}l@{}}UDV: \textbf{5.248}\\ UDV-s: \textbf{5.248}\\ UDV-v1: \textbf{5.249}\\ UDV-v2: \textbf{5.248}\\ UDV-ReLU: 5.261\\ UDV(unconstrained): \textbf{5.248}\\ UFV: 5.256\\ UV-ReLU: \textbf{5.271}\\ UV: --\end{tabular} &
  \begin{tabular}[c]{@{}l@{}}UDV: --\\ UDV-s: --\\ UDV-v1: --\\ UDV-v2: --\\ UDV-ReLU: --\\ UDV(unconstrained): --\\ UFV: --\\ UV-ReLU: --\\ UV: --\\ M: --\end{tabular} &
  \begin{tabular}[c]{@{}l@{}}UDV: --\\ UDV-s: --\\ UDV-v1: \textbf{99.65}\\ UDV-v2: 99.59\\ UDV-ReLU: --\\ UDV(unconstrained): \textbf{99.62}\\ UFV: --\\ UV-ReLU: \textbf{99.64}\\ UV: --\\ E: 99.55\end{tabular} &
  \begin{tabular}[c]{@{}l@{}}UDV: --\\ UDV-s: --\\ UDV-v1: \textbf{99.67}\\ UDV-v2: \textbf{99.69}\\ UDV-ReLU: --\\ UDV(unconstrained): 99.57\\ UFV: --\\ UV-ReLU: \textbf{99.73}\\ UV: \textbf{99.66}\\ R: \textbf{99.66}\end{tabular} \\ \hline
LR: $2 \times 10^{0}$ &
  \begin{tabular}[c]{@{}l@{}}UDV: 6.244\\ UDV-s: 48.63\\ UDV-v1: 2.559\\ UDV-v2: --\\ UDV-ReLU: 47.50\\ UDV(unconstrained): --\\ UFV: --\\ UV-ReLU: 48.63\\ UV: --\end{tabular} &
  \begin{tabular}[c]{@{}l@{}}UDV: 5.248\\ UDV-s: 5.248\\ UDV-v1: 5.249\\ UDV-v2: 5.249\\ UDV-ReLU: 5.259\\ UDV(unconstrained): 5.249\\ UFV: --\\ UV-ReLU: 5.293\\ UV: --\end{tabular} &
  \begin{tabular}[c]{@{}l@{}}UDV: --\\ UDV-s: --\\ UDV-v1: --\\ UDV-v2: --\\ UDV-ReLU: --\\ UDV(unconstrained): --\\ UFV: --\\ UV-ReLU: --\\ UV: --\\ M: --\end{tabular} &
  \begin{tabular}[c]{@{}l@{}}UDV: --\\ UDV-s: --\\ UDV-v1: 98.37\\ UDV-v2: \textbf{99.66}\\ UDV-ReLU: --\\ UDV(unconstrained): --\\ UFV: --\\ UV-ReLU: --\\ UV: --\\ E: 98.88\end{tabular} &
  \begin{tabular}[c]{@{}l@{}}UDV: --\\ UDV-s: --\\ UDV-v1: 99.66\\ UDV-v2: 99.55\\ UDV-ReLU: --\\ UDV(unconstrained): --\\ UFV: --\\ UV-ReLU: 99.58\\ UV: 99.57\\ R: 99.56\end{tabular} \\ \hline
LR: $3 \times 10^{0}$ &
  \begin{tabular}[c]{@{}l@{}}UDV: --\\ UDV-s: --\\ UDV-v1: --\\ UDV-v2: --\\ UDV-ReLU: 48.62\\ UDV(unconstrained): --\\ UFV: --\\ UV-ReLU: 48.62\\ UV: --\end{tabular} &
  \begin{tabular}[c]{@{}l@{}}UDV: 5.248\\ UDV-s: 5.248\\ UDV-v1: 5.249\\ UDV-v2: 5.249\\ UDV-ReLU: \textbf{5.257}\\ UDV(unconstrained): 5.250\\ UFV: --\\ UV-ReLU: 5.296\\ UV: --\end{tabular} &
  \begin{tabular}[c]{@{}l@{}}UDV: --\\ UDV-s: --\\ UDV-v1: --\\ UDV-v2: --\\ UDV-ReLU: --\\ UDV(unconstrained): --\\ UFV: --\\ UV-ReLU: --\\ UV: --\\ M: --\end{tabular} &
  \begin{tabular}[c]{@{}l@{}}UDV: --\\ UDV-s: --\\ UDV-v1: 99.07\\ UDV-v2: --\\ UDV-ReLU: --\\ UDV(unconstrained): --\\ UFV: --\\ UV-ReLU: --\\ UV: --\\ E: 99.26\end{tabular} &
  \begin{tabular}[c]{@{}l@{}}UDV: --\\ UDV-s: --\\ UDV-v1: 99.48\\ UDV-v2: 99.65\\ UDV-ReLU: --\\ UDV(unconstrained): --\\ UFV: --\\ UV-ReLU: 99.57\\ UV: --\\ R: 99.57\end{tabular} \\ \hline
LR: $5\times 10^{0}$ &
  \multicolumn{1}{c}{--} &
  \multicolumn{1}{c}{--} &
  \begin{tabular}[c]{@{}l@{}}UDV: --\\ UDV-s: --\\ UDV-v1: --\\ UDV-v2: --\\ UDV-ReLU: --\\ UDV(unconstrained): --\\ UFV: --\\ UV-ReLU: --\\ UV: --\\ M: --\end{tabular} &
  \begin{tabular}[c]{@{}l@{}}UDV: --\\ UDV-s: --\\ UDV-v1: --\\ UDV-v2: --\\ UDV-ReLU: --\\ UDV(unconstrained): --\\ UFV: --\\ UV-ReLU: --\\ UV: --\\ E: 99.33\end{tabular} &
  \begin{tabular}[c]{@{}l@{}}UDV: --\\ UDV-s: --\\ UDV-v1: 99.55\\ UDV-v2: 99.41\\ UDV-ReLU: --\\ UDV(unconstrained): --\\ UFV: --\\ UV-ReLU: 99.42\\ UV: --\\ R: 99.55\end{tabular} \\ \hline
\end{tabular}
}
\end{table}

\begin{table}[]
\centering
\captionsetup{font=small}
\caption{Experiments using MBGDM optimizer. Not applicable or results with obvious oscillations or divergence are denoted as `--'.}
\label{tab:full_results_mbgdm}
\resizebox{\textwidth}{!}{
\begin{tabular}{clllll}
\hline
Tasks &
  \multicolumn{2}{c}{Regression (Test Loss)} &
  \multicolumn{3}{c}{Classification (Test Accuracy)} \\ \hline
\begin{tabular}[c]{@{}c@{}}Dataset \\ (Transferred model)\\ \textnormal{ }\end{tabular} &
  \multicolumn{1}{c}{\begin{tabular}[c]{@{}c@{}}HPART \\ --\\ ($\times 10^{-3}$)\end{tabular}} &
  \multicolumn{1}{c}{\begin{tabular}[c]{@{}c@{}}NYCTTD \\ --\\ ($\times 10^{-6}$)\end{tabular}} &
  \multicolumn{3}{c}{\begin{tabular}[c]{@{}c@{}}MNIST \\ (MaxVit-T [M] $|$ EfficientNet-B0 [E] $|$ RegNetX-32GF [R]) \\ ($\times 100\%$)\end{tabular}} \\ \hline
LR: $10^{-6}$ &
  \multicolumn{1}{c}{--} &
  \multicolumn{1}{c}{--} &
  \multicolumn{1}{c}{--} &
  \multicolumn{1}{c}{--} &
  \multicolumn{1}{c}{--} \\ \hline
LR: $10^{-5}$ &
  \multicolumn{1}{c}{--} &
  \multicolumn{1}{c}{--} &
  \multicolumn{1}{c}{--} &
  \multicolumn{1}{c}{--} &
  \multicolumn{1}{c}{--} \\ \hline
LR: $10^{-4}$ &
  \begin{tabular}[c]{@{}l@{}}UDV: 30.52\\ UDV-s: 21.50\\ UDV-v1: 23.79\\ UDV-v2: 15.84\\ UDV-ReLU: 40.00\\ UDV(unconstrained): 15.84\\ UFV: 6.729\\ UV-ReLU: 6.488\\ UV: 6.348\end{tabular} &
  \begin{tabular}[c]{@{}l@{}}UDV: 9.413\\ UDV-s: 10.29\\ UDV-v1: 10.85\\ UDV-v2: 12.59\\ UDV-ReLU: 7.224\\ UDV(unconstrained): 12.59\\ UFV: 8.288\\ UV-ReLU: 12.68\\ UV: 6.776\end{tabular} &
  \multicolumn{1}{c}{--} &
  \multicolumn{1}{c}{--} &
  \multicolumn{1}{c}{--} \\ \hline
LR: $10^{-3}$ &
  \begin{tabular}[c]{@{}l@{}}UDV: 6.105\\ UDV-s: 6.080\\ UDV-v1: 6.178\\ UDV-v2: 5.929\\ UDV-ReLU: 6.341\\ UDV(unconstrained): 5.929\\ UFV: 2.590\\ UV-ReLU: 2.529\\ UV: 2.252\end{tabular} &
  \begin{tabular}[c]{@{}l@{}}UDV: 5.457\\ UDV-s: 5.412\\ UDV-v1: 5.516\\ UDV-v2: 5.534\\ UDV-ReLU: 5.431\\ UDV(unconstrained): 5.534\\ UFV: 5.294\\ UV-ReLU: 5.520\\ UV: 5.296\end{tabular} &
  \begin{tabular}[c]{@{}l@{}}UDV: 99.26\\ UDV-s: 99.38\\ UDV-v1: 99.39\\ UDV-v2: 99.52\\ UDV-ReLU: --\\ UDV(unconstrained): 99.51\\ UFV: 99.56\\ UV-ReLU: 99.46\\ UV: 99.54\\ M: 99.58\end{tabular} &
  \begin{tabular}[c]{@{}l@{}}UDV: --\\ UDV-s: --\\ UDV-v1: 99.00\\ UDV-v2: 99.05\\ UDV-ReLU: --\\ UDV(unconstrained): 99.08\\ UFV: 99.18\\ UV-ReLU: 99.37\\ UV: 99.36\\ E: 99.45\end{tabular} &
  \begin{tabular}[c]{@{}l@{}}UDV: 99.50\\ UDV-s: 99.43\\ UDV-v1: 99.39\\ UDV-v2: 99.37\\ UDV-ReLU: --\\ UDV(unconstrained): 99.40\\ UFV: 99.38\\ UV-ReLU: 99.36\\ UV: 99.30\\ R: 99.41\end{tabular} \\ \hline
LR: $10^{-2}$ &
  \begin{tabular}[c]{@{}l@{}}UDV: 1.357\\ UDV-s: 1.388\\ UDV-v1: 1.554\\ UDV-v2: 1.569\\ UDV-ReLU: \textbf{1.312}\\ UDV(unconstrained): 1.569\\ UFV: \textbf{1.357}\\ UV-ReLU: 1.314\\ UV: \textbf{1.337}\end{tabular} &
  \begin{tabular}[c]{@{}l@{}}UDV: 5.253\\ UDV-s: 5.250\\ UDV-v1: 5.256\\ UDV-v2: 5.252\\ UDV-ReLU: 5.276\\ UDV(unconstrained): 5.252\\ UFV: 5.254\\ UV-ReLU: 5.267\\ UV: 5.279\end{tabular} &
  \begin{tabular}[c]{@{}l@{}}UDV: \textbf{99.67}\\ UDV-s: 99.38\\ UDV-v1: 99.63\\ UDV-v2: 99.60\\ UDV-ReLU: \textbf{99.62}\\ UDV(unconstrained): 99.59\\ UFV: \textbf{99.65}\\ UV-ReLU: \textbf{99.63}\\ UV: \textbf{99.69}\\ M: \textbf{99.64}\end{tabular} &
  \begin{tabular}[c]{@{}l@{}}UDV: 99.48\\ UDV-s: 99.60\\ UDV-v1: 99.54\\ UDV-v2: 99.59\\ UDV-ReLU: 99.54\\ UDV(unconstrained): 99.55\\ UFV: 99.50\\ UV-ReLU: 99.53\\ UV: 99.50\\ E: 99.59\end{tabular} &
  \begin{tabular}[c]{@{}l@{}}UDV: \textbf{99.74}\\ UDV-s: \textbf{99.72}\\ UDV-v1: 99.61\\ UDV-v2: \textbf{99.67}\\ UDV-ReLU: \textbf{99.65}\\ UDV(unconstrained): 99.63\\ UFV: \textbf{99.66}\\ UV-ReLU: 99.61\\ UV:99.53\\ R: 99.60\end{tabular} \\ \hline
LR: $10^{-1}$ &
  \begin{tabular}[c]{@{}l@{}}UDV: \textbf{1.345}\\ UDV-s: \textbf{1.339}\\ UDV-v1: \textbf{1.313}\\ UDV-v2: \textbf{1.302}\\ UDV-ReLU: 1.318\\ UDV(unconstrained): \textbf{1.302}\\ UFV: 1.358\\ UV-ReLU: \textbf{1.244}\\ UV: --\end{tabular} &
  \begin{tabular}[c]{@{}l@{}}UDV: \textbf{5.248}\\ UDV-s: 5.248\\ UDV-v1: 5.249\\ UDV-v2: 5.248\\ UDV-ReLU: 5.259\\ UDV(unconstrained): 2.248\\ UFV: \textbf{5.251}\\ UV-ReLU: \textbf{5.264}\\ UV: \textbf{5.259}\end{tabular} &
  \begin{tabular}[c]{@{}l@{}}UDV: 99.61\\ UDV-s: \textbf{99.60}\\ UDV-v1: \textbf{99.65}\\ UDV-v2: \textbf{99.68}\\ UDV-ReLU: --\\ UDV(unconstrained): \textbf{99.70}\\ UFV: --\\ UV-ReLU: 99.63\\ UV: --\\ M: --\end{tabular} &
  \begin{tabular}[c]{@{}l@{}}UDV: \textbf{99.63}\\ UDV-s: \textbf{99.63}\\ UDV-v1: \textbf{99.64}\\ UDV-v2: \textbf{99.66}\\ UDV-ReLU: \textbf{99.61}\\ UDV(unconstrained): \textbf{99.59}\\ UFV: \textbf{99.68}\\ UV-ReLU: \textbf{99.68}\\ UV: \textbf{99.59}\\ E: \textbf{99.63}\end{tabular} &
  \begin{tabular}[c]{@{}l@{}}UDV: 99.60\\ UDV-s: 99.57\\ UDV-v1: \textbf{99.71}\\ UDV-v2: 99.67\\ UDV-ReLU: --\\ UDV(unconstrained): \textbf{99.67}\\ UFV: 99.63\\ UV-ReLU: \textbf{99.66}\\ UV: \textbf{99.56}\\ R: \textbf{99.67}\end{tabular} \\ \hline
LR: $10^{0}$ &
  \begin{tabular}[c]{@{}l@{}}UDV: 123.6\\ UDV-s: --\\ UDV-v1: --\\ UDV-v2: --\\ UDV-ReLU: 47.03\\ UDV(unconstrained): --\\ UFV: --\\ UV-ReLU: 48.62\\ UV: --\end{tabular} &
  \begin{tabular}[c]{@{}l@{}}UDV: 5.248\\ UDV-s: 5.248\\ UDV-v1: 5.249\\ UDV-v2: 5.249\\ UDV-ReLU: 5.251\\ UDV(unconstrained): 5.249\\ UFV: --\\ UV-ReLU: 5.281\\ UV: --\end{tabular} &
  \begin{tabular}[c]{@{}l@{}}UDV: --\\ UDV-s: --\\ UDV-v1: --\\ UDV-v2: --\\ UDV-ReLU: --\\ UDV(unconstrained): --\\ UFV: --\\ UV-ReLU: --\\ UV: --\\ M: --\end{tabular} &
  \begin{tabular}[c]{@{}l@{}}UDV: --\\ UDV-s: --\\ UDV-v1: 99.54\\ UDV-v2: --\\ UDV-ReLU: --\\ UDV(unconstrained): --\\ UFV: --\\ UV-ReLU: --\\ UV: --\\ E: 99.24\end{tabular} &
  \begin{tabular}[c]{@{}l@{}}UDV: --\\ UDV-s: --\\ UDV-v1: 99.52\\ UDV-v2: 99.58\\ UDV-ReLU: --\\ UDV(unconstrained): --\\ UFV: --\\ UV-ReLU: 99.53\\ UV: --\\ R: 99.34\end{tabular} \\ \hline
LR: $2 \times 10^{0}$ &
  \begin{tabular}[c]{@{}l@{}}UDV: --\\ UDV-s: --\\ UDV-v1: --\\ UDV-v2: --\\ UDV-ReLU: 48.62\\ UDV(unconstrained): --\\ UFV: --\\ UV-ReLU: 48.62\\ UV: --\end{tabular} &
  \begin{tabular}[c]{@{}l@{}}UDV: 5.248\\ UDV-s: 5.248\\ UDV-v1: 5.249\\ UDV-v2: 5.248\\ UDV-ReLU: \textbf{5.249}\\ UDV(unconstrained): 5.248\\ UFV: --\\ UV-ReLU: 5.289\\ UV: --\end{tabular} &
  \begin{tabular}[c]{@{}l@{}}UDV: --\\ UDV-s: --\\ UDV-v1: --\\ UDV-v2: --\\ UDV-ReLU: --\\ UDV(unconstrained): --\\ UFV: --\\ UV-ReLU: --\\ UV: --\\ M: --\end{tabular} &
  \begin{tabular}[c]{@{}l@{}}UDV: --\\ UDV-s: --\\ UDV-v1: --\\ UDV-v2: --\\ UDV-ReLU: --\\ UDV(unconstrained): --\\ UFV: --\\ UV-ReLU: --\\ UV: --\\ E: 98.95\end{tabular} &
  \begin{tabular}[c]{@{}l@{}}UDV: --\\ UDV-s: --\\ UDV-v1: 99.31\\ UDV-v2: 99.37\\ UDV-ReLU: --\\ UDV(unconstrained): --\\ UFV: --\\ UV-ReLU: 99.28\\ UV: --\\ R: 99.31\end{tabular} \\ \hline
LR: $3 \times 10^{0}$ &
  \begin{tabular}[c]{@{}l@{}}UDV: --\\ UDV-s: --\\ UDV-v1: --\\ UDV-v2: --\\ UDV-ReLU: 48.62\\ UDV(unconstrained): --\\ UFV: --\\ UV-ReLU: 48.62\\ UV: --\end{tabular} &
  \begin{tabular}[c]{@{}l@{}}UDV: 5.248\\ UDV-s: \textbf{5.248}\\ UDV-v1: \textbf{5.248}\\ UDV-v2: \textbf{5.248}\\ UDV-ReLU: 5.249\\ UDV(unconstrained): \textbf{5.248}\\ UFV: --\\ UV-ReLU: 5.292\\ UV: --\end{tabular} &
  \begin{tabular}[c]{@{}l@{}}UDV: --\\ UDV-s: --\\ UDV-v1: --\\ UDV-v2: --\\ UDV-ReLU: --\\ UDV(unconstrained): --\\ UFV: --\\ UV-ReLU: --\\ UV: --\\ M: --\end{tabular} &
  \begin{tabular}[c]{@{}l@{}}UDV: --\\ UDV-s: --\\ UDV-v1: --\\ UDV-v2: --\\ UDV-ReLU: --\\ UDV(unconstrained): --\\ UFV: --\\ UV-ReLU: --\\ UV: --\\ E: 98.70\end{tabular} &
  \begin{tabular}[c]{@{}l@{}}UDV: --\\ UDV-s: --\\ UDV-v1: --\\ UDV-v2: --\\ UDV-ReLU: --\\ UDV(unconstrained): --\\ UFV: --\\ UV-ReLU: --\\ UV: --\\ R: 99.12\end{tabular} \\ \hline
LR: $5\times 10^{0}$ &
  \multicolumn{1}{c}{--} &
  \multicolumn{1}{c}{--} &
  \begin{tabular}[c]{@{}l@{}}UDV: --\\ UDV-s: --\\ UDV-v1: --\\ UDV-v2: --\\ UDV-ReLU: --\\ UDV(unconstrained): --\\ UFV: --\\ UV-ReLU: --\\ UV: --\\ M: --\end{tabular} &
  \begin{tabular}[c]{@{}l@{}}UDV: --\\ UDV-s: --\\ UDV-v1: --\\ UDV-v2: --\\ UDV-ReLU: --\\ UDV(unconstrained): --\\ UFV: --\\ UV-ReLU: --\\ UV: --\\ E: --\end{tabular} &
  \begin{tabular}[c]{@{}l@{}}UDV: --\\ UDV-s: --\\ UDV-v1: 99.55\\ UDV-v2: 99.41\\ UDV-ReLU: --\\ UDV(unconstrained): --\\ UFV: --\\ UV-ReLU: 99.42\\ UV: --\\ R: 99.55\end{tabular} \\ \hline
\end{tabular}
}
\end{table}

\clearpage

\end{document}